\documentclass[11pt,english]{article}
\usepackage[T1]{fontenc}
\usepackage[utf8]{inputenc}
\usepackage{geometry}
\usepackage{comment}
\usepackage{algorithmic}
\usepackage{algorithm}
\usepackage{color}
\usepackage{mathrsfs}
\usepackage{float}
\usepackage{amsmath} 
\usepackage{amssymb}
\usepackage{amsfonts}
\usepackage{amsthm}
\usepackage{stmaryrd}
\usepackage{stackrel}
\usepackage{graphicx}
\usepackage{dsfont}
\usepackage{tikz}
\usetikzlibrary{decorations.pathreplacing}
\usepackage{framed}
\usepackage{appendix}
\usepackage{enumerate}
\usepackage{placeins}
\usepackage{authblk}
\usepackage{tikz-layers}
\usepackage{multirow}

\numberwithin{equation}{section}
\theoremstyle{plain}
\newtheorem{theorem*}{Theorem}
\newtheorem{definition*}{Definition}
\newtheorem{corollary*}{Corollary}
\newtheorem{lemma*}{Lemma}

\newtheorem{remark*}{Remark}
\newtheorem{assumption*}{Assumption}
\newtheorem{proposition*}{Proposition}
\newtheorem{property*}{Property}

\usepackage[round]{natbib}
\usepackage{hyperref}
\hypersetup{
colorlinks = true,
urlcolor = blue, 
linkcolor = blue,
citecolor = blue,
}

\newcommand{\Dn}{\mathscr{D}_{n}}

\newcommand{\bX}{\textbf{X}}

\newcommand{\bTheta}{\Theta}

\newcommand{\bx}{\textbf{x}}
\newcommand{\bz}{\textbf{z}}
\renewcommand{\P}{\mathds{P}}
\newcommand{\R}{\mathds{R}}
\newcommand{\E}{\mathbb{E}}
\newcommand{\V}{\mathbb{V}}

\newcommand{\INDSTATE}[1][1]{\STATE\hspace{#1\algorithmicindent}}

\newcommand{\Xperm}{X_{\pi_j}}
\newcommand{\Xpermi}{X_{i, \pi_j}}
\newcommand{\Xpermk}{X_{k, \pi_j}}

\usepackage{xpatch}
\makeatletter
\AtBeginDocument{\xpatchcmd{\@thm}{\thm@headpunct{.}}{\thm@headpunct{}}{}{}}
\makeatother

\title{\textbf{\LARGE MDA for random forests: inconsistency, and a practical solution via the Sobol-MDA}}
\author[1,2]{Clément Bénard}
\affil[1]{Safran Tech, Digital Sciences \& Technologies, 78114 Magny-Les-Hameaux, France}
\author[1]{Sébastien Da Veiga}
\affil[2]{Sorbonne Université, CNRS, LPSM, 75005 Paris, France}
\author[3]{Erwan Scornet}
\affil[3]{Ecole Polytechnique, IP Paris, CMAP, 91128 Palaiseau, France}
\date{}

\begin{document}

\maketitle

\begin{abstract}
    Variable importance measures are the main tools to analyze the black-box mechanisms of random forests. Although the mean decrease accuracy (MDA) is widely accepted as the most efficient variable importance measure for random forests, little is known about its statistical properties. In fact, the definition of MDA varies across the main random forest software. In this article, our objective is to rigorously analyze the behavior of the main MDA implementations. Consequently, we mathematically formalize the various implemented MDA algorithms, and then establish their limits when the sample size increases. This asymptotic analysis reveals that these MDA versions differ as importance measures, since they converge towards different quantities. More importantly, we break down these limits into three components: the first two terms are related to Sobol indices, which are well-defined measures of a covariate contribution to the response variance, widely used in the sensitivity analysis field, as opposed to the third term, whose value increases with dependence within covariates. Thus, we theoretically demonstrate that the MDA does not target the right quantity to detect influential covariates in a dependent setting, a fact that has already been noticed experimentally.
    To address this issue, we define a new importance measure for random forests, the Sobol-MDA, which fixes the flaws of the original MDA, and consistently estimates the accuracy decrease of the forest retrained without a given covariate, but with an efficient computational cost. The Sobol-MDA empirically outperforms its competitors on both simulated and real data for variable selection. An open source implementation in \texttt{R} and \texttt{C++} is available online.
\end{abstract}

\textbf{Keywords:}
MDA; Random forests; Sensitivity analysis; Sobol indices; Variable importance; Variable selection

\section{Introduction}

Random forests \citep{breiman2001random} are a statistical learning algorithm, which aggregates a large number of trees to solve regression and classification problems, and achieves state-of-the-art performance on a wide range of problems. In particular, random forests exhibit a good behavior on high-dimensional or noisy data, without any parameter tuning, and are also well known for their robustness. 
However, they suffer from a major drawback: a given prediction is generated through a large number of operations, typically tens of thousands, which makes the interpretation of the prediction mechanism impossible. Because of this complexity, random forests are often qualified as black boxes. More generally, the interpretability of learning algorithms is receiving an increasingly high interest since this black-box characteristic is a strong practical limitation. For example, applications involving critical decisions, typically healthcare, require predictions to be justified. The most popular way to interpret random forests is variable importance analysis: covariates are ranked by decreasing order of their importance in the algorithm prediction process. Thus, specific variable importance measures were developed along with random forests \citep{breiman2001random, breiman2003atechnical}. However, we will see that they may not target the right variable ranking to detect influential covariates in a dependent setting, and could therefore be improved. First, we present the objectives of variable importance. Second, we review the existing variable importance measures for random forests, and then conduct a theoretical analysis of their limitations. Finally, we introduce the Sobol-MDA algorithm, a new importance measure for random forests, which mimics the brute force algorithm of retraining the forest without a covariate to measure the accuracy decrease, but with a much higher computational efficiency. The Sobol-MDA is proved to be consistent, and outperforms the existing variable importance competitors for variable selection, as shown in the experiments. An implementation in \texttt{R} and \texttt{C++} of the Sobol-MDA is available at \texttt{https://gitlab.com/drti/sobolmda}, and is based on \texttt{ranger} \citep{wright2017ranger}, a fast implementation of random forests.

\section{Context and Objectives}

\subsection{Variable Importance for Random Forests.}
There are essentially two importance measures for random forests: the mean decrease accuracy (MDA) \citep{breiman2001random} and the mean decrease impurity (MDI) \citep{breiman2003atechnical}. 
The MDA measures the decrease of accuracy when the values of a given covariate are permuted, thus breaking its relation to the response variable and to the other covariates. On the other hand, the MDI sums the weighted decreases of impurity over all nodes that split on a given covariate, averaged over all trees in the forest.
In both cases, a high value of the metric means that the covariate is used in many important operations of the prediction mechanism of the forest. Unfortunately, there is no precise and rigorous interpretation since these two definitions are purely empirical. Furthermore, in the last two decades, many empirical analyses have highlighted the flaws of the MDI \citep{strobl2007bias}.
Although \citet{li2019debiased}, \citet{zhou2019unbiased}, and \citet{loecher2020unbiased} recently improved the MDI to partially remove its bias, \citet{scornet2020trees} demonstrated that the MDI is consistent under a strong and restrictive assumption: the regression function is additive and the covariates are independent. Otherwise, the MDI is ill-defined. Overall, the MDA is widely considered as the most efficient variable importance measure for random forests \citep{strobl2007bias, ishwaran2007variable, genuer2010variable, boulesteix2012overview}, and we therefore focus on the MDA. Although it is extensively used in practice, little is known about its statistical properties. To our knowledge, only \citet{ishwaran2007variable} and \citet{zhu2015reinforcement} provide theoretical analyses of modified versions of the MDA, but the asymptotic behavior of the original MDA algorithm \citep{breiman2001random} is unknown: \citet{ishwaran2007variable} considers Breiman's forests but simplifies the MDA procedure, whereas \citet{zhu2015reinforcement} considers the original MDA but assumes the independence of the covariates and an exponential concentration inequality on the random forest estimate, the latter being proved only for purely random forests, which do not use the data set to build the tree partitions.
On the practical side, many empirical analyses provide evidence that when covariates are dependent, the MDA may fail to detect some relevant covariates \citep{archer2008empirical, strobl2008conditional, nicodemus2009predictor, genuer2010variable, auret2011empirical, tolocsi2011classification, gregorutti2017correlation, hooker2019please, mentch2020getting}.
Several proposals \citep[][]{mentchquantifying2016, candes2016panning, williamson2020unified} were recently made to overcome this issue. \citet{mentchquantifying2016} prove the asymptotic normality of random forests, which enables detection of whether the predictions of a forest built without a given covariate are significantly different from the ones of the original forest with all covariates. Alternatively, \citet{candes2016panning} introduce model-X knockoffs, which rely on conditional randomization tests, where the relation between a covariate and the response variable is broken without modifying the joint distribution of the covariates.  Finally, \citet{williamson2020unified} propose to measure the decrease of accuracy between the original procedure and a new run without a given covariate.
However, these methods have a much higher computational cost, as many model retrains are involved, and are in particular intractable in high dimension.
Furthermore, it is critical to assess that the properties of a variable importance measure are in line with the final objective of the conducted analysis. In the following subsection, we review the possible goals of variable importance, and then introduce sensitivity analysis to deepen the theoretical understanding of the MDA. 

\subsection{Sensitivity Analysis}
In practice, obtaining raw measures of variable importance is rarely the end goal. Rather, practitioners are frequently interested in using such measures to detect influential covariates to either \citep{genuer2010variable}: (i) find a small number of covariates with a maximized accuracy, or (ii) detect and rank all influential covariates to focus on for further exploration with domain experts. Depending on which of these two objectives is of interest, different strategies should be used as the following example shows: if two influential covariates are strongly correlated, one must be discarded in the first case, while the two must be kept in the second case. Indeed, if two covariates convey the same statistical information, only one should be selected if the goal is to maximize the predictive accuracy with a small number of covariates, i.e., objective (i). On the other hand, these two covariates may be acquired differently and represent distinct physical quantities. Therefore, they may have different interpretations for domain experts, and both should be kept for objective (ii) of ranking all variables for interpretation.

Sensitivity analysis is the study of uncertainties in a system. The main goal is to apportion the uncertainty of a system response to the uncertainty of the different covariates. \citet{iooss2015review} and \citet{ghanem2017handbook} provide detailed reviews of global sensitivity analysis. In particular, sensitivity analysis introduces well-defined importance measures of covariate contributions to the response variance: Sobol indices \citep{sobol1993sensitivity, saltelli2002making} and Shapley effects \citep{shapley1953value, owen2014sobol, iooss2017shapley}.
These metrics are widely used to analyze computer code experiments, especially for the design of industrial systems. However, the literature about variable importance in the fields of statistical learning and machine learning rarely mentions sensitivity analysis. The reason of this hiatus is clear: until quite recently, sensitivity analysis was focused on independent covariates, whereas such an assumption is generally unreasonable in machine learning contexts. In the last years, \citet{gregorutti:tel-01146830} first established a link between sensitivity analysis and the MDA: in the case of independent covariates, the theoretical counterpart of the MDA is the unnormalized total Sobol index, i.e., twice the amount of explained variance lost when a given covariate is removed from the model, which is the expected quantity for both objectives (i) and (ii) in this independent setting. Accordingly, the algorithm from \citet{williamson2020unified} also estimates the total Sobol index when the accuracy metric is the explained variance, even when covariates are dependent, and although this connection is not explicitly mentioned. 
When one is using variable importance to select a small number of covariates while maximizing predictive accuracy, i.e. objective (i), the total Sobol index is clearly the relevant measure to eliminate the less influential covariates, as also suggested by \citet{williamson2020unified}.
Additionally, \citet{owen2014sobol} reintroduced Shapley effects, originally proposed in game theory \citep{shapley1953value}. Shapley effects exhibit very interesting properties for objective (ii), of ranking all variables for interpretation, as they equitably allocate the mutual contribution due to dependence and interactions to individual covariates. Shapley effects are now widely used by the machine learning community to interpret both tree ensembles and neural networks. SHAP values \citep{lundberg2017unified} also adapt Shapley effects for local interpretation of model predictions, and \citet{lundberg2018consistent} provide a fast algorithm for tree ensembles. Finally, we refer to \citet{antoniadis2020random} for a review of random forests and sensitivity analysis.

\section{MDA Theoretical Limitations} \label{sec_MDA}

\subsection{MDA Definitions}

The MDA was originally proposed by Breiman in his seminal article \citep{breiman2001random}, and works as follows. The values of a specific covariate are permuted to break its relation to the response variable. Then, the predictive accuracy is computed for this perturbed dataset. The difference between this degraded accuracy and the original one gives the importance of the covariate: a high decrease of accuracy means that the considered covariate has a strong influence on the prediction mechanism. However, a review of the literature on random forests and their software implementations reveals that there is no consensus on the exact mathematical formulation of the MDA. We focus on the most popular random forest algorithms: the \texttt{R} package \texttt{randomForests} \citep{RrandomForest} based on the original \texttt{Fortran} code from Breiman and Cutler, the fast \texttt{R/C++} implementation \texttt{ranger} \citep{wright2017ranger}, the most widely used \texttt{python} machine learning library \texttt{scikit-learn} \citep{scikit-learn} (\texttt{RandomForestClassifier}/\texttt{RandomForestRegressor}), and the \texttt{R} package \texttt{randomForestSRC} \citep{RrandomForestSRC}, which implements survival forests in addition to the original algorithm.
To give an order of magnitude, the typical number of users of each of these packages during the year 2020 is about half a million. A close inspection of their code exhibits that essentially three distinct definitions of the MDA are widely used. References and details about the MDA implementation in the package codes are provided in the Supplementary Material. The differences between the three MDA versions are twofold: the MDA can be computed based on the tree error or the whole forest error, and via a test set or out-of-bag samples, as summarized in Table \ref{table_MDA_def}. We first introduce the required notations, and then mathematically formalize these different MDA definitions.
\begin{table} 
\setlength{\tabcolsep}{2pt}
\centering
\begin{tabular}{| c | c | c | c |}
  \hline \hline
  Algorithm & Package & Error Estimate & Data \\
  \hline
  Train-Test MDA & \begin{tabular}{c}
    \texttt{scikit-learn}\\\texttt{randomForestSRC} \end{tabular} & Forest & Testing dataset \\
  \hline
  Breiman-Cutler MDA & \begin{tabular}{c}
    \texttt{randomForest} (normalized)  \\ \texttt{ranger} / \texttt{randomForestSRC} \end{tabular} & Tree & OOB sample  \\
  \hline
  Ishwaran-Kogalur MDA & \begin{tabular}{c}
    \vspace*{-2.5mm} \\ \texttt{randomForestSRC} \\ \vspace*{-2.5mm} \end{tabular} & Forest & OOB sample  \\
  \hline \hline
\end{tabular}
\vspace*{1.5mm}
\caption{Summary of the different MDA characteristics.}
\label{table_MDA_def}
\end{table}
We define a standard regression setting with the following Assumption \ref{A1}, as well as the random forest notations below.
\begin{assumption*} \label{A1}
The response variable $Y \in \R$ follows $Y = m(X) + \varepsilon$, 
where the covariate vector $X = (X^{(1)}, \hdots, X^{(p)}) \in [0,1]^p$ admits a density over $[0,1]^p$ bounded from above and below by strictly positive constants, $m$ is continuous, and the noise $\varepsilon$ is sub-Gaussian, independent of $X$, and centered.
A sample $\Dn = \{(X_1, Y_1), \hdots, (X_n, Y_n) \}$ of $n$ independent random variables distributed as $(X, Y)$ is available.
\end{assumption*}
The random CART estimate $m_n(x,\Theta)$ is trained with $\Dn$ and $\Theta$, where $\Theta$ is used to generate the bootstrap sampling and the split randomization, and $x \in [0,1]^p$ is a new observation. The component of $\Theta$ used to resample the data is denoted $\Theta^{(S)} \subset \{1,\hdots,n\}$.
The random forest estimate $m_{M,n}(x, \bTheta_{(M)})$ aggregates $M$ $\Theta$-random CART, each of which is randomized by a component of $\bTheta_{(M)} = (\Theta_1,\hdots,\Theta_M)$. 
In the sequel, we consider a fixed index $j \in \{1,\hdots,p\}$. Next, we define $\Xpermi$ as the vector $X_i$ where the $j$-th component is permuted between observations. Similarly, $\Xperm$ is the vector $X$ where the $j$-th component is replaced by an independent copy of $\smash{X^{(j)}}$.
Finally, we also introduce $\smash{X^{(-j)}}$, as the random vector $X$ without the $j$-th component.
Now, we can detail the three MDA definitions, summarized in Table \ref{table_MDA_def}.

The most simple approach is taken by \texttt{scikit-learn} where the forest is fit with a training sample and the accuracy decrease is estimated with an independent testing sample $\Dn' = \{(X'_1, Y'_1), \hdots, (X'_n, Y'_n) \}$. Throughout the article, we call the generalization error of the forest the expected squared error for a new observation, usually estimated with an independent sample. Thus, forest predictions are run for both the test set and its permuted version, and the corresponding mean squared errors are subtracted to give the generalization error increase, called the Train-Test MDA. 
\begin{definition*}[Train/Test MDA]
The Train/Test MDA is defined by
\begin{align*}
    \widehat{\textrm{MDA}}_{M,n}^{(TT)}(X^{(j)}) = \frac{1}{n} \sum_{i = 1}^{n}  \big\{Y'_i - m_{M,n}(\Xpermi', \bTheta_{(M)})\big\}^2 - \big\{Y'_i - m_{M,n}(X'_i, \bTheta_{(M)})\big\}^2.
\end{align*}
\end{definition*}
This algorithm is the only MDA version implemented in \texttt{scikit-learn}, and is one possibility in \texttt{randomForestSRC}. Note that the Train/Test-MDA is straightforward to implement with any random forest package by simply running predictions.

In practice, splitting the sample in two parts for training and testing often hurts the accuracy of the procedure, then decreasing the accuracy of the MDA estimate.
Since the data is bootstrapped prior to each tree growing, a portion of the sample is left out, which is called the out-of-bag sample and can be used to measure accuracy. Despite the lack of mathematical formulation in the original MDA introduction \citep{breiman2001random}, it seems clear that for each tree, the generalization error is estimated using its out-of-bag sample and the permuted version. Then, the two errors are subtracted and this difference is averaged across all trees to give the Breiman-Cutler MDA.
\begin{definition*}[Breiman-Cutler MDA]
If $X_{i,\pi_{j\ell}}$ is the $i$-th permuted out-of-bag sample for the $\ell$-th tree and for $\smash{i \in \{1,\hdots,n\} \setminus \Theta_{\ell}^{(S)}}$, then the Breiman-Cutler MDA (BC-MDA)  \citep{breiman2001random} is defined by
\begin{align*}
    \widehat{\textrm{MDA}}_{M,n}^{(BC)}(X^{(j)}) =
    \frac{1}{M} \sum_{\ell = 1}^{M}
    \frac{1}{N_{n,\ell}} \sum_{i = 1}^{n}  \big[\{Y_i - m_{n}(X_{i,\pi_{j\ell}}, \Theta_{\ell})\}^2 - \{Y_i - m_{n}(X_i, \Theta_{\ell})\}^2\big]
    \mathds{1}_{i \notin \Theta_{\ell}^{(S)}},
\end{align*}
where $N_{n,\ell} = \sum_{i = 1}^{n}  \mathds{1}_{i \neq \Theta_{\ell}^{(S)}}$ is the size of the out-of-bag sample of the $\ell$-th tree. 
\end{definition*}
Among the four main random forest implementations introduced above, only \texttt{ranger} and \texttt{randomForestSRC} exactly follow this original definition. In \texttt{randomForests}, the final quantity is normalized by the standard deviation of the generalization error differences. However, this procedure is questionable \citep{diaz2006gene, strobl2008danger}: a non-influential covariate would constantly have a standard deviation close to zero, potentially leading to a high normalized MDA.

More importantly, observe that Breiman's MDA definition is in fact a Monte-Carlo estimate of a random tree decrease of accuracy when a covariate is noised up. 
Since we are interested in the covariate influence in the entire forest, and not only in a single tree, it seems natural to extend the out-of-bag procedure to estimate the forest error \citep{ishwaran2008random} as implemented in \texttt{randomForestSRC}: for each observation $X_i$, we retrieve the random set $\Lambda_{n,i}$ of trees which do not involve $X_i$ in their construction because of the resampling step, formally defined by
\begin{align*}
    \Lambda_{n,i} = \{\ell \in \{1,\hdots,M\}: i \notin \Theta_{\ell}^{(S)}\}.
\end{align*}
We can take advantage of such batch of trees to define the out-of-bag random forest estimate by averaging the tree predictions considering only trees that belong to $\Lambda_{n,i}$, i.e., for $i \in \{1,\hdots, n\}$,
\begin{align*}
    m_{M,n}^{(OOB)}(X_i, \bTheta_{(M)}) = \frac{1}{|\Lambda_{n,i}|} \sum_{\ell \in \Lambda_{n,i}} m_n(X_i, \Theta_{\ell}) \mathds{1}_{|\Lambda_{n,i}| > 0}.
\end{align*}
It is therefore possible to estimate the random forest error using $\Dn$ alone.
Recall that for each $\Theta_{\ell}$-random tree, we randomly permute the $j$-th component of the out-of-bag dataset to define $X_{i,\pi_{j\ell}}$, and we stress that the permutation is independent for each tree. Then, we define the permuted out-of-bag forest estimate as
\begin{align*}
    m_{M,n,\pi_{j}}^{(OOB)}(X_{i}, \bTheta_{(M)}) = \frac{1}{|\Lambda_{n,i}|} \sum_{\ell \in \Lambda_{n,i}} m_n(X_{i,\pi_{j\ell}}, \Theta_{\ell}) \mathds{1}_{|\Lambda_{n,i}| > 0}.
\end{align*}
These estimates enable to compute both the out-of-bag error of the forest and the inflated out-of-bag forest error when a covariate is noised up. Finally, the difference between these two errors forms the Ishwaran-Kogalur MDA. From an algorithmic point of view, the only difference with Breiman's definition is the mechanism to aggregate tree predictions and compute the errors, as highlighted in Algorithms $1$ and $2$ of the Supplementary Material.
\begin{definition*}[Ishwaran-Kogalur MDA]
The Ishwaran-Kogalur MDA (IK-MDA) \citep{ishwaran2007variable, ishwaran2008random} is defined by
\begin{align*}
\widehat{\textrm{MDA}}_{M,n}^{(IK)}(X^{(j)}) = \frac{1}{N_{M,n}} \sum_{i = 1}^{n} \{Y_i - m_{M,n,\pi_{j}}^{(OOB)}(X_i, \bTheta_{(M)})\}^2 - \{Y_i - m_{M,n}^{(OOB)}(X_i, \bTheta_{(M)})\}^2,
\end{align*}
where $N_{M,n} = \sum_{i = 1}^{n} \mathds{1}_{|\Lambda_{n,i}| > 0}$ is the number of points which are not used in all tree constructions. 
\end{definition*}

An asymptotic analysis of these three MDA versions, summarized in Table \ref{table_MDA_def}, reveals that they do not share the same theoretical counterpart. Consequently, they have different meanings and generate different variable rankings, from which divergent conclusions can be drawn. However, these MDA versions are used interchangeably in practice. The convergence of the MDA is established in the next subsection, and then the different theoretical counterparts are analyzed in the following subsection.

\subsection{MDA Inconsistency}

The out-of-bag estimate is involved in both the Breiman-Cutler MDA and Ishwaran-Kogalur MDA, but is also used in practice to provide a fast estimate of the random forest error. We begin our asymptotic analysis by a result on the efficiency of the out-of-bag estimate, stated in Proposition \ref{prop_oob_risk} below, which shows that the out-of-bag error consistently estimates the generalization error of the forest. This result will be later used to establish the convergence of the Ishwaran-Kogalur MDA.
The only difference between the implemented algorithms and our theoretical results, is that the resampling in the forest growing is done without replacement to alleviate the mathematical analysis \citep{scornet2015consistency, mentchquantifying2016, wager2018estimation}. We define $a_n$ the number of subsampled training observations used to build each tree.

\begin{proposition*} \label{prop_oob_risk}
    If Assumption \ref{A1} is satisfied, for a fixed sample size $n$ and $i \in \{1,\hdots,n\}$, we have
    \begin{align*}
         \Big|\E\big[ \big\{ m_{M,a_n,n}^{(OOB)}(X_i, \bTheta_{(M)}) - m(X_i) \big\}^2  \big] - \E\big[ \big\{m_{M,a_n,n-1}(X, \bTheta_{(M)}) - m(X) \big\}^2 \big] \Big| = O\Big(\frac{1}{M}\Big).
    \end{align*}
\end{proposition*}
First observe that, by construction of the set of trees $\Lambda_{n,i}$, the out-of-bag estimate aggregates a smaller number of trees than in the standard forest: $\E[|\Lambda_{n,i}|] = (1 - a_n/n) M$ trees in average. Therefore, the errors of the out-of-bag and standard forest estimates are different quantities. 
To our knowledge, this is the first result which states the convergence of the out-of-bag error towards the forest error for any fixed sample size, with a fast rate of $1/M$. This suggests that growing a large number of trees in the forest, which is computationally possible and what is done in practice, ensures that the out-of-bag estimate provides a good approximation of the forest error.

Next, the convergence of the three versions of the MDA holds under the following Assumption \ref{A2} of the consistency of a theoretical randomized CART. Since we are interested in the random forest interpretation through the MDA, it seems natural to conduct our analysis assuming that each tree of the forest is an efficient learner, i.e., consistent.
To formalize such an assumption, we first define the variation of the regression function within a cell $A \subset [0,1]^p$ by
\begin{align*}
	\Delta(m, A) = \underset{x,x'\in  A}{\sup}|m(x) - m(x')|,
\end{align*}
and secondly, we introduce $A_k^{\star}(x, \Theta)$ the cell of the theoretical CART of depth $k$ (randomized with $\Theta$) in which the observation $x \in [0,1]^p$ falls.

\begin{assumption*} \label{A2}
The randomized theoretical CART tree built with the distribution of $(X, Y)$ is consistent, that is, for all $x \in [0,1]^p$, almost surely,
\begin{align*}
	\lim \limits_{k \to \infty} \Delta \{ m,  A_k^{\star}(x, \Theta) \} = 0.
\end{align*}
\end{assumption*}
At first glance, Assumption \ref{A2} seems quite obscure since it involves the theoretical CART. However, \citet{scornet2015consistency} show that Assumption \ref{A2} holds if the regression function is additive. Because the original CART \citep{leo1984classification} is a greedy algorithm, Assumption \ref{A2} may not always be satisfied when the regression function $m$ has interaction terms. However, it holds if the CART algorithm is slightly modified to avoid splits close to the edges of cells, and the split randomization is slightly increased to have a positive probability to split in all directions at all nodes \citep{meinshausen2006quantile, wager2018estimation}. Indeed in that case, all cells become infinitely small as the tree depth $k$ increases, and therefore Assumption \ref{A2} holds by continuity of $m$. Such modifications of CART have a negligible impact in practice on the random forest estimate since the cut threshold and the split randomization increase can be chosen arbitrarily small. Notice that such asymptotic regime is specifically analyzed in the next section.

As specified above, $a_n$ is the number of training observations subsampled without replacement to build each tree, and we define $t_n$ as the final number of terminal leaves in every tree.
Notice that we can specify $a_n$ in $m_{M,a_n,n}(x, \bTheta_{(M)})$ or $m_{a_n,n}(x, \Theta)$ when needed, but we omit it in general to avoid cumbersome notations.
In order to properly define the MDA procedures, the out-of-bag sample needs to be at least of size $2$ to enable permutations, i.e., $a_n \leq n-2$.
Finally, we need the following Assumption \ref{A3} on the asymptotic regime of the empirical forest as stated in \citet{scornet2015consistency}, which essentially controls the number of terminal leaves with respect to the sample size $n$ to enforce the random forest consistency. 

\begin{assumption*} \label{A3}
The asymptotic regime of $a_n$, the size of the subsampling without replacement, and the number of terminal leaves $t_n$ is such that $a_n \leq n-2$, $a_n/n < 1 - \kappa$ for a fixed $\kappa > 0$, $\lim \limits_{n \to \infty} a_n = \infty$, $\lim \limits_{n \to \infty} t_n = \infty$, and $\lim \limits_{n \to \infty} t_n \frac{(\log(a_n))^9}{a_n} = 0$.
\end{assumption*}

In the case of the Ishwaran-Kogalur MDA, the number of trees has to tend to infinity with the sample size to ensure convergence. To lighten notations, we drop the dependence of $M_n$ to $n$.
\begin{assumption*} \label{A4}
    The number of trees grows to infinity with the sample size $n$: $M \underset{n \to \infty}{\longrightarrow} \infty$.
\end{assumption*}

\begin{theorem*} \label{thm_MDA}
    If Assumptions \ref{A1}, \ref{A2}, and \ref{A3} are satisfied,
	then, for all $M \in \mathbb{N}^{\star}$ and $j \in \{1,\hdots,p\}$, we have
	\begin{align*} &(i) \quad  \widehat{\textrm{MDA}}_{M,n}^{(TT)}(X^{(j)}) \overset{\mathbb{L}^1}{\longrightarrow} \E[\{m(X) - m(\Xperm)\}^2] \\ &(ii) \quad \widehat{\textrm{MDA}}_{M,n}^{(BC)}(X^{(j)}) \overset{\mathbb{L}^1}{\longrightarrow} \E[\{m(X) - m(\Xperm)\}^2].
	\end{align*}
	If Assumption \ref{A4} is additionally satisfied, then
	\begin{align*}
	(iii) \quad \widehat{\textrm{MDA}}_{M,n}^{(IK)}(X^{(j)}) \overset{\mathbb{L}^1}{\longrightarrow} \E[\{m(X) - \E[m(\Xperm)|X^{(-j)}]\}^2].
	\end{align*}
\end{theorem*}
Theorem~\ref{thm_MDA} reveals that the theoretical MDA counterparts are not identical across the different MDA definitions. Thus, covariates are ranked according to different criteria depending on the MDA version involved. We deepen this discussion in the following subsection.

\subsection{MDA Analysis}

The theoretical counterparts of the MDA established in Theorem \ref{thm_MDA} are hard to interpret since $\Xperm$ has a different distribution from the original covariate vector $X$ whenever components of $X$ are dependent. These different MDA versions are widely used in practice to assess the variable importance of random forests, but the relevance of such analyses completely relies on the ranking criteria $\E[\{m(X) - m(\Xperm)\}^2]$ or $\smash{\E[\{m(X) - \E[m(\Xperm)|X^{(-j)}]\}^2]}$, according to Theorem~\ref{thm_MDA}. It is possible to deepen the discussion, observing that $X$ and $\Xperm$ are independent conditionally on $\smash{X^{(-j)}}$ by construction. It enables to break down the MDA limit using Sobol indices that are well-defined quantity to measure the contribution of a covariate to the response variance.

\begin{definition*}[Total Sobol Index]
	The total Sobol index of covariate $X^{(j)}$ \citep{sobol1993sensitivity, saltelli2002making} gives the proportion of explained response variance lost when $X^{(j)}$ is removed from the model, that is
	\begin{align*}
	    ST^{(j)} = \frac{\E\{\V(m(X) \mid X^{(-j)})\}}{\mathbb{V}(Y)} = \frac{\V\{m(X)\} - \V\{m^{(-j)}(X^{(-j)})\}}{\mathbb{V}(Y)},
	\end{align*}
	where $m^{(-j)}(X^{(-j)}) \overset{def}{=} \E\{m(X) \mid X^{(-j)}\}$. Notice that $ST^{(j)}$ is also called the independent total Sobol index in \citet{kucherenko2012estimation} and \citet{benoumechiara:tel-02293846}.
\end{definition*}
We also introduce a new sensitivity index: the total Sobol index computed for the input vector $\Xperm$. We call it the marginal total Sobol index, since the distribution of $\Xperm$ is the product of the marginal distributions of $\smash{X^{(j)}}$ and $\smash{X^{(-j)}}$. It can take high values even when $\smash{X^{(j)}}$ is strongly correlated with other covariates, as opposed to the original total Sobol index. We derive the main properties of this new sensitivity index below, proved in the Supplementary Material.
\begin{definition*}[Marginal Total Sobol Index]	
	The marginal total Sobol index of covariate $X^{(j)}$ is defined by
	\begin{align*}
	    ST^{(j)}_{mg} = \frac{\E\{\V(m(\Xperm) \mid X^{(-j)})\}}{\mathbb{V}(Y)}.
	\end{align*}
\end{definition*}

\begin{property*}[Marginal Total Sobol Index] \label{property_Smg}
    If Assumption \ref{A1} is satisfied, the marginal total Sobol index $ST^{(j)}_{mg}$ satisfies the following properties. \vspace*{-2mm}
    \begin{enumerate}[(a)]
        \item $ST^{(j)}_{mg} = 0 \iff ST^{(j)} = 0$.
        \item If the components of $X$ are independent, then we have $ST^{(j)}_{mg} = ST^{(j)}$.
        \item If $m$ is additive, i.e. $m(X) = \sum_{k} m_k(X^{(k)})$, then we have $ST^{(j)}_{mg} = \V\{m_j(X^{(j)})\}/\V[Y]$, and $ST^{(j)}_{mg} \geq ST^{(j)}$.
    \end{enumerate}
\end{property*}

Notice that the last property states that $ST^{(j)}_{mg} \geq ST^{(j)}$ for additive regression functions, which may also hold in the general case with interactions. However, such an extension is out of the scope of the article. We also mention that total Sobol indices can be generalized to a group of covariates.
It is now possible to break down the MDA limits as the sum of positive terms using total Sobol indices and the following quantity $\smash{\textrm{MDA}^{\star (j)}_3}$, further discussed below and defined as
\begin{align*}
    \textrm{MDA}_3^{\star (j)} = \E[\{\E(m(X) \mid X^{(-j)}) - \E(m(\Xperm) \mid X^{(-j)})\}^2].
\end{align*}
\begin{proposition*} \label{prop_MDA}
    If Assumptions \ref{A1}, \ref{A2} and \ref{A3} are satisfied, then for all $M \in \mathbb{N}^{\star}$ and $j \in \{1,\hdots,p\}$, we have
	\begin{align*} &(i) \quad  \widehat{\textrm{MDA}}_{M,n}^{(TT)}(X^{(j)}) \overset{\mathbb{L}^1}{\longrightarrow} \V[Y] \times ST^{(j)} + \V[Y] \times ST^{(j)}_{mg} + \textrm{MDA}_3^{ \star (j)} \\ &(ii) \quad \widehat{\textrm{MDA}}_{M,n}^{(BC)}(X^{(j)}) \overset{\mathbb{L}^1}{\longrightarrow} \V[Y] \times ST^{(j)} + \V[Y] \times ST^{(j)}_{mg} + \textrm{MDA}_3^{ \star (j)}.
	\end{align*}
	If Assumption \ref{A4} is additionally satisfied, then
	\begin{align*}
	(iii) \quad \widehat{\textrm{MDA}}_{M,n}^{(IK)}(X^{(j)}) \overset{\mathbb{L}^1}{\longrightarrow} \V[Y] \times ST^{(j)} + \textrm{MDA}_3^{ \star (j)}.
	\end{align*}
\end{proposition*}
Importantly, each term of the decompositions of Proposition \ref{prop_MDA} is positive, and can be interpreted alone. We denote $\smash{\textrm{MDA}_1^{ \star (j)} = \V[Y] \times ST^{(j)}}$ and $\smash{\textrm{MDA}_2^{ \star (j)} = \V[Y] \times ST_{mg}^{(j)}}$.

$\smash{\textrm{MDA}_1^{ \star (j)}}$ is the non-normalized total Sobol index that has a straightforward interpretation: the amount of explained output variance lost when $\smash{X^{(j)}}$ is removed from the model. This quantity is really the information one is looking for when computing the MDA for objective (i) of finding a small group of the most predictive covariates.

$\smash{\textrm{MDA}_2^{ \star (j)}}$ is the non-normalized marginal total Sobol index. Its interpretation is more difficult.
Intuitively, in the case of $\smash{\textrm{MDA}_1^{ \star (j)}}$, contributions due to the dependence between $X^{(j)}$ and $X^{(-j)}$ are excluded because of the conditioning on $X^{(-j)}$. For $\smash{\textrm{MDA}_2^{ \star (j)}}$, this dependence is ignored, and therefore such removal does not take place.
For example, if $X^{(j)}$ has a strong influence on the regression function but is highly correlated with other covariates, then $\smash{\textrm{MDA}_1^{ \star (j)}}$ is small, whereas $\smash{\textrm{MDA}_2^{ \star (j)}}$ is high. For objective (i), one wants to keep only one covariate of a group of highly influential and correlated inputs, and therefore $\smash{ST^{(j)}_{mg}}$ can be a misleading component.

$\smash{\textrm{MDA}_3^{ \star (j)}}$ is not a known measure of importance, and seems to have no clear interpretation: it measures how the permutation shifts the average of $m$ over the $j$-th covariate, and thus characterizes the structure of $m$ and the dependence of $X$ combined.
$\smash{\textrm{MDA}_3^{ \star (j)}}$ is null if covariates are independent.  The value of $\smash{\textrm{MDA}_3^{ \star (j)}}$ increases with dependence, and this effect can be amplified by interactions between covariates.

Overall, all MDA definitions are misleading with respect to both objectives $(i)$ and $(ii)$ since they include $\smash{\textrm{MDA}_3^{ \star (j)}}$ in their theoretical counterparts. In the Supplementary Material, we provide an analytical example to show how the MDA can fail to detect relevant covariates when the data has both dependence and interactions.
From a practical perspective, it is only possible to conclude in general that the Breiman-Cutler MDA or Ishwaran-Kogalur MDA should be used rather than the Train/Test-MDA. Indeed, on the one hand we only have access to one finite sample $\Dn$ in practice, which has to be split in two parts to use the Train/Test-MDA, hurting the forest accuracy. On the other hand, it is possible to grow many trees at a reasonable linear computational cost, and Proposition \ref{prop_oob_risk} ensures that the out-of-bag estimate is efficient in this case.
With additional assumptions on the data distribution, the Breiman-Cutler MDA and the Ishwaran-Kogalur MDA recover meaningful theoretical counterparts. 

\begin{corollary*} \label{cor_MDA_indep}
If covariates are independent, and if Assumptions \ref{A1}-\ref{A3} are satisfied, for all $M \in \mathbb{N}^{\star}$ and $j \in \{1,\hdots,p\}$, we have 
\begin{align*}
\widehat{\textrm{MDA}}_{M,n}^{(TT)}(X^{(j)})& \overset{\mathbb{L}^1}{\longrightarrow} 2\V[Y] \times ST^{(j)} \quad \textrm{and} \quad 
\widehat{\textrm{MDA}}_{M,n}^{(BC)}(X^{(j)}) \overset{\mathbb{L}^1}{\longrightarrow} 2\V[Y] \times ST^{(j)}.
\end{align*}
In addition, if Assumption \ref{A4} is satisfied,
\begin{align*}
\widehat{\textrm{MDA}}_{M,n}^{(IK)}(X^{(j)})& \overset{\mathbb{L}^1}{\longrightarrow} \V[Y] \times ST^{(j)}.
\end{align*}
\end{corollary*}
Thus, Corollary \ref{cor_MDA_indep} states that when covariates are independent, all MDA versions estimate the same quantity, the unnormalized total Sobol index (up to a factor $2$), as stated in \citet{gregorutti:tel-01146830}. However, since the Train/Test-MDA is based on a portion of the training sample, the Breiman-Cutler MDA on the accuracy of a single tree, and the Ishwaran-Kogalur MDA on the accuracy of the forest, the Ishwaran-Kogalur MDA appears to be a more efficient estimate than the two others in this independent setting.
Also notice that in the case of independent covariates, the total Sobol index is a relevant measure for both objectives (i) and (ii).
Interestingly, when covariates are dependent but without interactions, all MDA versions then estimate the marginal total Sobol index, as stated in the following Corollary.

\begin{corollary*} \label{cor_MDA_additive}
If the regression function $m$ is additive, and if Assumptions \ref{A1}-\ref{A3} are satisfied, for all $M \in \mathbb{N}^{\star}$ and $j \in \{1,\hdots,p\}$, we have 
\begin{align*}
\widehat{\textrm{MDA}}_{M,n}^{(TT)}(X^{(j)})& \overset{\mathbb{L}^1}{\longrightarrow} 2 \V[Y] \times ST_{mg}^{(j)} \quad \textrm{and} \quad 
\widehat{\textrm{MDA}}_{M,n}^{(BC)}(X^{(j)}) \overset{\mathbb{L}^1}{\longrightarrow} 2 \V[Y] \times ST_{mg}^{(j)}.
\end{align*}
In addition, if Assumption \ref{A4} is satisfied,
\begin{align*}
\widehat{\textrm{MDA}}_{M,n}^{(IK)}(X^{(j)})& \overset{\mathbb{L}^1}{\longrightarrow} \V[Y] \times ST_{mg}^{(j)}.
\end{align*}
\end{corollary*}

In this correlated and additive setting, the MDA versions now estimate the marginal total Sobol index, which takes the simple form stated in Property \ref{property_Smg}-(c), but is difficult to estimate with a finite sample because of dependence. The MDA is thus quite relevant for objective (ii) of ranking all variables for interpretation: while contributions due to the dependence between covariates are removed in the total Sobol index, it is not the case here. Also notice that covariates with no influence in the regression function are excluded. If we further assume that the regression function is linear, the MDA limits can be written with the linear coefficients and the input variances as stated in \citet{gregorutti2015grouped, hooker2019please}.

\begin{remark*}[Distribution Support] \label{rmk_dist}
    Our asymptotic analysis relies on Assumption \ref{A1}, which states that the support of the covariate distribution $X$ is a hypercube. Without such geometrical assumption, the support of $\Xperm$ may differ from the support of $X$ in the dependent case. It means that the random forest estimate may be applied on regions  with no training samples, resulting in inconsistent forest and MDA estimates, and then in a low predictive accuracy \citep{hooker2019please}. This is an additional source of confusion of the MDA when inputs are dependent, induced by the permutation trick. 
\end{remark*}

\section{Sobol-MDA} \label{sec_sobol}

\subsection{Objectives}
When covariates are dependent, the MDA fails to estimate the total Sobol index, which is our true objective to solve problem (i) of finding a small group of the most predictive covariates, as shown in Section \ref{sec_MDA}.
Therefore, we introduce an improved MDA procedure for random forests: the Sobol-MDA, that consistently estimates the total Sobol index even when covariates are dependent and have interactions.
The Sobol-MDA is able to identify the less relevant covariates, as the total Sobol index is the proportion of response explained variance lost when a given covariate is removed from the model. For example, if two influential variables are strongly correlated, the Sobol-MDA takes small values for these two variables, since the model accuracy does not decrease much when one of them is removed. Therefore, a recursive feature elimination procedure based on the Sobol-MDA is highly efficient for our objective (i) of selecting a small number of covariates while maximizing predictive accuracy.
Notice that training a random forest without the covariate of interest would also enable to get an estimate of the total Sobol index, and is the approach taken by \citet{williamson2020unified}. However, the Sobol-MDA only requires to perform forest predictions, which is computationally faster than the forest growing, and scales with the dimension $p$ as opposed to this brute force approach from \citet{williamson2020unified}. Similarly, \citet{mentchquantifying2016} detect influential covariates with hypothesis tests based on the asymptotic normality of random forests and a model retrain without the considered covariate. However, this approach is only valid in specific forest settings \citep{pengEJS1958}, which considerably reduce the accuracy of tree ensembles compared to Breiman's algorithm, and therefore the ability to identify influential covariates. It is also possible to estimate total Sobol indices with existing algorithms which are not specific to random forests. Indeed, this type of methods only requires a black-box estimate to generate predictions from given values of the covariates. Initially, \citet{mara2015non} introduce Monte-Carlo algorithms for the estimation of total Sobol indices in a dependent setting. The first step of the method is to generate a sample from the conditional distributions of the covariates. However, in our setting defined in Assumption \ref{A1}, we do not have access to these conditional distributions, and their estimation is a difficult problem when only a limited sample $\Dn$ is available. Consequently, the approach of \citet{mara2015non} is not really appropriate for our setting.     
Notice that the promising approach from \citet{candes2016panning} to detect relevant covariates also requires to sample from the conditional distributions of the covariates, and is therefore not adapted to our problem as well. In the sequel, we describe the Sobol-MDA algorithm, along with its main properties.

\subsection{Sobol-MDA Algorithm}

The key feature of the original MDA procedures is to permute the values of the $j$-th covariate to break its relation to the response, and then compute the degraded accuracy of the forest. Observe that this is strictly equivalent to drop the original dataset down each tree of the forest, but when a sample hits a split involving covariate $j$, it is randomly sent to the left or right side with a probability equal to the proportion of points in each child node. 
This fact highlights that the goal of the MDA is simply to perturb the tree prediction process to cancel out the splits on covariate $j$. 
Besides, notice that this point of view on the MDA procedure (using the original dataset and noisy trees) is introduced by \citet{ishwaran2007variable} to conduct a theoretical analysis of a modified version of the MDA. 
Here, our Sobol-MDA algorithm builds on the same principle of ignoring splits on covariate $j$, such that the noisy CART tree predicts $m^{(-j)}(X^{(-j)}) = \E[m(X)|X^{(-j)}]$, similarly to the tree that is rebuilt by removing $X^{(j)}$ from the training data. It enables to recover the proper theoretical counterpart: the unnormalized total Sobol index, i.e., $\smash{\E[\V(m(X)|X^{(-j)})]}$.
To achieve this, we leave aside the permutation trick, and use another approach to cancel out a given covariate $j$ in the tree prediction process: the partition of the covariate space obtained with the terminal leaves of the original tree is projected along the $j$-th direction, as shown in Figure \ref{fig_proj_CART}, and the outputs of the cells of this new projected partition are recomputed with the training data. From an algorithmic point of view, this procedure is quite straightforward as we will see below, and enables to get rid of covariate $X^{(j)}$ in the tree estimate. Then, it is possible to compute the accuracy of the associated out-of-bag projected forest estimate, subtract it from the original accuracy, and normalize the obtained difference by $\V[Y]$ to obtain the Sobol-MDA for $X^{(j)}$.

Interestingly, to compute SHAP values for tree ensembles, \citet{lundberg2018consistent} also introduce an algorithm to modify the CART predictions to estimate $\smash{m^{(-j)}(X^{(-j)})}$. More precisely, they propose the following recursive algorithm: the observation $x$ is dropped down the tree, but when a split on covariate $j$ is hit, $x$ is sent to both the left and right children nodes. Then, $x$ falls in multiple terminal cells of the tree. The final prediction is the weighted average of the cell outputs, where the weight associated to a terminal leave $A$ is given by an estimate of $\P(X \in A | X^{(-j)} = x^{(-j)})$: the product of the empirical probabilities to choose the side that leads to A at each split on covariate $j$ in the path of the original tree. At first sight, their approach seems suited to estimate total Sobol indices, but unfortunately, the weights are properly estimated by such procedure only if the covariates are independent.
Therefore, as highlighted in \citet{aas2019explaining}, this algorithm gives biased predictions in a correlated setting.

We improve over \citet{lundberg2018consistent} with the Projected-CART algorithm, formalized in Algorithm $3$ in the Supplementary Material: both training and out-of-bag samples are dropped down the tree and sent on both right and left children nodes when a split on covariate $j$ is met. Again, each observation may belong to multiple cells at each level of the tree. For each out-of-bag sample, the associated prediction is the output average over all training observations that belong to the same collection of terminal leaves. In other words, we compute the intersection of these terminal leaves to select the training observations belonging to every cell of this collection to estimate the prediction. This intersection gives the projected cell.
Overall, this mechanism is equivalent to projecting the tree partition on the subspace span by $X^{(-j)}$, as illustrated in Figure \ref{fig_proj_CART} for $p = 2$ and $j = 2$. Recall that $A_n(X, \Theta)$ is the cell of the original tree partition where $X$ falls, whereas the associated cell of the projected partition is denoted $\smash{A_n^{(-j)}(X^{(-j)}, \Theta)}$.
Formally, we respectively denote the associated projected tree and projected out-of-bag forest estimates as $\smash{m_n^{(-j)}(X^{(-j)}, \Theta)}$ and $\smash{m_{M,n}^{(-j, OOB)}(X_i^{(-j)}, \bTheta_{(M)})}$, respectively defined by
\begin{align*}
    m_n^{(-j)}(X^{(-j)}, \Theta) &= \frac{ \sum_{i=1}^{a_n} Y_i \mathds{1}_{X_i \in A_n^{(-j)}(X^{(-j)}, \Theta)}}{\sum_{i=1}^{a_n} \mathds{1}_{X_i \in A_n^{(-j)}(X^{(-j)}, \Theta)}}, \\
    m_{M,n}^{(-j, OOB)}(X_i^{(-j)}, \bTheta_{(M)}) &= \frac{1}{|\Lambda_{n,i}|} \sum_{\ell \in \Lambda_{n,i}} m_n^{(-j)}(X_i^{(-j)}, \Theta_{\ell}) \mathds{1}_{|\Lambda_{n,i}| > 0}.
\end{align*}
The Projected-CART algorithm provides two sources of improvements over \citet{lundberg2018consistent}: first, the training data points are dropped down the modified tree to recompute the cell outputs, and thus $\smash{\E[m(X)|X^{(-j)} \in A]}$ is directly estimated in each cell. Secondly, the projected partition is finer than in the original tree, which mitigates masking effects (when an influential covariate is not often selected in the tree splits because of other highly correlated covariates).

Finally, the Sobol-MDA estimate is given by the normalized difference of the squared error of the out-of-bag projected forest with the out-of-bag error of the original forest. Formally, we define the Sobol-MDA as 
\begin{align*}
\widehat{\textrm{S-MDA}}_{M,n}(X^{(j)}) =
\frac{1}{\hat{\sigma}^2_Y}
\frac{1}{n} \sum_{i = 1}^{n}  \big\{Y_i - m_{M,n}^{(-j, OOB)}&(X_i^{(-j)}, \bTheta_{(M)})\big\}^2 \\[-1em] & - \big\{Y_i - m_{M,n}^{(OOB)}(X_i, \bTheta_{(M)})\big\}^2,
\end{align*}
where $\hat{\sigma}^2_Y = \frac{1}{n-1} \sum_{i = 1}^{n}  (Y_i - \bar{Y})^2$ is the standard variance estimate of the response $Y$. An implementation in \texttt{R} and \texttt{C++} of the Sobol-MDA is available at \texttt{https://gitlab.com/drti/sobolmda} and is based on \texttt{ranger} \citep{wright2017ranger}, a fast implementation of random forests.
Given an initial random forest, the Sobol-MDA algorithm has a computational complexity of $O\{M n \log^3(n)\}$, which is in particular independent of the dimension $p$, and quasi-linear with the sample size $n$. On the other hand, the brute force approach from \citet{williamson2020unified} has a complexity of $O\{M p^2 n \log^2(n)\}$, which is quadratic with the dimension $p$ and therefore intractable in high-dimensional settings, as opposed to the Sobol-MDA. Additional details are provided in the Supplementary Material.

\begin{remark*}[Empty Cells]
    Some cells of the projected partition may contain no training samples. Consequently, the prediction for a new query point falling in such cells is undefined. In practice, the Projected-CART algorithm uses the following strategy to avoid empty cells. Recall that each level of the tree defines a partition of the input space (if a terminal leave occurs before the final tree level, it is copied down the tree at each level), and that a projected partition can thus be associated to each tree level. When a new observation is dropped down the tree, if it falls in an empty cell of the projected partition at a given tree level, the prediction is computed using the previous level. Notice that empty cells cannot occur in the partitions associated to the root and the first level of the tree by construction. Therefore, this mechanism enforces that the projected tree estimate is well defined over the full covariate space.
\end{remark*}

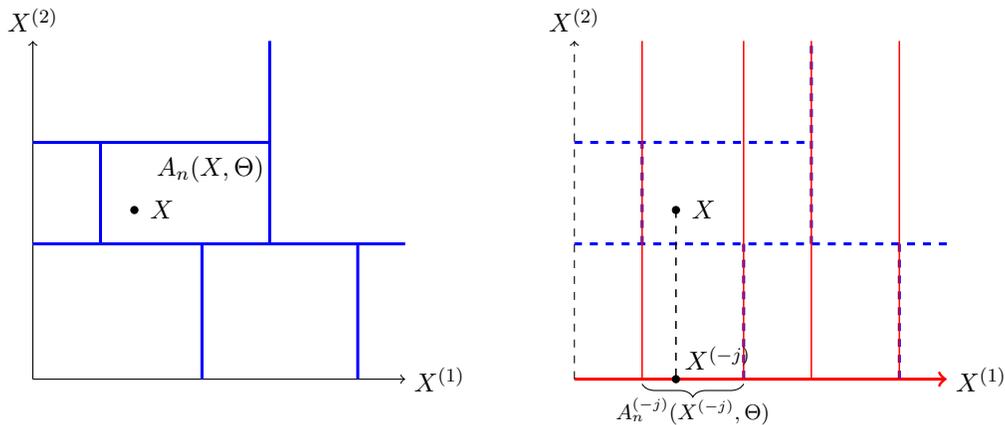
\begin{figure}
\centering
\begin{tikzpicture}[scale=0.9, every node/.style={scale=0.9}]

\draw[->] (-0,0) -- (5.5,0);
\draw (5.5,0) node[right] {$X^{(1)}$};
\draw [->] (0,-0) -- (0,5);
\draw (0,5) node[above] {$X^{(2)}$};
\draw [color = blue, line width = 0.4 mm] (3.5,2) -- (3.5,5);
\draw [color = blue, line width = 0.4 mm] (1,2) -- (1,3.5);
\draw [color = blue, line width = 0.4 mm] (2.5,2) -- (2.5,-0);
\draw [color = blue, line width = 0.4 mm] (5.5,2) -- (-0,2);
\draw [color = blue, line width = 0.4 mm] (-0,3.5) -- (3.5,3.5);
\draw [color = blue, line width = 0.4 mm] (4.8,2) -- (4.8,-0);
\draw (1.7,3.1) node[right] {$A_n(X, \Theta)$};
\filldraw (1.5,2.5) circle[radius=1.5pt];
\draw (1.6,2.5) node[right] {$X$};

\draw[->, line width = 0.4 mm, color = red] (8,0) -- (13.5,0);
\draw (13.5,0) node[right] {$X^{(1)}$};
\draw [->, dashed] (8,-0) -- (8,5);
\draw (8,5) node[above] {$X^{(2)}$};
\draw [color = blue, line width = 0.4 mm, dashed] (11.5,2) -- (11.5,5);
\draw [color = blue, line width = 0.4 mm, dashed] (9,2) -- (9,3.5);
\draw [color = blue, line width = 0.4 mm, dashed] (10.5,2) -- (10.5,-0);
\draw [color = blue, line width = 0.4 mm, dashed] (13.5,2) -- (8,2);
\draw [color = blue, line width = 0.4 mm, dashed] (8,3.5) -- (11.5,3.5);
\draw [color = blue, line width = 0.4 mm, dashed] (12.8,2) -- (12.8,0);
\draw [color = red, line width = 0.2 mm] (11.5,-0) -- (11.5,5);
\draw [color = red, line width = 0.2 mm] (9,-0) -- (9,5);
\draw [color = red, line width = 0.2 mm] (10.5,5) -- (10.5,-0);
\draw [color = red, line width = 0.2 mm] (12.8,5) -- (12.8,-0);
\filldraw (9.5,2.5) circle[radius=1.5pt];
\draw (9.6,2.5) node[right] {$X$};
\draw [line width = 0.2 mm, dashed] (9.5,2.5) -- (9.5,0);
\filldraw (9.5,0) circle[radius=1.5pt];
\draw (9.5,0) node[above right] {$X^{(-j)}$};
\draw [decorate,decoration={brace,mirror,amplitude=5pt},xshift=0pt,yshift=-2pt]
(9,0) -- (10.5,0) node [black,midway,yshift=-0.4cm] 
{\footnotesize $A_n^{(-j)}(X^{(-j)}, \Theta)$};

\end{tikzpicture}
\caption{Example of the partition of $[0,1]^2$ by a random CART tree (left side) projected on the subspace span by $X^{({-2})} = X^{(1)}$ (right side). Here, $p = 2$ and $j = 2$.}
\label{fig_proj_CART}
\end{figure}

\subsection{Sobol-MDA Consistency}
The original MDA versions do not converge towards the total Sobol index, which is the relevant quantity for our objective (i) of finding a small group of the most predictive covariates, as stated in Proposition \ref{prop_MDA}. On the other hand, the Sobol-MDA is consistent as stated below. Before introducing this convergence result, we need to introduce additional assumptions. Indeed, in Section \ref{sec_MDA}, we show the convergence of the different MDA versions provided that the forest is an efficient estimate, i.e. consistent. To enforce the consistency of random forests, we used Assumption \ref{A2} which controls the variation of the regression function in each cell of the theoretical tree: $\Delta\{m, A_k^{\star}(x, \Theta)\} \overset{a.s.}{\longrightarrow} 0$. Because the covariates may be dependent, Assumption \ref{A2} does not imply the same property for the projected partition. Therefore, we cannot directly build on \citet{scornet2015consistency} to prove the consistency of the Sobol-MDA. Thus, we take another route and define a new Assumption \ref{A5} which brings two modifications to the random forest algorithm. 

\begin{assumption*} \label{A5}
A node split is constrained to generate child nodes with at least a small fraction $\gamma > 0$ of the parent node observations. 
Secondly, the split selection is slightly modified: at each tree node, the number \texttt{mtry} of covariates drawn to optimize the split is set to $\texttt{mtry} = 1$ with a small probability $\delta > 0$. Otherwise, with probability $1 - \delta$, the default value of \texttt{mtry} is used.
\end{assumption*}
Importantly, since $\gamma$ and $\delta$ can be chosen arbitrarily small, the modifications of Assumption \ref{A5} are mild.
Besides, notice that this assumption follows \citet{meinshausen2006quantile} and \citet{wager2018estimation}: we slightly modify the random forest algorithm to enforce empirical cells to become infinitely small as the sample size increases. The projected forest inherits this property and an asymptotic analysis from \citet{gyorfi2006distribution} gives the consistency of the Sobol-MDA, provided that the complexity of tree partitions is appropriately controlled. If an original tree has $t_n$ terminal leaves, the associated projected partition may have a higher number of terminal leaves, at most $2^{t_n}$. Thus, we introduce Assumption \ref{A6}, which slightly modifies Assumption \ref{A3} with a more restrictive regime for the number of terminal leaves $t_n$ in the original trees.
\begin{assumption*} \label{A6}
The asymptotic regime of $a_n$, the size of the subsampling without replacement, and the number of terminal leaves $t_n$ is such that $a_n \leq n-2$, $a_n/n < 1 - \kappa$ for a fixed $\kappa > 0$, $\lim \limits_{n \to \infty} a_n = \infty$, $\lim \limits_{n \to \infty} t_n = \infty$, and $\lim \limits_{n \to \infty} 2^{t_n} \frac{(\log(a_n))^9}{a_n} = 0$.
\end{assumption*}

The Projected-CART algorithm ignores the splits based on the $j$-th covariate, and the associated out-of-bag projected forest consistently estimates $m^{(-j)}(X^{(-j)})$ under Assumptions \ref{A1}, \ref{A5}, and \ref{A6}, which leads to the Sobol-MDA consistency towards the total Sobol index, as stated below.
\begin{theorem*} \label{thm_MDA_sobol}
    If Assumptions \ref{A1}, \ref{A5}, and \ref{A6} are satisfied, for all $M \in \mathbb{N}^{\star}$ and $j \in \{1,\hdots,p\}$, we have
    \begin{align*}
        \widehat{\textrm{S-MDA}}_{M,n}(&X^{(j)}) \overset{p}{\longrightarrow} ST^{(j)}.
    \end{align*}
\end{theorem*}
According to Theorem \ref{thm_MDA_sobol}, the Sobol-MDA targets the appropriate quantity for objective (i), of selecting a small number of covariates while maximizing accuracy, whereas original MDA versions target a biased quantity, as stated in Proposition \ref{prop_MDA}. 
Notice that the brute force approach of retraining the forest without covariate $\smash{X^{(j)}}$ also estimates the total Sobol index, as proved in Theorem $3$ of the Supplementary Material. However, the brute force method has a quadratic computational complexity with respect to the dimension $p$, and is thus intractable in high dimensional settings. Since the Sobol-MDA complexity is independent of $p$, our approach is much more computationally efficient and outperforms its competitors, as shown in the next subsections.

\subsection{Experiments with Simulated Data} \label{subsec_xp_simu}

We conduct three batches of experiments. First, we use the analytical example of the Supplementary Material, and show empirically that the Sobol-MDA leads to the accurate importance variable ranking, while original MDA versions do not. Next, we simulate a typical setting where several groups of covariates are strongly correlated and only few covariates are involved in the regression function. In such difficult setting, the Sobol-MDA identifies the relevant covariates, as opposed to its competitors. Finally, we apply the recursive feature elimination algorithm on real data to show the performance improvement of the Sobol-MDA for variable selection.

We first consider the analytical example of the Supplementary Material, where the data has both dependence and interactions.
In this example, the covariates are distributed as a Gaussian vector with $p = 5$, and the regression function is given by
\begin{align*}
    m(X) = \alpha X^{(1)} X^{(2)} \mathds{1}_{X^{(3)} > 0} + \beta X^{(4)} X^{(5)} \mathds{1}_{X^{(3)} < 0}. 
\end{align*}
Here, we set $\alpha = 1.5$, $\beta = 1$, $\V[X^{(j)}] = 1$ for all covariates $j \in \{1,\hdots, 5\}$. The correlation coefficients are set to $\rho_{1,2} = 0.9$ and $\rho_{4,5} = 0.6$, and other covariance terms are null. Finally, we define the model response as $Y = m(X) + \varepsilon$, where $\varepsilon$ is an independent centered Gaussian noise whose variance verifies $\V[\varepsilon] / \V[Y] = 10\%$.
Then, we run the following experiment: first, we generate a sample $\Dn$ of size $n = 3000$ and distributed as the Gaussian vector $X$. 
Next, a random forest of $M = 300$ trees is fit with $\Dn$ and we compute the Breiman-Cutler MDA, Ishwaran-Kogalur MDA, the algorithm from \citet{williamson2020unified} denoted by $\widehat{\psi_{n,j}}$, and the Sobol-MDA.
To enable comparisons, the Breiman-Cutler MDA is normalized by $2\V[Y]$, and the Ishwaran-Kogalur MDA by $\V[Y]$, as suggested by Proposition \ref{prop_MDA}. To show the improvement of our Projected-CART algorithm, we also compute the Sobol-MDA using the algorithm from \citet{lundberg2018consistent}, denoted \smash{$\widehat{\textrm{S-MDA}_{Ldg}}$}.
All results are reported in Table \ref{table_SMDA}, along with the theoretical counterparts of the estimates, and the covariates are ranked by decreasing values of the theoretical total Sobol index since it is the value of interest: $X^{(3)}$, then $X^{(4)}$ and $X^{(5)}$, and finally $X^{(1)}$ and $X^{(2)}$.
\begin{table} 
\setlength{\tabcolsep}{3pt}
\centering
\begin{tabular}{|c ||c | c || c | c || c | c | c | c |}
  \hline \hline
   & $\textrm{BC-MDA}^{\star}$ & $\widehat{\textrm{BC-MDA}}$ & $\textrm{IK-MDA}^{\star}$ & $\widehat{\textrm{IK-MDA}}$ & $\textrm{ST}^{\star}$ & $\widehat{\textrm{S-MDA}}$ & $\widehat{\psi_{n,j}}$ & $\widehat{\textrm{S-MDA}_{Ldg}}$ \\
  \hline
  $X^{(3)}$ & 0.47 & 0.37 \small{(0.03)} & 0.47 & 0.43 \small{(0.02)} & \textit{0.47} & 0.45 \small{(0.03)} & 0.42 \small{(0.06)} & 0.43 \small{(0.03)} \\
  $X^{(4)}$ & 0.21 & 0.10 \small{(0.02)} & 0.37 & 0.14 \small{(0.01)} & \textit{0.10} & 0.08 \small{(0.01)} & 0.06 \small{(0.04)} & 0.13 \small{(0.01)}  \\
  $X^{(5)}$ & 0.21 & 0.09 \small{(0.01)} & 0.37 & 0.13 \small{(0.01)} & \textit{0.10} & 0.08 \small{(0.01)} & 0.06 \small{(0.04)} & 0.13 \small{(0.01)} \\
  $X^{(1)}$ & 0.64 & 0.24 \small{(0.02)} & 1.0 & 0.29 \small{(0.02)} & \textit{0.07} & 0.05 \small{(0.01)} & 0.03 \small{(0.04)} & 0.22 \small{(0.02)} \\
  $X^{(2)}$ & 0.64 & 0.24 \small{(0.02)} & 1.0 & 0.28 \small{(0.02)} & \textit{0.07} & 0.05 \small{(0.01)} & 0.03 \small{(0.04)} & 0.23 \small{(0.01)} \\
  \hline \hline
\end{tabular}
\vspace*{1.5mm}
\caption{BC-MDA (normalized by $2\V[Y]$), IK-MDA (normalized by $\V[Y]$), \citet{williamson2020unified} ($\widehat{\psi_{n,j}}$), and Sobol-MDA estimates for Example $1$ (standard deviations over $10$ repetitions in brackets). Theoretical counterparts are defined in Proposition \ref{prop_MDA}.}
\label{table_SMDA}
\end{table}
Thus, only the Sobol-MDA computed with the Projected-CART algorithm and \citet{williamson2020unified} rank the covariates in the same appropriate order than the total Sobol index. In particular, $X^{(4)}$ and $X^{(5)}$ have a higher total Sobol index than covariates $1$ and $2$ because of the stronger correlation between $X^{(1)}$ and $X^{(2)}$ than between $X^{(4)}$ and $X^{(5)}$.
For all the other importance measures, $X^{(1)}$ and $X^{(2)}$ are more important than $X^{(4)}$ and $X^{(5)}$.
For the original MDA, this is essentially due to the term $\smash{\textrm{MDA}_3^{ \star (j)}}$, which increases with correlation.
Since the explained variance of the random forest is $82\%$ in this experiment, all estimates have a negative bias. The bias of the Breiman-Cutler MDA and Ishwaran-Kogalur MDA dramatically increases with correlation. Indeed, a strong correlation between covariates leaves some regions of the input space free of training data. However, the out-of-bag permuted sample may fall in these regions, regions for which the forest has to extrapolate, resulting in a low predictive accuracy, and then in a high bias of the Breiman-Cutler MDA and Ishwaran-Kogalur MDA for correlated covariates. 
Finally, the Sobol-MDA computed with the algorithm of \citep{lundberg2018consistent} is biased as suggested by \citep{aas2019explaining}, and the bias also seems to increase with correlation.

We then consider the following problem inspired by \citet{archer2008empirical, gregorutti2017correlation} and related to gene expressions. The goal is to identify relevant covariates among several groups of many strongly correlated covariates. More precisely, we define $X$, a random vector of dimension $p = 200$, composed of $5$ independent groups of $40$ covariates. Each group is a centered gaussian random vector where two distinct components have a correlation of $0.8$ and the variance of each component is $1$. The regression function $m$ only involves one covariate from each group, and is simply defined by
\begin{align*}
    m(X) = 2 X^{(1)} + X^{(41)} + X^{(81)} + X^{(121)} + X^{(161)}. 
\end{align*}
Finally, we define the model response as $Y = m(X) + \varepsilon$, where $\varepsilon$ is an independent gaussian noise ($\V[\varepsilon] / \V[Y] = 10\%$).
Next, a sample of size $n = 1000$ is generated based on the distribution of $X$, and a random forest of $M = 300$ trees is fit. 
\begin{table} 
\centering
\begin{tabular}{| l | c |}
  \hline \hline
  \multicolumn{2}{|c|}{$\widehat{\textrm{S-MDA}}$} \\
  \hline
  $\textcolor{blue}{X^{(1)}}$ & 0.035 \\
  $\textcolor{blue}{X^{(161)}}$ & 0.005 \\
  $\textcolor{blue}{X^{(81)}}$ & 0.004 \\
  $\textcolor{blue}{X^{(121)}}$ & 0.004 \\
  $\textcolor{blue}{X^{(41)}}$ & 0.002 \\
  $X^{(179)}$ & 0.002 \\
  $X^{(13)}$ & 0.001 \\
  $X^{(25)}$ & 0.001 \\
  \hline \hline
\end{tabular}
\begin{tabular}{|l |c |}
  \hline \hline
  \multicolumn{2}{|c|}{$\widehat{\textrm{BC-MDA}}/2\V[Y]$} \\
  \hline
  $\textcolor{blue}{X^{(1)}}$ & 0.048 \\
  $X^{(25)}$ & 0.010 \\
  $X^{(31)}$ & 0.008 \\
  $X^{(14)}$ & 0.008 \\
  $X^{(40)}$ & 0.007 \\
  $X^{(3)}$ & 0.007 \\
  $X^{(17)}$ & 0.006 \\
  $X^{(26)}$ & 0.006 \\
  \hline \hline
\end{tabular}
\begin{tabular}{|l |c |}
  \hline \hline
  \multicolumn{2}{|c|}{$\widehat{\textrm{IK-MDA}}/\V[Y]$} \\
  \hline
  $\textcolor{blue}{X^{(1)}}$ & 0.056 \\
  $X^{(5)}$ & 0.009 \\
  $\textcolor{blue}{X^{(81)}}$ & 0.007 \\
  $\textcolor{blue}{X^{(41)}}$ & 0.005 \\
  $\textcolor{blue}{X^{(161)}}$ & 0.005 \\
  $X^{(15)}$ & 0.005 \\
  $\textcolor{blue}{X^{(121)}}$ & 0.005 \\
  $X^{(7)}$ & 0.005 \\
  \hline \hline
\end{tabular}
\begin{tabular}{|l |c |}
  \hline \hline
  \multicolumn{2}{|c|}{$\widehat{\psi_{n,j}}$} \\
  \hline
  $\textcolor{blue}{X^{(1)}}$ & 0.042 \\
  $X^{(119)}$ & 0.031 \\
  $X^{(155)}$ & 0.029 \\
  $X^{(24)}$ & 0.029 \\
  $X^{(54)}$ & 0.029 \\
  $X^{(72)}$ & 0.028 \\
  $X^{(103)}$ & 0.028 \\
  $X^{(124)}$ & 0.027 \\
  \hline \hline
\end{tabular}
\caption{Sobol-MDA, BC-MDA, IK-MDA, and \citet{williamson2020unified} ($\widehat{\psi_{n,j}}$) for Example $2$ (influential covariates in blue).}
\label{table_simulated_data}
\end{table}
\begin{table} 
\centering
\begin{tabular}{| c | c | c | c |}
  \hline \hline
  $\widehat{\textrm{S-MDA}}$ & $\widehat{\textrm{BC-MDA}}$ & $\widehat{\textrm{IK-MDA}}$ & $\widehat{\psi_{n,j}}$ \\
  \hline
  $0.90$ & $0$ & $0.33$ & $0$ \\
  \hline \hline
\end{tabular}
\caption{Probability to recover the $5$ relevant covariates in Example $2$ as the top $5$ most important covariates ranked using the BC-MDA, IK-MDA, Sobol-MDA, and \citet{williamson2020unified}.}
\label{table_proba}
\end{table}
Thus, Tables \ref{table_simulated_data} and \ref{table_proba} show that the Sobol-MDA identifies the five relevant covariates, whereas the Breiman-Cutler MDA, Ishwaran-Kogalur MDA, and \citet{williamson2020unified} identify some noisy covariates among the top five.
In this additive and correlated example, Corollary \ref{cor_MDA_additive} states that all MDA algorithms have an appropriate theoretical counterpart to identify the five relevant covariates involved in the regression function, because these five covariates are mutually independent. 
However, in this finite sample setting, the original MDA versions give a high importance to the covariates of the first group because of their correlation with the most influential covariate $X^{(1)}$. Since the Ishwaran-Kogalur MDA is based on the forest error, it outperforms the Breiman-Cutler MDA, which relies on the tree error.
Quite surprisingly, \citet{williamson2020unified} is the worst performing algorithm although it uses a brute force approach by retraining the forest without a given covariate to consistently estimate its total Sobol index, the appropriate theoretical counterpart. In fact, the multiple layers of data splitting involved in \citet{williamson2020unified} generate a high variance of the associated estimate, whereas the MDA and the Sobol-MDA operate with a given dataset and a given initial forest structure to compute the decrease of accuracy, resulting in finer estimates and a higher performance to detect irrelevant covariates.

\subsection{Experiments for Variable Selection with Real Data} \label{subsec_xp_real}
The recursive feature elimination algorithm is originally introduced by \citet{guyon2002gene} to perform variable selection with SVM. \citet{gregorutti2017correlation} apply the recursive feature elimination algorithm to random forests with the MDA as importance measure. The principle is to discard the less relevant covariates one by one, and is summarized in Algorithm $4$ in the Supplementary Material. Thus, the recursive feature elimination algorithm is a relevant  strategy for our objective (i) of finding a small group of the most predictive covariates.
At each step of the algorithm, the goal is to detect the less relevant covariates based on the trained model. Since the total Sobol index measures the proportion of explained response variance lost when a given covariate is removed, the optimal strategy is therefore to discard the covariate with the smallest total Sobol index. The Sobol-MDA directly estimates the total Sobol index, and therefore improves the performance of the recursive feature elimination procedure with respect to the original MDA, as shown in the following experiments. Indeed, the original MDA inflates the importance of dependent covariates, which leads to discard influential independent covariates, in favor of covariates which are related to the response only through correlation with others.

\begin{figure}
	\begin{center}
		\includegraphics[height=7.3cm,width=7.3cm]{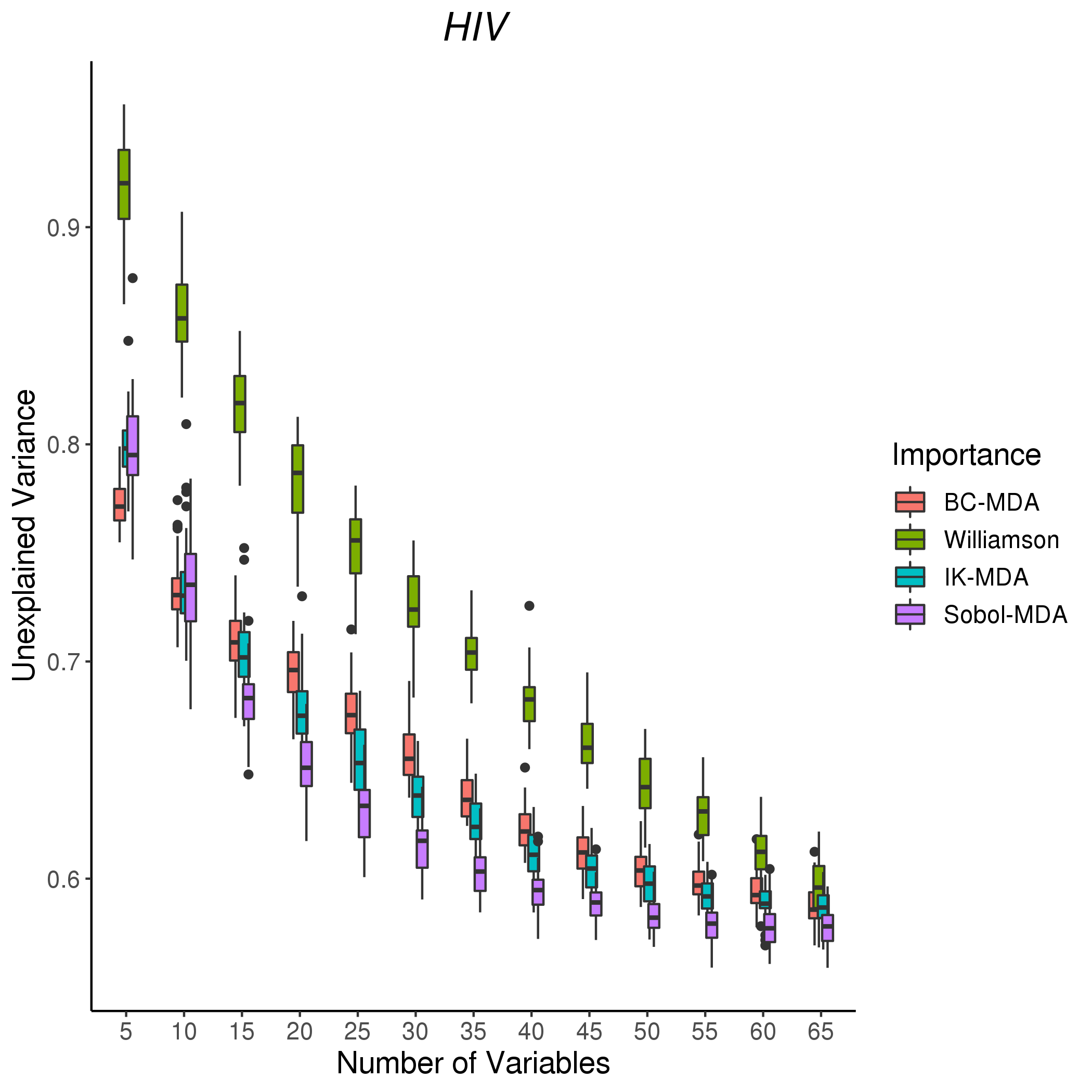}
		\includegraphics[height=7.3cm,width=7.3cm]{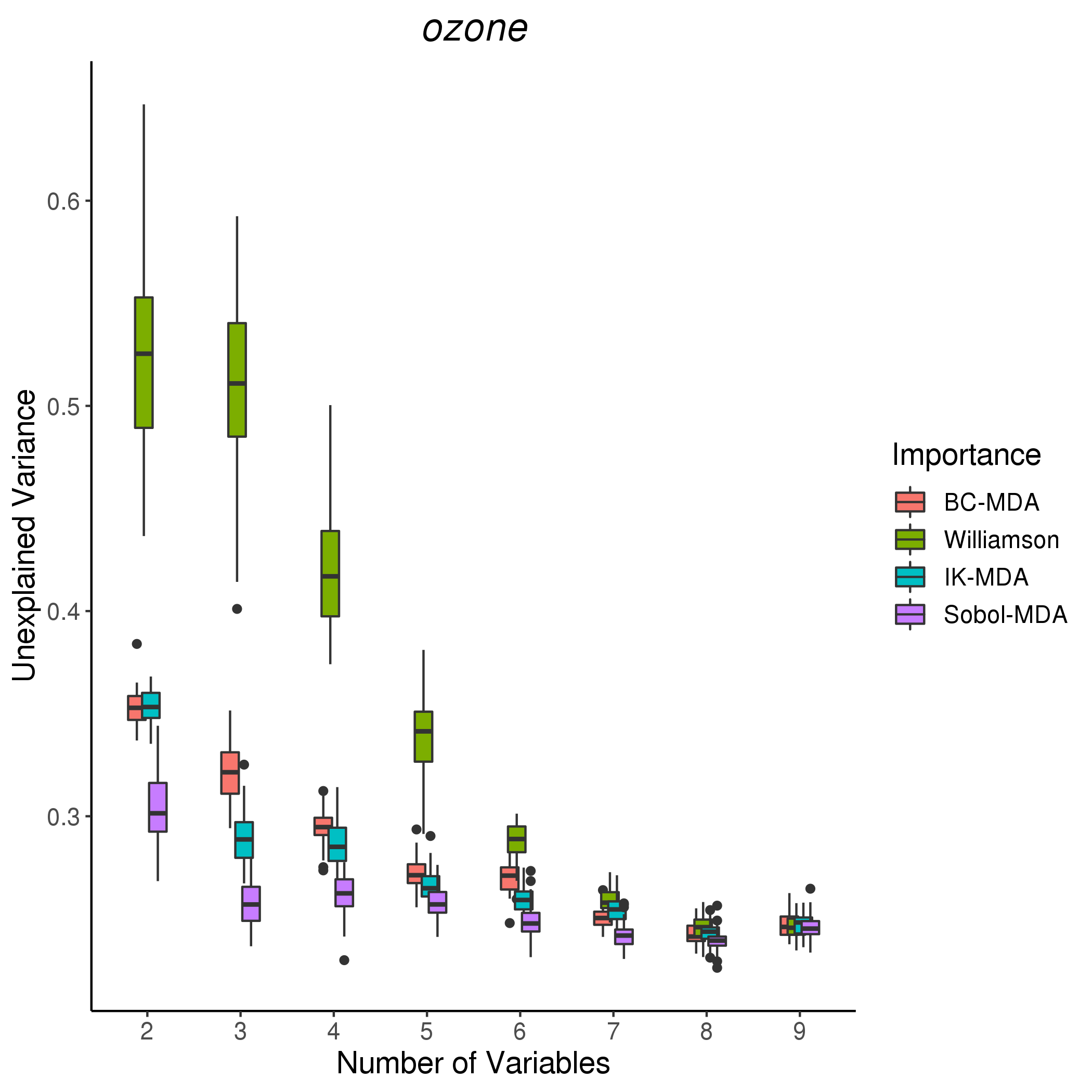}
		\caption{
		Random forest error versus the number of covariates for the ``HIV'' and ``Ozone'' datasets at each step of the recursive feature elimination algorithm, using different importance measures.
		} \label{fig_RFE_HIV}
	\end{center}
\end{figure}
The recursive feature elimination algorithm is illustrated with the ``Ozone'' data \citep{Dua:2019} and the high-dimensional dataset ``HIV'' as suggested in \citet{williamson2020unified}, and run using the original MDA, the Sobol-MDA, and \citet{williamson2020unified}.
At each step of the recursive feature elimination algorithm, the explained variance of the forest is retrieved. Following \citet{gregorutti2017correlation}, we do not use the out-of-bag error since it gives optimistically biased results, but use instead a $10$-fold cross-validation, repeated $40$ times to get uncertainties: the forest and the associated importance measure are computed with $9$ folds, and the error is estimated with the $10$-th fold.
Thus, Figure \ref{fig_RFE_HIV} highlights that the Sobol-MDA leads to a more efficient variable selection than all competitors for the ``HIV'' and ``Ozone'' datasets. We refer to the Supplementary Material for additional experiments. Notice that the Ishwaran-Kogalur MDA performs better than the Breiman-Cutler MDA, as expected from their theoretical counterparts stated in Proposition \ref{prop_MDA}. Finally the algorithm from \citet{williamson2020unified} is the worst performing approach because of the data splitting procedure, as explained in the previous subsection.

\FloatBarrier

\section*{Acknowledgement}
We thank the referees and the editors for their relevant suggestions to improve the article.

\bibliography{biblio}


\pagebreak
\vspace*{0.5cm}
\begin{center}
\textbf{\LARGE Supplementary Material for ``MDA for random forests: inconsistency, and a practical solution via the Sobol-MDA''}
\end{center}

\setcounter{section}{0}
\renewcommand*{\theHsection}{chX.\the\value{section}}
\setcounter{assumption*}{0}
\setcounter{theorem*}{0}
\setcounter{lemma*}{0}
\setcounter{proposition*}{0}
\setcounter{corollary*}{0}
\setcounter{property*}{0}

\section{Analytical Example for the MDA}
To illustrate the behavior of the MDA, we take a simple example and analytically derive the MDA limit and its three associated components $\smash{\textrm{MDA}_1^{ \star (j)}}$, $\smash{\textrm{MDA}_2^{ \star (j)}}$, and $\smash{\textrm{MDA}_3^{ \star (j)}}$. This example shows how the MDA is misleading when input variables are dependent.
We consider the Breiman-Cutler MDA, denoted by MDA to lighten notations. The TT-MDA or Ishwaran-Kogalur MDA lead to identical conclusions.

The input $\bX$ is a Gaussian vector of dimension $p = 5$. Its covariance matrix is defined by $\smash{\V[X^{(j)}] = \sigma_j^2}$ for $\smash{j \in \{1,\hdots,5\}}$, and all covariance terms are null except $\smash{\textrm{Cov}[X^{(1)},X^{(2)}] = \rho_{1,2} \sigma_1 \sigma_2}$ and $\smash{\textrm{Cov}[X^{(4)},X^{(5)}] = \rho_{4,5} \sigma_4 \sigma_5}$.
The regression function $m$ is given by
\begin{align*}
    m(\bX) = \alpha X^{(1)} X^{(2)} \mathds{1}_{X^{(3)} > 0} + \beta X^{(4)} X^{(5)} \mathds{1}_{X^{(3)} < 0}. 
\end{align*}
Notice that $m$ has a simple form to enable an easy interpretation of the importance measures, but that interaction terms are required to highlight the different behaviors of the three MDA components in a correlated setting.   
Simple calculations give the analytical expression $\smash{\textrm{MDA}^{\star (1)}}$ of the MDA limit for $\smash{X^{(1)}}$ as
\begin{align*}
    \textrm{MDA}^{\star (1)} =& 
    \underbrace{\frac{1}{2}(\alpha\sigma_{1}\sigma_2)^2(1 - \rho_{1,2}^2)}_{\textrm{MDA}_1^{\star (1)}} + \underbrace{\frac{1}{2}(\alpha\sigma_{1}\sigma_2)^2}_{\textrm{MDA}_2^{\star (1)}} + \underbrace{\frac{3}{2}\rho_{1,2}^2(\alpha\sigma_{1}\sigma_2)^2}_{\textrm{MDA}_3^{\star (1)}}.
\end{align*}
First, observe that $\smash{\textrm{MDA}_1^{\star (1)}}$ decreases with the correlation between $\smash{X^{(1)}}$ and $\smash{X^{(2)}}$. Indeed, $\smash{\textrm{MDA}_1^{\star (1)}}$ is the total Sobol index and when these two variables are strongly dependent, the additional information provided by $\smash{X^{(1)}}$ alone is small. In the extreme case, $\smash{\rho_{1,2} = 1}$ implies that $\smash{\textrm{MDA}_1^{\star (1)} = 0}$, i.e., $\smash{X^{(1)}}$ can be removed from the model without hurting the model accuracy since all its information is contained in $\smash{X^{(2)}}$.
On the other hand, $\smash{\textrm{MDA}_2^{\star (1)}}$ does not rely on the dependence between $\smash{X^{(1)}}$ and $\smash{X^{(2)}}$. Indeed, recall that in the case of $\smash{\textrm{MDA}_1^{ \star (1)}}$, contributions due to the dependence between $X^{(1)}$ and $X^{(2)}$ are excluded because of the conditioning on $X^{(2)}$. For $\smash{\textrm{MDA}_2^{ \star (1)}}$, this dependence is ignored, and therefore such removal does not take place.
Therefore, it is clear that the MDA mixes two terms with opposite meanings.
Finally, the third term $\smash{\textrm{MDA}_3^{\star (1)}}$ measures how the permutation of $\smash{X^{(1)}}$ shifts the mean value of the regression function averaged over $\smash{X^{(1)}}$, which is not a quantity of interest to rank variables. However, in a high correlation setting $\smash{\big(\rho_{1,2} > \frac{\sqrt{2}}{2}\big)}$, we have $\smash{\textrm{MDA}_3^{\star (1)} > \textrm{MDA}_1^{\star (1)}+ \textrm{MDA}_2^{\star (1)}}$, which means that the meaningless third term is the main contribution of the MDA value of variable $X^{(1)}$.
Besides, symmetrically for the other input variables, we have $\smash{\textrm{MDA}^{\star (1)} = \textrm{MDA}^{\star (2)}}$, and the same formula for $\smash{X^{(4)}}$ and $\smash{X^{(5)}}$ with the appropriate parameters. MDA formulas for variables $3, 4$, and $5$ are to be found in the last section of the Supplementary Material.

As stated in the introduction, one of the main objective of variable importance analysis is usually to select a small number of variables while maximizing the model accuracy. In our example, we show how the MDA fails for this purpose.
Let say we want to remove the less relevant input variable in a setting where the two vectors $\bX^{(1,2)}$ and $\bX^{(4,5)}$ are interchangeable ($\alpha\sigma_{1}\sigma_2 = \beta\sigma_{4}\sigma_5$), except that their dependence strengths differ and satisfy $\rho_{1,2} < \rho_{4,5}$. Since the correlation between variables $4$ and $5$ is higher than between variables $1$ and $2$, we should remove $X^{(4)}$ or $X^{(5)}$ to minimize the information loss, as suggested by the total Sobol index ranking
\begin{align*}
    ST^{(4)} = ST^{(5)}
    < ST^{(1)}  = ST^{(2)} < ST^{(3)}.
\end{align*}
However, in such setting we have 
\begin{align*}
    \textrm{MDA}^{\star (1)} = \textrm{MDA}^{\star (2)}
    < \textrm{MDA}^{\star (3)}
    < \textrm{MDA}^{\star (4)} = \textrm{MDA}^{\star (5)},
\end{align*}
that would lead to discard $X^{(1)}$ or $X^{(2)}$, which is suboptimal---see the last section of the Supplementary Material for computation details. On the other hand, using only $\smash{\textrm{MDA}_1^{ \star (j)}}$ or $\smash{\textrm{MDA}_1^{ \star (j)} + \textrm{MDA}_2^{ \star (j)}}$ as importance measures gives the accurate variable selection. The term $\smash{\textrm{MDA}_3^{ \star (j)}}$ artificially increases the MDA value because of correlation, and is thus misleading for both objectives (i) and (ii).

\section{Algorithms}

\subsection{MDA Algorithm Formulations}

Algorithms \ref{algo_BC_MDA} and \ref{algo_IK_MDA} respectively provide an algorithmic formulation of the Breiman-Cutler MDA and the Ishwaran-Kogalur MDA.

\begin{algorithm} \label{algo_BC_MDA}
\caption{Breiman-Cutler MDA}
\begin{algorithmic}[1]
\STATE \textbf{Input:} A random forest and a variable index $j \in \{1,\hdots,p\}$
\STATE for $\ell$ in $1,\hdots,M$:
\INDSTATE randomly permute the $j$-th component of the out-of-bag observations of the $\ell$-th tree
\INDSTATE for all permuted observations of the $\ell$-th tree:
\INDSTATE[2] compute the $\ell$-th tree prediction
\INDSTATE compute the quadratic error associated to these predictions
\INDSTATE subtract the original $\ell$-th tree error to the obtained quadratic error
\STATE average the error difference over all trees
\end{algorithmic}
\end{algorithm}

\begin{algorithm} \label{algo_IK_MDA}
\caption{Ishwaran-Kogalur MDA}
\begin{algorithmic}[1]
\STATE \textbf{Input:} A random forest and a variable index $j \in \{1,\hdots,p\}$
\STATE for $\ell$ in $1,\hdots,M$:
\INDSTATE randomly permute the $j$-th component of the out-of-bag observations of the $\ell$-th tree
\INDSTATE for all permuted observations of the $\ell$-th tree:
\INDSTATE[2] compute the $\ell$-th tree prediction
\STATE for $i \in 1, \hdots, n$:
\INDSTATE get the set of trees $\Lambda_{n,i}$ which do not involve $X_i$ in their construction
\INDSTATE average the tree prediction of the permuted $i$-th observation $X_{i,\pi_{j,\ell}}$ across all trees in $\Lambda_{n,i}$
\STATE compute the quadratic error associated to these averaged predictions
\STATE subtract the original forest error to the obtained quadratic error
\end{algorithmic}
\end{algorithm}

\subsection{Ishwaran-Kogalur MDA by Blocks}
The Ishwaran-Kogalur MDA is implemented in \texttt{randomForestSRC}. This package also provides the possibility to define the Ishwaran-Kogalur MDA by blocks: the trees of the forest are divided in a fixed number of blocks. The Ishwaran-Kogalur MDA is estimated for each block and then averaged. Thus, the Breiman-Cutler MDA can be seen as a specific case where the number of blocks is the number of trees $M$, and each block contains only one tree. 
On the theoretical side, if the number of blocks is fixed and Assumption \ref{A4} is satisfied, the number of trees in each block grows to infinity, and therefore Theorem \ref{thm_MDA}-(iii) still holds.

\subsection{Sobol-MDA Computational Complexity}
Recall that the computational complexity of the brute force approach of \citet{williamson2020unified}, where a forest is retrained without each input variable, is $O(M p^2 n \log^2(n))$, which is quadratic with the dimension $p$ and therefore intractable in high-dimensional settings.

On the other hand, the original MDA procedure has an average complexity of $O(M p n \log(n))$: to run a balanced tree prediction for a given data point, it is dropped down the $\log(n)$ levels of the tree, which makes a complexity of $O(n \log(n))$ for the full out-of-bag sample, repeated for the $M$ trees of the forest and the $p$ variables.
In the Sobol-MDA procedure, the complexity analysis is similar, except that when a point is dropped down the tree, it can be sent to both the left and right children nodes, generating multiple operations at a given tree level and then an additional multiplicative factor of $\log(n)$.
However, it is not necessary to run the Projected-CART algorithm for each of the $p$ covariates. Indeed, when a given observation is dropped down the tree, it meets at most $\log(n)$ different variables in the original tree path. Therefore, the Projected-CART prediction has to be computed only for $\log(n)$ covariates for each observation.
Thus, the Sobol-MDA algorithm has a computational complexity of $O(M n \log^3(n))$, which is in particular independent of the dimension $p$, and quasi-linear with the sample size $n$.

\subsection{Projected-CART}
We provide below Algorithm \ref{algo_proj_CART} for an implementation of the projected random forests.
\begin{algorithm}
\caption{Projected-CART}
\label{algo_proj_CART}
\begin{algorithmic}[1]
\STATE \textbf{Input:} A $\Theta$-random CART built with $\Dn$, and a variable index $j \in \{1,\hdots,p\}$. (Note that if a terminal leave occurs before the final tree level, it is copied at each level down the tree.)
\STATE Initialize both in-bag and OOB samples at the root node of the tree;
\STATE for all tree levels:
\INDSTATE[1] for all level nodes:
\INDSTATE[2] if the splitting variable is not $j$:
\INDSTATE[3] send each data point to the right or left children node according to the node split;
\INDSTATE[2] if the splitting variable is $j$:
\INDSTATE[3] send the node sample to both the right and left children node ignoring the split;
\INDSTATE[1] for all data points:
\INDSTATE[2] retrieve the collection of nodes where the data point falls at the current tree level;
\INDSTATE[1] for all OOB data points:
\INDSTATE[2] retrieve the set of in-bag points which fall in the same node collection;
\INDSTATE[2] if all nodes in the considered node collection are terminal:
\INDSTATE[3] compute the output average of the in-bag points;
\INDSTATE[3] set this average as the prediction for the considered OOB observation;
\INDSTATE[2] if no in-bag points fall in the same node collection:
\INDSTATE[3] retrieve the corresponding in-bag data points at the previous tree level;
\INDSTATE[3] set the output average of these in-bag points as the prediction for the considered \\ \hspace*{1.12cm} OOB observation;
\STATE return predictions;
\end{algorithmic}
\end{algorithm}

\subsection{Recursive Feature Elimination}

Figures \ref{fig_RFE_1} and \ref{fig_RFE_2} provide additional experiments to show that the Sobol-MDA leads to a more efficient variable selection than the Breiman-Cutler MDA, \citet{williamson2020unified}, and the Ishwaran-Kogalur MDA. Notice that Algorithm \ref{algo_RFE} recalls the RFE procedure. The ``Prostate'' dataset in Figure \ref{fig_RFE_2} is an example where the Sobol-MDA does not significantly improve over the original MDA. 

\begin{algorithm}
\caption{Recursive Feature Elimination}
\label{algo_RFE}
\begin{algorithmic}[1]
\STATE for $j$ in $1,\hdots,p$:
\INDSTATE train a random forest
\INDSTATE compute the MDA for all variables
\INDSTATE remove the variable with the smallest MDA
\STATE return the ordered list of removed variables
\end{algorithmic}
\end{algorithm}

\begin{figure}
	\begin{center}
		\includegraphics[height=7cm,width=7cm]{RFE_SobolMDA_ozone.png}
		\includegraphics[height=7cm,width=7cm]{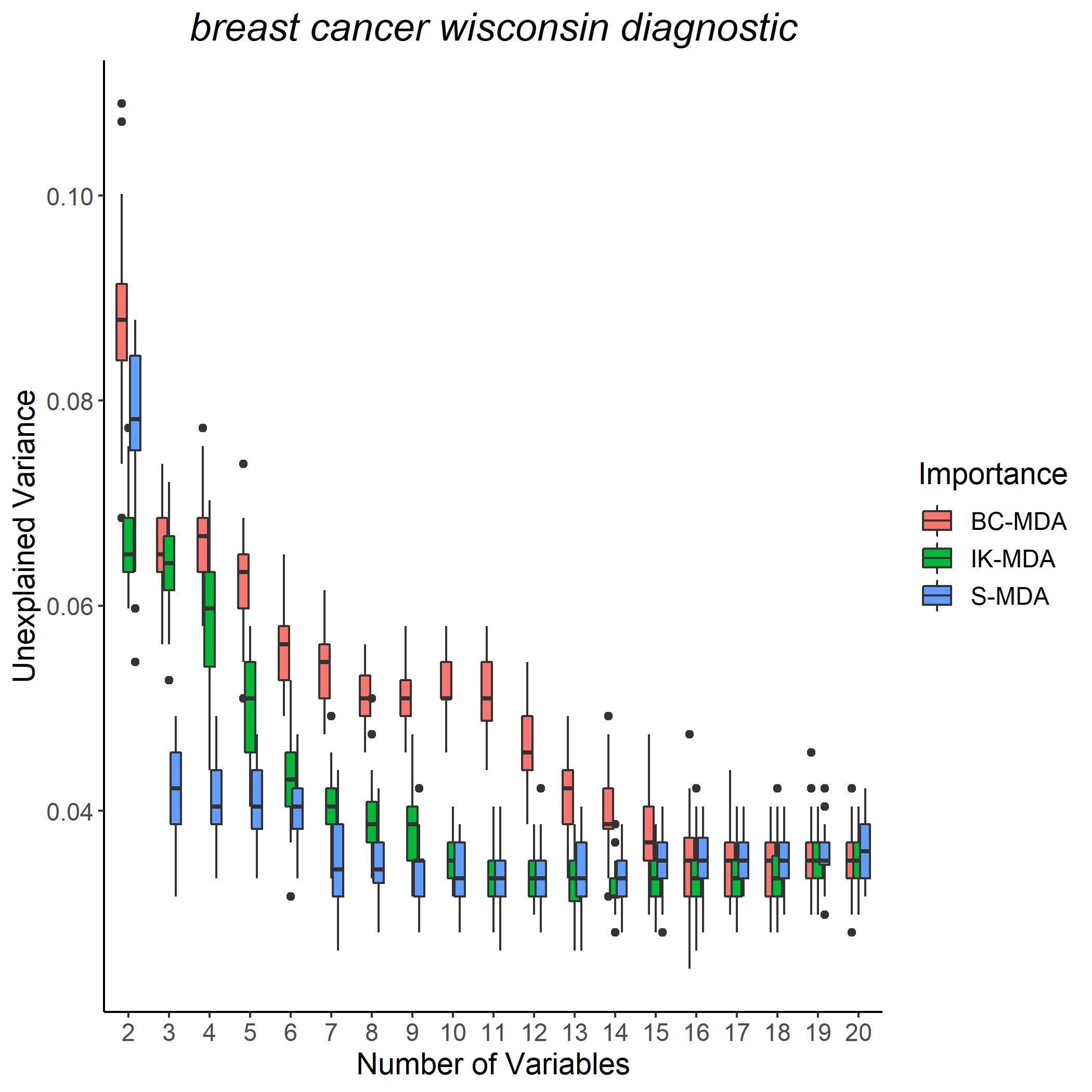}
		\caption{
		Random forest error versus the number of variables for the ``Ozone'' and ``Breast Cancer Wisconsin Diagnostic'' datasets at each step of the RFE, using different importance measures: BC-MDA, \citet{williamson2020unified}, IK-MDA, and Sobol-MDA.
		} \label{fig_RFE_1}
	\end{center}
\end{figure}
\begin{figure}
	\begin{center}
		\includegraphics[height=7cm,width=7cm]{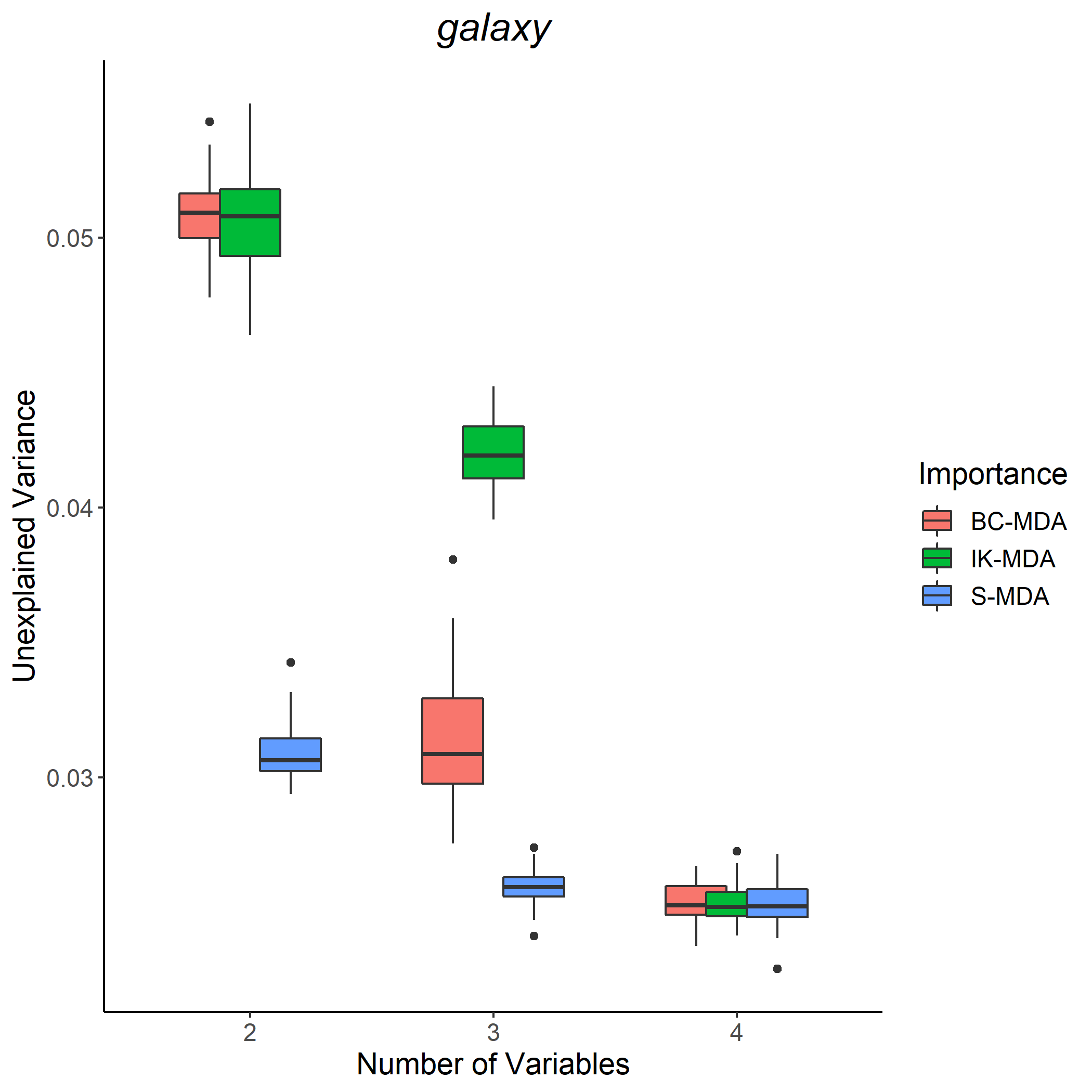}
		\includegraphics[height=7cm,width=7cm]{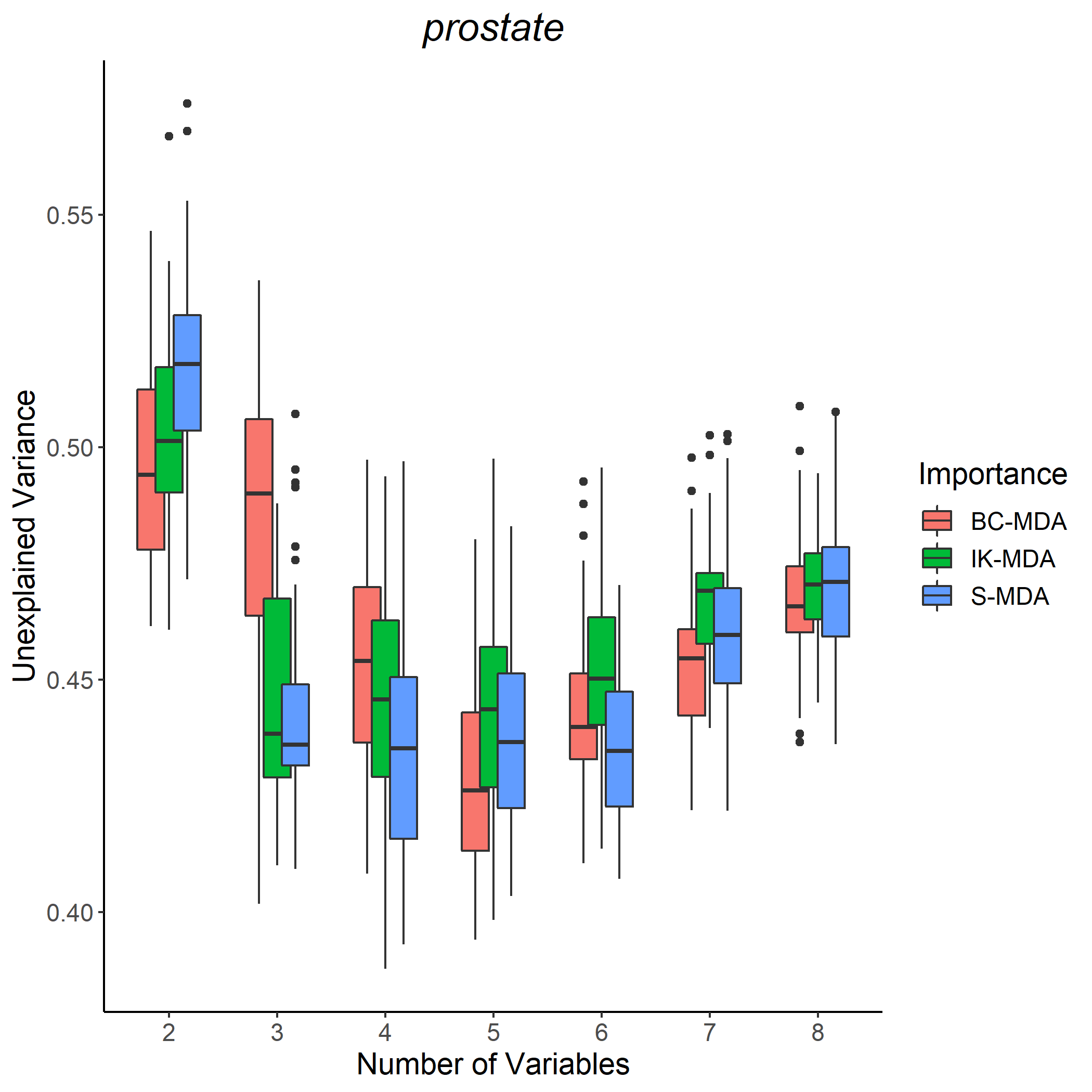}
		\caption{
		Random forest error versus the number of variables for the ``Galaxy'' and ``Prostate'' datasets at each step of the RFE, using different importance measures: BC-MDA, IK-MDA, and Sobol-MDA.
		} \label{fig_RFE_2}
	\end{center}
\end{figure}

\section{Proof of the MDA Consistency} \label{sec_proof_MDA}

\subsection{Assumptions and Theorem \ref{thm_MDA}}

We recall Assumptions \ref{A1}, \ref{A2}, \ref{A3}, \ref{A4}, Proposition \ref{prop_oob_risk}, and Theorem \ref{thm_MDA} for the sake of clarity.

\begin{assumption*} \label{A1}
The response $Y \in \R$ follows \vspace*{-2mm}
\begin{align*}
    Y = m(\bX) + \varepsilon
\end{align*}
where $\bX = (X^{(1)}, \hdots, X^{(p)}) \in [0,1]^p$ admits a density over $[0,1]^p$ bounded from above and below by strictly positive constants, $m$ is continuous, and the noise $\varepsilon$ is sub-Gaussian, independent of $\bX$, and centered.
A sample $\Dn = \{(\bX_1, Y_1), \hdots, (\bX_n, Y_n) \}$ of $n$ independent random variables distributed as $(\bX, Y)$ is available.
\end{assumption*}

\begin{assumption*} \label{A2}
The randomized theoretical CART tree built with the distribution of $(\bX, Y)$ is consistent, that is, for all $\bx \in [0,1]^p$, almost surely,
\begin{align*}
	\lim \limits_{k \to \infty} \Delta(m,  A_k^{\star}(\bx, \Theta)) = 0.
\end{align*}
\end{assumption*}

\begin{assumption*} \label{A3}
The asymptotic regime of $a_n$, the size of the subsampling without replacement, and the number of terminal leaves $t_n$ is such that $a_n \leq n-2$, $a_n/n < 1 - \kappa$ for a fixed $\kappa > 0$, $\lim \limits_{n \to \infty} a_n = \infty$, $\lim \limits_{n \to \infty} t_n = \infty$, and $\lim \limits_{n \to \infty} t_n \frac{(\log(a_n))^9}{a_n} = 0$.
\end{assumption*}

\begin{assumption*} \label{A4}
    The number of trees grows to infinity with the sample size $n$: $M \underset{n \to \infty}{\longrightarrow} \infty$.
\end{assumption*}

\begin{proposition*} \label{prop_oob_risk}
    If Assumption \ref{A1} is satisfied, for a fixed $n$ and $i \in \{1,\hdots,n\}$, we have 
    \begin{align*}
         \Big|\E\big[ \big( m_{M,a_n,n}^{(OOB)}(\bX_i, \bTheta_{M}) - m(\bX_i) \big)^2  \big] - \E\big[ \big(m_{M,a_n,n-1}(\bX, \bTheta_{M}) - m(\bX) \big)^2 \big] \Big| = O\Big(\frac{1}{M}\Big).
    \end{align*}
\end{proposition*}

\begin{theorem*} \label{thm_MDA}
    If Assumptions \ref{A1}, \ref{A2}, and \ref{A3} are satisfied,
	then, for all $M \in \mathbb{N}^{\star}$ and $j \in \{1,\hdots,p\}$ we have
	\begin{align*} &(i) \quad  \widehat{\textrm{MDA}}_{M,n}^{(TT)}(X^{(j)}) \overset{\mathbb{L}^1}{\longrightarrow} \E[(m(\bX) - m(\Xperm))^2] \\ &(ii) \quad \widehat{\textrm{MDA}}_{M,n}^{(BC)}(X^{(j)}) \overset{\mathbb{L}^1}{\longrightarrow} \E[(m(\bX) - m(\Xperm))^2].
	\end{align*}
	If Assumption \ref{A4} is additionally satisfied, then
	\begin{align*}
	(iii) \quad \widehat{\textrm{MDA}}_{M,n}^{(IK)}(X^{(j)}) \overset{\mathbb{L}^1}{\longrightarrow} \E[(m(\bX) - \E[m(\Xperm)|\bX^{(-j)}])^2].
	\end{align*}
\end{theorem*}

\subsection{Proof of Theorem \ref{thm_MDA}-(i)} \label{sec_proof_TT}

Assumptions \ref{A1}, \ref{A2} and \ref{A3} are sufficient to slightly extend the $\mathbb{L}^2$-consistency of random forests from \citet[Theorem 1]{scornet2015consistency} to the case where inputs are dependent, and also when the prediction is performed for the permuted sample (i.e, for a query point with a different distribution than the training data). Then, the TT-MDA consistency follows using a standard asymptotic analysis.
\begin{lemma*} \label{lemme_rf_consistency}
    If Assumptions \ref{A1}, \ref{A2}, and \ref{A3} are satisfied, for $M \in \mathbb{N}^{\star}$ we have
    \begin{align*}
    \lim \limits_{n \to \infty} \E[(m_{M,n}(\bX, \bTheta_{M}) - m(\bX))^2] = 0, 
    \end{align*}
    and for all $j \in \{1,\hdots,p\}$
    \begin{align*}
        \lim \limits_{n \to \infty} \E[(m_{M,n}(\Xperm, \bTheta_{M}) - m(\Xperm))^2] = 0.
    \end{align*}
\end{lemma*}

\begin{proof}[Proof of Theorem \ref{thm_MDA}-(i)]
    We assume that \ref{A1}, \ref{A2}, and \ref{A3} are satisfied, and fix $j \in \{1,\hdots,p\}$ and $M \in \mathbb{N}^{\star}$. \\
    Firstly, according to Lemma \ref{lemme_rf_consistency}, we have
    \begin{align} \label{X_consistency}
        \lim \limits_{n \to \infty} \E[(m_{M,n}(\bX, \bTheta_{M}) - m(\bX))^2] = 0, 
    \end{align}
    and
    \begin{align} \label{Z_consistency}
        \lim \limits_{n \to \infty} \E[(m_{M,n}(\Xperm, \bTheta_{M}) - m(\Xperm))^2] = 0.
    \end{align}
    
    Next, we can break down the Train/Test-MDA as follows
    \begin{align*}
    \widehat{\textrm{MDA}}_{M,n}^{(TT)}(X^{(j)}) = \frac{1}{n} \sum_{i = 1}^{n}&  \big(Y'_i - m_{M,n}(\Xpermi', \bTheta_{M})\big)^2 - \big(Y'_i - m_{M,n}(\bX'_i, \bTheta_{M})\big)^2 \\
    = \frac{1}{n} \sum_{i = 1}^{n}&  \big(m(\bX'_i) + \varepsilon'_i - m_{M,n}(\Xpermi', \bTheta_{M})\big)^2 - \big(m(\bX'_i) + \varepsilon'_i - m_{M,n}(\bX'_i, \bTheta_{M})\big)^2 \\
     = \frac{1}{n} \sum_{i = 1}^{n}&  \big([m(\bX'_i) - m(\Xpermi')] + [m(\Xpermi') - m_{M,n}(\Xpermi', \bTheta_{M})] + \varepsilon'_i \big)^2 \\ & - \big(m(\bX'_i) - m_{M,n}(\bX'_i, \bTheta_{M}) + \varepsilon'_i\big)^2 \\
     =  \frac{1}{n} \sum_{i = 1}^{n}&  [m(\bX'_i) - m(\Xpermi')]^2 
      + [m(\Xpermi') - m_{M,n}(\Xpermi', \bTheta_{M})]^2 + \varepsilon'^2_i \\
     & + 2[m(\bX'_i) - m(\Xpermi')][m(\Xpermi') - m_{M,n}(\Xpermi', \bTheta_{M})] \\[1em]
     & + 2\varepsilon'_i[m(\bX'_i) - m(\Xpermi')] + 2\varepsilon'_i[m(\Xpermi') - m_{M,n}(\Xpermi', \bTheta_{M})] \\[1em]
     & - [m(\bX'_i) - m_{M,n}(\bX'_i, \bTheta_{M})]^2 - \varepsilon'^2_i - 2\varepsilon'_i[m(\bX'_i) - m_{M,n}(\bX'_i, \bTheta_{M})].
    \end{align*}
    Then, we use the triangle inequality and the previous expression to get the following bound
    \begin{align}
    \E\big[\big|\widehat{\textrm{MDA}}_{M,n}^{(TT)}(X^{(j)}) & - \E[(m(\bX) -   m(\Xperm))^2]\big|\big] \nonumber\\
     &\leq \E\big[\big|\frac{1}{n} \sum_{i = 1}^{n}  [m(\bX'_i) - m(\Xpermi')]^2 - \E[(m(\bX) - m(\Xperm))^2]\big|\big] \label{term_1}\\
     &+ \E\big[\frac{1}{n}\sum_{i = 1}^{n}[m(\Xpermi') - m_{M,n}(\Xpermi', \bTheta_{M})]^2\big] \label{term_2}\\
     & + \E\big[\big|\frac{2}{n}\sum_{i = 1}^{n}[m(\bX'_i) - m(\Xpermi')][m(\Xpermi') - m_{M,n}(\Xpermi', \bTheta_{M})]\big|\big]
     \label{term_3}\\ & + \E\big[\big|\frac{2}{n}\sum_{i = 1}^{n}\varepsilon'_i[m(\bX'_i) - m(\Xpermi')]\big|\big]
     \label{term_4}\\ & + \E\big[\big|\frac{2}{n}\sum_{i = 1}^{n}\varepsilon'_i[m(\Xpermi') - m_{M,n}(\Xpermi', \bTheta_{M})]\big|\big]
     \label{term_5}\\ & + \E\big[\frac{1}{n}\sum_{i = 1}^{n}[m(\bX'_i) - m_{M,n}(\bX'_i, \bTheta_{M})]^2\big] \label{term_6}\\ & + \E\big[\big|\frac{2}{n}\sum_{i = 1}^{n}\varepsilon'_i[m(\bX'_i) - m_{M,n}(\bX'_i, \bTheta_{M})]\big|\big]. \label{term_7}
    \end{align}
    Now, let us consider all the terms on the right hand side one by one.
    
    The first and fourth terms (\ref{term_1}) and (\ref{term_4}) do not depend on the forest estimate, but it is not possible to simply apply the law of large numbers since the permutation introduces dependence within samples. For both terms, we prove $\mathbb{L}^2$-convergence, which implies the $\mathbb{L}^1$-convergence we are looking for.
    For the first term (\ref{term_1}), we define $\Delta_{n,1}$ as
    \begin{align*}
        \Delta_{n,1} = \frac{1}{n} \sum_{i = 1}^{n}  [m(\bX'_i) - m(\Xpermi')]^2 - \E[(m(\bX) - m(\Xperm))^2].
    \end{align*}
    Clearly, we have $\E[\Delta_{n,1}] = 0$. Its variance writes
    \begin{align*}
        \V[\Delta_{n,1}] = \frac{1}{n^2} \E\big[ \sum_{i,k = 1}^{n}  &([m(\bX_i) - m(\Xpermi)]^2 - \E[(m(\bX) - m(\Xperm))^2]) \\ &\times([m(\bX_k) - m(\Xpermk)]^2 - \E[(m(\bX) - m(\Xperm))^2]) \big].
    \end{align*}
    Because of the permutation, each element of the sum is dependent on only two other terms. Therefore, only $3n$ terms of the double sum are not null, and because $m$ is bounded (continuous on a compact), we get
    \begin{align*}
        \V[\Delta_{n,1}] \leq \frac{3}{n} \times 64 ||m||_{\infty}^4.
    \end{align*}
    Thus, $\lim_{n \to \infty} \V[\Delta_{n,1}] = 0$, which proves $\mathbb{L}^2$-convergence of $\Delta_{n,1}$ towards $\E[\Delta_{n,1}] = 0$.
    We can handle the fourth term (\ref{term_4}) in the same way. 
    For the second term (\ref{term_2}), by symmetry,
    \begin{align*}
        \E\big[\frac{1}{n}\sum_{i = 1}^{n}[m(\Xpermi') - m_{M,n}(\Xpermi', \bTheta_{M})]^2\big]
        = \E[ (m(\Xperm) - m_{M,n}(\Xperm, \bTheta_{M}))^2],
    \end{align*}
    which tends to zero according to (\ref{Z_consistency}). The sixth term (\ref{term_6}) is handled similarly using (\ref{X_consistency}).
    Since $m$ is bounded, we can bound the third term (\ref{term_3})
    \begin{align*}
        \E\big[\big|\frac{2}{n}\sum_{i = 1}^{n}[m(\bX'_i) - m(\Xpermi')]&[m(\Xpermi') - m_{M,n}(\Xpermi', \bTheta_{M})]\big|\big] \\
        &\leq 4\|m\|_{\infty} \E[|m(\Xperm) - m_{M,n}(\Xperm, \bTheta_{M})|],
    \end{align*}
    and since $\mathbb{L}^2$ convergence implies $\mathbb{L}^1$ convergence, we use (\ref{Z_consistency}) to obtain the convergence towards $0$ of this third term (\ref{term_3}).
    For the fifth term (\ref{term_5}) we first apply the triangle inequality, and by symmetry we get
    \begin{align*}
        \E\big[\big|\frac{2}{n}\sum_{i = 1}^{n}\varepsilon'_i[m(\Xpermi') - m_{M,n}(\Xpermi', \bTheta_{M})]\big|\big] 
        &\leq 2\E[|\varepsilon'(m(\Xperm) - m_{M,n}(\Xperm, \bTheta_{M}))|] \\
        &\leq 2\E[|\varepsilon'|]\E[|m(\Xperm) - m_{M,n}(\Xperm, \bTheta_{M})|],
    \end{align*}
    which tends to zero according to (\ref{Z_consistency}). Similarly, the last term (\ref{term_7}) is handled with (\ref{X_consistency}).
    Gathering all previous convergence results on \eqref{term_1}-\eqref{term_7}, we have for all $M$, for all $j \in \{1, \hdots, p\}$, 
    \begin{align*}
	     \widehat{\textrm{MDA}}_{M,n}^{(TT)}(X^{(j)}) \overset{\mathbb{L}^1}{\longrightarrow} \E[(m(\bX) - m(\Xperm))^2].
	\end{align*}
\end{proof}

\begin{proof}[Proof of Lemma \ref{lemme_rf_consistency}]
We assume that Assumptions \ref{A1}, \ref{A2}, and \ref{A3} are satisfied, and fix $j \in \{1,\hdots,p\}$ and $M \in \mathbb{N}^{\star}$.
We first introduce the infinite forest estimate $m_n(\bx)$ defined as $m_n(\bx) = \E_{\Theta}[m_{n}(\bx, \Theta)]$ where $m_{n}(\bx, \Theta)$ is the randomized CART estimate.

Theorem 1 from \citet{scornet2015consistency} states the $\mathbb{L}^2$-consistency of infinite random forests. It relies on Assumption \ref{A3} for the asymptotic regime of $a_n$ and $t_n$, and on a modified version of \ref{A1}, where the regression function is additive and $\bX$ is uniformly distributed over $[0,1]^p$. 
Here, we extend this result to any continuous regression function and any positive distribution for $\bX$ with support on the unit cube.
First, the extension to the case where $\bX$ has any distribution bounded from above and below by positive constants can be easily obtained by several technical adaptations as already highlighted in \citet{scornet2020trees}. Secondly, notice that the additive structure of the regression function is only required in \citet{scornet2015consistency} to show the consistency of a theoretical randomized CART.
Therefore we can drop the additivity assumption and replace it by Assumption \ref{A2}.
Overall, we can extend Theorem 1 from \citet{scornet2015consistency}: provided that Assumptions \ref{A1}, \ref{A2}, and \ref{A3} are satisfied, we have
\begin{align} \label{eq_X}
    \lim \limits_{n \to \infty} \E[(m_n(\bX) - m(\bX))^2] = 0.
\end{align}
Next, this result needs to be extended when the query point $\bX$  is replaced by $\Xperm$.
From Assumption \ref{A1}, $\bX$ admits a density $f_X$ over $[0,1]^p$. By construction, the random vector $\Xperm$ is the vector $\bX$ where the $j$-th component is replaced by an independent copy of $X^{(j)}$. Therefore $\Xperm$ admits a density $f_{\pi_j}$, which is the product of the densities of $X^{(j)}$ and $\bX^{(-j)}$, i.e., for $\bx \in [0,1]^p$,
\begin{align} \label{eq_fZ}
    f_{\pi_j}(\bx) = \int_{[0,1]^{p-1}} f_X(\bx)d\bx^{(-j)} \times \int_{[0,1]} f_X(\bx)d\bx^{(j)}.
\end{align}
From Assumption \ref{A1}, $f_X$ is bounded from above and below by positive constants. Thus, it exists $c_1, c_2 > 0$ such that for all $\bx \in [0,1]^p$,
\begin{align} \label{eq_fX}
    c_1 \leq f_X(\bx) \leq c_2.
\end{align}
Combining (\ref{eq_fX}) and (\ref{eq_fZ}), we obtain that for all $\bx \in [0,1]^p$, $
    c_1^2 \leq f_{\pi_j}(\bx) \leq c_2^2$, and consequently, 
\begin{align*}
    \underset{\bx \in [0,1]^p}{\sup}\frac{f_{\pi_j}(\bx)}{f_X(\bx)} \leq \frac{c_2^2}{c_1}.
\end{align*}
Now, we write
\begin{align*}
    \E[(m_n(\Xperm) - m(\Xperm))^2|\Dn] &= 
    \int_{[0,1]^p} (m_n(\bx) - m(\bx))^2 f_{\pi_j}(\bx) d\bx \\
    &= \int_{[0,1]^p} (m_n(\bx) - m(\bx))^2 f_X(\bx) \frac{f_{\pi_j}(\bx)}{f_X(\bx)} d\bx \\
    &\leq \frac{c_2^2}{c_1} \int_{[0,1]^p} (m_n(\bx) - m(\bx))^2 f_X(\bx) d\bx \\
    &\leq \frac{c_2^2}{c_1}\E[(m_n(\bX) - m(\bX))^2|\Dn].
\end{align*}
Taking expectations on both sides and using (\ref{eq_X}), we finally obtain
\begin{align} \label{eq_Z}
        \lim \limits_{n \to \infty} \E[(m_n(\Xperm) - m(\Xperm))^2] = 0.
\end{align}

Equations (\ref{eq_X}) and (\ref{eq_Z}) state that infinite forests evaluated at $\bX$ or $\Xperm$ are $\mathbb{L}^2$ consistent.
The first of these two results can be extended to get the consistency of a single randomized CART $m_n(\bX, \Theta)$, as shown in \citet{scornet2015consistency} by an easy adaptation of the infinite forest case. Formally, we obtain
\begin{align} \label{eq_X_CART}
    \lim \limits_{n \to \infty}
    \E[(m_n(\bX, \Theta) - m(\bX))^2] = 0.
\end{align}
The exact same reasoning as for the infinite forest above applies to get the extension to $\Xperm$, and thus, we have
\begin{align} \label{eq_Xperm_CART}
    \lim \limits_{n \to \infty}
    \E[(m_n(\Xperm, \Theta) - m(\Xperm))^2] = 0.
\end{align}

Now, we expand the final quantity of interest $\E[(m_{M,n}(\bX, \bTheta_{M}) - m(\bX))^2]$ (and its counterpart for $\Xperm$):
\begin{align*}
    \E[(m_{M,n}(\bX,& \bTheta_{M}) - m(\bX))^2] \\ = &\E\big[\big(\frac{1}{M} \sum_{\ell = 1}^{M} m_n(\bX, \bTheta_{\ell}) - m(\bX)\big)^2\big] \\
    = &\E\big[\E\big[\big(\frac{1}{M} \sum_{\ell = 1}^{M} m_n(\bX, \bTheta_{\ell}) - m(\bX)\big)^2\big|\bX,\Dn\big]\big] \\
    = &\frac{1}{M^2} \E\big[\E\big[\sum_{\ell,\ell' = 1}^{M} [m_n(\bX, \bTheta_{\ell}) - m(\bX)][m_n(\bX, \bTheta_{\ell'}) - m(\bX)] \big|\bX,\Dn\big]\big] \\
     = &\frac{1}{M^2} \E\big[\E\big[ \sum_{\ell = 1}^{M} \big(m_n(\bX, \bTheta) - m(\bX)\big)^2 \big|\bX,\Dn\big]\big] \\
      &+ \frac{1}{M^2} \E\big[\E\big[\sum_{\ell \neq \ell'} [m_n(\bX, \bTheta_{\ell}) - m(\bX)][m_n(\bX, \bTheta_{\ell'}) - m(\bX)] \big|\bX,\Dn\big]\big].
\end{align*}
Conditional on $(\bX, \Dn)$, the random variables $m_n(\bX, \bTheta_{\ell})$ for $\ell = 1, \hdots, M$ are iid. Hence
\begin{align} \label{risk_forest_cart}
    \E[(m_{M,n}(\bX,& \bTheta_{M}) - m(\bX))^2] \nonumber\\
     = &\frac{1}{M} \E\big[\E\big[ \big(m_n(\bX, \bTheta) - m(\bX)\big)^2 \big|\bX,\Dn\big]\big] \nonumber\\
      &+ \frac{1}{M^2} \E\big[\sum_{\ell \neq \ell'} \big(\E[m_n(\bX, \bTheta_{\ell})\big|\bX,\Dn] - m(\bX)\big) \big(\E[m_n(\bX, \bTheta_{\ell'})\big|\bX,\Dn\big] - m(\bX) \big) \big] \nonumber\\
      = &\frac{1}{M} \E\big[\big(m_n(\bX, \bTheta) - m(\bX)\big)^2 \big] + \big(1 - \frac{1}{M}\big) \E\big[\big(m_n(\bX) - m(\bX)\big)^2 \big].
\end{align}
Using (\ref{eq_X}) and (\ref{eq_X_CART}), we obtain the final result
\begin{align*}
    \lim \limits_{n \to \infty} \E[(m_{M,n}(\bX, \bTheta_{M}) - m(\bX))^2] = 0,
\end{align*}
which also holds for $\Xperm$, using (\ref{eq_Z}) and (\ref{eq_Xperm_CART}): 
\begin{align*}
    \lim \limits_{n \to \infty} \E[(m_{M,n}(\Xperm, \bTheta_{M}) - m(\Xperm))^2] = 0.
\end{align*}

\end{proof}

\subsection{Proof of Theorem \ref{thm_MDA}-(ii)} \label{sec_proof_BC}

Theorem \ref{thm_MDA}-(i) can be quite easily adapted to the BC-MDA (ii).

\begin{proof}[Proof of Theorem \ref{thm_MDA}-(ii)]
We assume that Assumptions \ref{A1}-\ref{A3} are satisfied, and fix $j \in \{1,\hdots,p\}$ and $M \in \mathbb{N}^{\star}$. 
Recall that the Breiman-Cutler MDA is formally defined by
\begin{align*}
    \widehat{\textrm{MDA}}_{M,n}^{(BC)}(X^{(j)}) =
    \frac{1}{M} \sum_{\ell = 1}^{M}
    \frac{1}{N_{n,\ell}} \sum_{i = 1}^{n}  \big[(Y_i - m_{n}(\bX_{i,\pi_{j\ell}}, \Theta_{\ell}))^2 - (Y_i - m_{n}(\bX_i, \Theta_{\ell}))^2\big]
    \mathds{1}_{i \notin \Theta_{\ell}^{(S)}},
\end{align*}
where $N_{n,\ell} = \sum_{i = 1}^{n}  \mathds{1}_{i \notin \Theta_{\ell}^{(S)}}$ is the size of the out-of-bag sample of the $\ell$-th tree.

Since $a_n$ observations are subsampled without replacement prior to the construction of each tree, all out-of-bag samples have the same constant size of $N_{n,\ell} = n - a_n$. Using the triangle inequality, we have
\begin{align*}
    \E\big[\big|\widehat{\textrm{MDA}}_{M,n}^{(BC)}(X^{(j)}) - \E[(m(\bX) - &m(\Xperm))^2]\big|\big] \\
    \leq 
    \frac{1}{M} \sum_{\ell = 1}^{M}
    \frac{1}{n - a_n}  \E\big[&\big|\sum_{i = 1}^{n} [(Y_i - m_{n}(\bX_{i,\pi_{j\ell}}, \Theta_{\ell}))^2 - (Y_i - m_{n}(\bX_i, \Theta_{\ell}))^2 \\
    & - \E[(m(\bX) - m(\Xperm))^2]] \mathds{1}_{i \notin \Theta_{\ell}^{(S)}} \big| \big] ,
\end{align*}
and by symmetry, this boils down to 
\begin{align*}
    \E\big[\big|\widehat{\textrm{MDA}}_{M,n}^{(BC)}(X^{(j)}) - \E[(m(\bX) - &m(\Xperm))^2]\big|\big] \\
    \leq 
    \frac{1}{n - a_n} \E\big[&\big|\sum_{i = 1}^{n} [(Y_i - m_{n}(\bX_{i,\pi_{j1}}, \Theta_{1}))^2 - (Y_i - m_{n}(\bX_i, \Theta_{1}))^2 \\
    & - \E[(m(\bX) - m(\Xperm))^2]]
     \mathds{1}_{i \notin \Theta_{1}^{(S)}}\big|\big].
\end{align*}
Next, we expand the sum in the right hand side and obtain a similar decomposition as the one in the proof of Theorem~\ref{thm_MDA}-(i),
\begin{align*}
    \frac{1}{n - a_n} \sum_{i = 1}^{n} [(Y_i - m_{n}(&\bX_{i,\pi_{j1}}, \Theta_{1}))^2 - (Y_i - m_{n}(\bX_i, \Theta_{1}))^2]\mathds{1}_{i \notin \Theta_{1}^{(S)}}\\[-0.5em]
    = \frac{1}{n - a_n} \sum_{i = 1}^{n} & [([m(\bX_i) - m(\bX_{i,\pi_{j1}})] + [m(\bX_{i,\pi_{j1}}) - m_{n}(\bX_{i,\pi_{j1}}, \Theta_{1})] + \varepsilon_i)^2 \\[-0.5em]
    &- ([m(\bX_i) - m_{n}(\bX_i, \Theta_{1})] + \varepsilon_i)^2] \mathds{1}_{i \notin \Theta_{1}^{(S)}}\\[-0.5em]
     = \frac{1}{n - a_n} \sum_{i = 1}^{n} &  [m(\bX_i) - m(\bX_{i,\pi_{j1}})]^2 \mathds{1}_{i \notin \Theta_{1}^{(S)}} \\[-0.5em]
     &+ [m(\bX_{i,\pi_{j1}}) - m_{n}(\bX_{i,\pi_{j1}}, \Theta_{1})]^2\mathds{1}_{i \notin \Theta_{1}^{(S)}} + \varepsilon_i^2\mathds{1}_{i \notin \Theta_{1}^{(S)}} \\
     & + 2[m(\bX_i) - m(\bX_{i,\pi_{j1}})][m(\bX_{i,\pi_{j1}}) - m_{n}(\bX_{i,\pi_{j1}}, \Theta_{1})]\mathds{1}_{i \notin \Theta_{1}^{(S)}}
     \\ & + 2 \varepsilon_i[m(\bX_i) - m(\bX_{i,\pi_{j1}})]\mathds{1}_{i \notin \Theta_{1}^{(S)}}
     \\ & + 2 \varepsilon_i[m(\bX_{i,\pi_{j1}}) - m_{n}(\bX_{i,\pi_{j1}}, \Theta_{1})]\mathds{1}_{i \notin \Theta_{1}^{(S)}}
     \\ & - [m(\bX_i) - m_{n}(\bX_i, \Theta_{1})]^2\mathds{1}_{i \notin \Theta_{1}^{(S)}} - \varepsilon_i^2\mathds{1}_{i \notin \Theta_{1}^{(S)}}\\ & -  2 \varepsilon_i[m(\bX_i) - m_{n}(\bX_i, \Theta_{1})]\mathds{1}_{i \notin \Theta_{1}^{(S)}}.
    \end{align*}
    Thus, we have the following bound
    \begin{align}
    \E\big[\big|&\widehat{\textrm{MDA}}_{M,n}^{(BC)}(X^{(j)}) - \E[(m(\bX) -   m(\Xperm))^2]\big|\big] \nonumber \\
     &\leq \E\big[\big|\frac{1}{n - a_n} \sum_{i = 1}^{n}  ([m(\bX_i) - m(\bX_{i,\pi_{j1}})]^2 - \E[(m(\bX) - m(\Xperm))^2]) \mathds{1}_{i \notin \Theta_{1}^{(S)}} \big|\big] \label{term_BC_1}\\
     &+ \E\big[\frac{1}{n - a_n}\sum_{i = 1}^{n}[m(\bX_{i,\pi_{j1}}) - m_{n}(\bX_{i,\pi_{j1}}, \Theta_{1})]^2\mathds{1}_{i \notin \Theta_{1}^{(S)}}\big] \label{term_BC_2}\\
     & + \E\big[\big|\frac{2}{n - a_n}\sum_{i = 1}^{n}[m(\bX_i) - m(\bX_{i,\pi_{j1}})][m(\bX_{i,\pi_{j1}}) - m_{n}(\bX_{i,\pi_{j1}}, \Theta_{1})]\mathds{1}_{i \notin \Theta_{1}^{(S)}}\big|\big]
     \label{term_BC_3}\\ & + \E\big[\big|\frac{2}{n - a_n}\sum_{i = 1}^{n}\varepsilon_i[m(\bX_i) - m(\bX_{i,\pi_{j1}})]\mathds{1}_{i \notin \Theta_{1}^{(S)}}\big|\big]
     \label{term_BC_4}\\ & + \E\big[\big|\frac{2}{n - a_n}\sum_{i = 1}^{n}\varepsilon_i[m(\bX_{i,\pi_{j1}}) - m_{n}(\bX_{i,\pi_{j1}}, \Theta_{1})]\mathds{1}_{i \notin \Theta_{1}^{(S)}}\big|\big]
     \label{term_BC_5}\\ & + \E\big[\frac{1}{n - a_n}\sum_{i = 1}^{n}[m(\bX_i) - m_{n}(\bX_i, \Theta_{1})]^2\mathds{1}_{i \notin \Theta_{1}^{(S)}}\big] \label{term_BC_6}\\ & +  \E\big[\big|\frac{2}{n - a_n}\sum_{i = 1}^{n}\varepsilon_i[m(\bX_i) - m_{n}(\bX_i, \Theta_{1})]\mathds{1}_{i \notin \Theta_{1}^{(S)}}\big|\big]. \label{term_BC_7}
    \end{align}

    Now, let us consider all the terms on the right hand side one by one.
    
    For the first term (\ref{term_BC_1}), we define $\Delta_{n,1}$ as
    \begin{align*}
        \Delta_{n,1} = \sum_{i = 1}^{n}  \frac{1}{n - a_n} ([m(\bX_i) - m(\bX_{i,\pi_{j1}})]^2 - \E[(m(\bX) - m(\Xperm))^2]) \mathds{1}_{i \notin \Theta_{1}^{(S)}}.
    \end{align*}
    Its expectation is
     \begin{align*}
        \E[\Delta_{n,1}] = &\E[\frac{n}{n - a_n} ([m(\bX_1) - m(\bX_{1,\pi_{j1}})]^2 - \E[(m(\bX) - m(\Xperm))^2]) \mathds{1}_{1 \notin \Theta_{1}^{(S)}}] \\
        &= \frac{n}{n - a_n} \E[(m(\bX_1) - m(\bX_{1,\pi_{j1}}))^2 - \E[(m(\bX) - m(\Xperm))^2]] \P(1 \notin \Theta_{1}^{(S)}) \\&=0.
    \end{align*}
    Next, observe that each term of the sum in $\Delta_{n,1}$ is dependent on two other terms because of the permutation of the $j$-th component, then we have $\V[\Delta_{n,1}] = O(1/(n - a_n))$. By Assumption \ref{A3}, $a_n/n < 1 - \kappa$ with a fixed $\kappa > 0$, thus $\V[\Delta_{n,1}] = O(1/n)$. Since $ \E[\Delta_{n,1}] = 0$ and $\lim_{n\to\infty} \V[\Delta_{n,1}] = 0$, $\Delta_{n,1}$ converges towards $0$ in $\mathbb{L}^2$, which implies $\mathbb{L}^1$-convergence.
    We can handle the fourth term (\ref{term_BC_4}) in the same way.
    For the second term (\ref{term_BC_2}),
    \begin{align*}
        \E\big[\frac{1}{n - a_n}&\sum_{i = 1}^{n}[m(\bX_{i,\pi_{j1}}) - m_{n}(\bX_{i,\pi_{j1}}, \Theta_{1})]^2\mathds{1}_{i \notin \Theta_{1}^{(S)}}\big]\\
        = &\sum_{i = 1}^{n}\E\big[[m(\bX_{i,\pi_{j1}}) - m_{n}(\bX_{i,\pi_{j1}}, \Theta_{1})]^2\big|i \notin \Theta_{1}^{(S)}\big] \frac{\P(i \notin \Theta_{1}^{(S)})}{n - a_n} \\
        = &\frac{1}{n} \sum_{i = 1}^{n}\E\big[[m(\bX_{i,\pi_{j1}}) - m_{n}(\bX_{i,\pi_{j1}}, \Theta_{1})]^2\big|i \notin \Theta_{1}^{(S)}\big]
    \end{align*}
    where the last equality results from $\P(i \notin \Theta_{1}^{(S)}) = (n-a_n)/n$.
    The conditioning event $\{i \notin \Theta_{1}^{(S)}\}$ means that the observation of index $i$ belongs to the out-of-bag sample. Thus, it is strictly equivalent to consider the tree trained with the sample $\Dn \setminus (\bX_i, Y_i)$ of size $n - 1$ with a subsampling size $a_n$. Furthermore, we can replace the query point $\bX_{i,\pi_{j1}}$ by $\Xperm$ because these two random vectors are iid and both independent of the training data of $m_{a_n,n-1}$. Then,
    \begin{align*}
        \E\big[\frac{1}{n - a_n}&\sum_{i = 1}^{n}[m(\bX_{i,\pi_{j1}}) - m_{n}(\bX_{i,\pi_{j1}}, \Theta_{1})]^2\mathds{1}_{i \notin \Theta_{1}^{(S)}}\big]\\
        = &\frac{1}{n} \sum_{i = 1}^{n}\E\big[[m(\Xperm) - m_{a_n,n-1}(\Xperm, \Theta)]^2\big] \\
        =& \E[ (m(\Xperm) - m_{a_n,n-1}(\Xperm, \Theta))^2],
    \end{align*}
    which tends to zero according to the second statement in Lemma \ref{lemme_rf_consistency} for $M=1$. The sixth term (\ref{term_BC_6}) is handled similarly using the first part of Lemma \ref{lemme_rf_consistency}.
    Since $m$ is bounded, we can bound the third term (\ref{term_BC_3})
    \begin{align*}
        \E\big[\big|\frac{2}{n - a_n}\sum_{i = 1}^{n}&[m(\bX_i) - m(\bX_{i,\pi_{j1}})][m(\bX_{i,\pi_{j1}}) - m_{n}(\bX_{i,\pi_{j1}}, \Theta_{1})]\mathds{1}_{i \notin \Theta_{1}^{(S)}}\big|\big] \\
        \leq & \frac{4||m||_{\infty}}{n - a_n} \E\big[\sum_{i = 1}^{n}\big|m(\bX_{i,\pi_{j1}}) - m_{n}(\bX_{i,\pi_{j1}}, \Theta_{1})\big|\times\mathds{1}_{i \notin \Theta_{1}^{(S)}}\big] \\
        \leq & \frac{4||m||_{\infty}}{n} \sum_{i = 1}^{n} \E\big[\big|m(\bX_{i,\pi_{j1}}) - m_{a_n,n-1}(\bX_{i,\pi_{j1}}, \Theta_{1})\big|\big] \\
        \leq & 4||m||_{\infty} \E\big[\big|m(\Xperm) - m_{a_n,n-1}(\Xperm, \Theta)\big|\big],
    \end{align*}
    which tends to zero according to Lemma~\ref{lemme_rf_consistency} (with $M = 1$).
    Similarly, for the fifth term (\ref{term_BC_5}), we have
    \begin{align*}
        \E\big[\big|\frac{2}{n - a_n}\sum_{i = 1}^{n}&\varepsilon_i[m(\bX_{i,\pi_{j1}}) - m_{n}(\bX_{i,\pi_{j1}}, \Theta_{1})]\mathds{1}_{i \notin \Theta_{1}^{(S)}}\big|\big] \\
        &\leq 2\E[|\varepsilon|] \E\big[\big|m(\Xperm) - m_{a_n,n-1}(\Xperm, \Theta)\big|\big],
    \end{align*}
    and the convergence towards $0$ is again given by Lemma \ref{lemme_rf_consistency}. The last term (\ref{term_BC_7}) is handled in the same way.
    Gathering all previous convergence results on \eqref{term_BC_1}-\eqref{term_BC_7}, we have for all $M$, for all $j \in \{1, \hdots, p\}$, 
    \begin{align*}
	     \widehat{\textrm{MDA}}_{M,n}^{(BC)}(X^{(j)}) \overset{\mathbb{L}^1}{\longrightarrow} \E[(m(\bX) - m(\Xperm))^2].
	\end{align*}

\end{proof}

\subsection{Proof of Theorems \ref{thm_MDA}-(iii) and Proposition \ref{prop_oob_risk}} \label{sec_proof_IK}

The obstacle in the asymptotic analysis of the IK-MDA arises from the randomness of $\Lambda_{n,i}$, which can even be empty. However, the quadratic risk of the OOB estimate can be bounded using the risk of the standard forest, as stated in the following Lemma.
\begin{lemma*} \label{lemma_oob_risk}
    If Assumption \ref{A1} is satisfied, for all $M \in \mathbb{N}^{\star}$ and $i \in \{1,\hdots,n\}$, we have
    \begin{align*}
        \E\big[ \big( m_{M,a_n,n}^{(OOB)}(\bX_i, \bTheta_{M}) - m(\bX_i) \big)^2 \mathds{1}_{|\Lambda_{n,i}| > 0} \big]
        \leq \frac{2}{1 - a_n/n}& \E\big[ \big(m_{M,a_n,n-1}(\bX, \bTheta_{M}) - m(\bX) \big)^2 \big].
    \end{align*}
\end{lemma*}
We can draw interesting insights from Lemma \ref{lemma_oob_risk}. First by construction, the OOB estimate aggregates a smaller number of trees than in the standard forest: $\E[|\Lambda_{n,i}|] = (1 - a_n/n) M$ trees in average. Therefore the risk of the standard forest is inflated by the coefficient $2/(1 - a_n/n) > 2$ to bound the OOB risk. Since the risk of the OOB estimate is bounded by the risk of the standard forest, the $\mathbb{L}^2$-consistency of random forests can be extended to the OOB estimate.
\begin{lemma*} \label{lemme_oob_consistency}
    If Assumptions \ref{A1}, \ref{A2}, and \ref{A3} are satisfied, for all $i \in \{1,\hdots,n\}$ and $M \in \mathbb{N}^{\star}$ we have
    \begin{align*}
    \lim \limits_{n \to \infty} \E[(m_{M,n}^{(OOB)}(\bX_i, \bTheta_{M}) - m(\bX_i))^2\mathds{1}_{|\Lambda_{n,i}| > 0} ] = 0, 
    \end{align*}
    and if Assumption \ref{A4} is additionally satisfied, for all $j \in \{1,\hdots,p\}$
    \begin{align*}
        \lim \limits_{n \to \infty} \E[(m_{M,n,\pi_{j}}^{(OOB)}(\bX_i, \bTheta_{M}) - \E[m(\Xpermi)|\bX_i^{(-j)}])^2\mathds{1}_{|\Lambda_{n,i}| > 0} ] = 0.
    \end{align*}
\end{lemma*}

To prove Lemma \ref{lemma_oob_risk} and \ref{lemme_oob_consistency}, we need the following Lemma \ref{lemma_tech}, proved at the end of the section.
\begin{lemma*} \label{lemma_tech}
    If $\delta_{M,n}$ and $\gamma_{M,n}$ are defined as
    \begin{align*}
        \delta_{M,n} = M^2 \E \Big[ \frac{1}{|\Lambda_{n,i}|^2} \big| 1,2 \in \Lambda_{n,i} \Big] \P(1,2 \in \Lambda_{n,i} ) \\
        \gamma_{M,n} = M^2 \E \Big[ \frac{1}{|\Lambda_{n,i}|^2} \big| 1 \in \Lambda_{n,i} \Big] \P(1 \in \Lambda_{n,i} ),
    \end{align*}
    for all $M \in \mathbb{N}\setminus\{0,1\}$, we have
    \begin{align*}
        \delta_{M,n} &\leq 1 \\
        \delta_{M,n} \leq \gamma_{M,n} &\leq \frac{2}{1 - \frac{a_n}{n}},
    \end{align*}
    and for a fixed sample size $n$,
    \begin{align*}
        1 - \delta_{M,n} = O\Big(\frac{1}{M}\Big).
    \end{align*}
\end{lemma*}
Then, we can deduce the consistency of the IK-MDA.

\begin{proof}[Proof of Theorem \ref{thm_MDA}-(iii)]
We assume that Assumptions \ref{A1}-\ref{A4} are satisfied, and fix $j \in \{1,\hdots,p\}$.
Recall that Ishwaran-Kogalur MDA is defined as
\begin{align*}
\widehat{\textrm{MDA}}_{M,n}^{(IK)}(X^{(j)}) =& \frac{1}{N_{M,n}} \sum_{i = 1}^{n} (Y_i - m_{M,n,\pi_{j}}^{(OOB)}(\bX_i, \bTheta_{M}))^2 - (Y_i - m_{M,n}^{(OOB)}(\bX_i, \bTheta_{M}))^2,
\end{align*}
where $N_{M,n} = \sum_{i = 1}^{n} \mathds{1}_{|\Lambda_{n,i}| > 0}$ is the number of points which do not belong to all trees, and
\begin{align*}
    &m_{M,n}^{(OOB)}(\bX_i, \bTheta_{M}) = \frac{1}{|\Lambda_{n,i}|} \sum_{\ell \in \Lambda_{n,i}} m_n(\bX_i, \Theta_{\ell}) \mathds{1}_{|\Lambda_{n,i}| > 0}, \\
    &m_{M,n,\pi_{j}}^{(OOB)}(\bX_{i}, \bTheta_{M}) = \frac{1}{|\Lambda_{n,i}|} \sum_{\ell \in \Lambda_{n,i}} m_n(\bX_{i,\pi_{j\ell}}, \Theta_{\ell}) \mathds{1}_{|\Lambda_{n,i}| > 0}.
\end{align*}
To lighten derivations, we define
$MDA_{IK}^{\star} =
\E[(m(\bX) - \E[m(\Xperm)|\bX^{(-j)}])^2]$.
We expand the following expression,
\begin{align*}
     \E\big[\big|&\widehat{\textrm{MDA}}_{M,n}^{(IK)}(X^{(j)}) - MDA_{IK}^{\star} \big|\big] \\
     = & \E\big[\big|\frac{1}{N_{M,n}} \sum_{i = 1}^{n}  \big[(Y_i - m_{M,n,\pi_{j}}^{(OOB)}(\bX_i, \bTheta_{M}))^2 - (Y_i - m_{M,n}^{(OOB)}(\bX_i, \bTheta_{M}))^2  - MDA_{IK}^{\star}\big]
    \mathds{1}_{|\Lambda_{n,i}| > 0} \big|\big].
\end{align*}
Observe that $N_{M,n}$ is bounded between $n$ and $n - a_n$, and consequently
\begin{align*}
     \E\big[\big|&\widehat{\textrm{MDA}}_{M,n}^{(IK)}(X^{(j)}) - MDA_{IK}^{\star} \big|\big] \\
     \leq & \E\big[\big|\frac{1}{n - a_n} \sum_{i = 1}^{n}  \big[(Y_i - m_{M,n,\pi_{j}}^{(OOB)}(\bX_i, \bTheta_{M}))^2 - (Y_i - m_{M,n}^{(OOB)}(\bX_i, \bTheta_{M}))^2  - MDA_{IK}^{\star}\big]
    \mathds{1}_{|\Lambda_{n,i}| > 0} \big|\big].
\end{align*}
Then, we follow the proof of Theorem \ref{thm_MDA}-(i) and (ii) with a similar decomposition of the sum of the above expression
\begin{align*}
    \sum_{i = 1}^{n} [(Y_i - & m_{M,n,\pi_{j}}^{(OOB)}(\bX_i, \bTheta_{M}))^2 - (Y_i - m_{M,n}^{(OOB)}(\bX_i, \bTheta_{M}))^2 - MDA_{IK}^{\star} ]\mathds{1}_{|\Delta_{n,i}|>0}\\[-0.5em]
    = \sum_{i = 1}^{n} & [([m(\bX_i) - \E[m(\Xpermi)|\bX_i^{(-j)}]] + [\E[m(\Xpermi)|\bX_i^{(-j)}] - m_{M,n,\pi_{j}}^{(OOB)}(\bX_i, \bTheta_{M})] + \varepsilon_i)^2 \\[-0.5em]
    &- ([m(\bX_i) - m_{M,n}^{(OOB)}(\bX_i, \bTheta_{M})] + \varepsilon_i)^2 - MDA_{IK}^{\star}] \mathds{1}_{|\Delta_{n,i}|>0}\\[-0.5em]
     = \sum_{i = 1}^{n} &  ([m(\bX_i) - \E[m(\Xpermi)|\bX_i^{(-j)}]]^2 - MDA_{IK}^{\star}) \mathds{1}_{|\Delta_{n,i}|>0} \\[-0.5em]
     &+ [\E[m(\Xpermi)|\bX_i^{(-j)}] - m_{M,n,\pi_{j}}^{(OOB)}(\bX_i, \bTheta_{M})]^2\mathds{1}_{|\Delta_{n,i}|>0} + \varepsilon_i^2\mathds{1}_{|\Delta_{n,i}|>0} \\
     & + 2[m(\bX_i) - \E[m(\Xpermi)|\bX_i^{(-j)}]][\E[m(\Xpermi)|\bX_i^{(-j)}] - m_{M,n,\pi_{j}}^{(OOB)}(\bX_i, \bTheta_{M})]\mathds{1}_{|\Delta_{n,i}|>0}
     \\ & + 2 \varepsilon_i[m(\bX_i) - \E[m(\Xpermi)|\bX_i^{(-j)}]]\mathds{1}_{|\Delta_{n,i}|>0}
     \\ & + 2 \varepsilon_i[\E[m(\Xpermi)|\bX_i^{(-j)}] - m_{M,n,\pi_{j}}^{(OOB)}(\bX_i, \bTheta_{M})]\mathds{1}_{|\Delta_{n,i}|>0}
     \\ & - [m(\bX_i) - m_{M,n}^{(OOB)}(\bX_i, \bTheta_{M})]^2\mathds{1}_{|\Delta_{n,i}|>0} - \varepsilon_i^2\mathds{1}_{|\Delta_{n,i}|>0}\\ & -  2 \varepsilon_i[m(\bX_i) - m_{M,n}^{(OOB)}(\bX_i, \bTheta_{M})]\mathds{1}_{|\Delta_{n,i}|>0}.
\end{align*}
We then obtain the following bound
\begin{align}
    \E\big[\big|\widehat{\textrm{MDA}}_{M,n}^{(IK)}(X^{(j)})& - MDA_{IK}^{\star} \big|\big] \nonumber \\
     \leq \E\big[\big|\frac{1}{n - a_n} \sum_{i = 1}^{n}&  ([m(\bX_i) - \E[m(\Xpermi)|\bX_i^{(-j)}]]^2 - MDA_{IK}^{\star}) \mathds{1}_{|\Lambda_{n,i}| > 0} \big|\big] \label{term_IK_1}\\
     + \E\big[\frac{1}{n - a_n}\sum_{i = 1}^{n}&[\E[m(\Xpermi)|\bX_i^{(-j)}] - m_{M,n,\pi_{j}}^{(OOB)}(\bX_i, \bTheta_{M})]^2\mathds{1}_{|\Lambda_{n,i}| > 0}\big] \label{term_IK_2}\\
     + \E\big[\big|\frac{2}{n - a_n}\sum_{i = 1}^{n}&[m(\bX_i) - \E[m(\Xpermi)|\bX_i^{(-j)}]] \nonumber\\
     &\times [\E[m(\Xpermi)|\bX_i^{(-j)}] - m_{M,n,\pi_{j}}^{(OOB)}(\bX_i, \bTheta_{M})]\mathds{1}_{|\Lambda_{n,i}| > 0}\big|\big]
     \label{term_IK_3}\\ + \E\big[\big|\frac{2}{n - a_n}\sum_{i = 1}^{n}&\varepsilon_i[m(\bX_i) - \E[m(\Xpermi)|\bX_i^{(-j)}]]\mathds{1}_{|\Lambda_{n,i}| > 0}\big|\big]
     \label{term_IK_4}\\ + \E\big[\big|\frac{2}{n - a_n}\sum_{i = 1}^{n}&\varepsilon_i[\E[m(\Xpermi)|\bX_i^{(-j)}] - m_{M,n,\pi_{j}}^{(OOB)}(\bX_i, \bTheta_{M})]\mathds{1}_{|\Lambda_{n,i}| > 0}\big|\big]
     \label{term_IK_5}\\ + \E\big[\frac{1}{n - a_n}\sum_{i = 1}^{n}&[m(\bX_i) - m_{M,n}^{(OOB)}(\bX_i, \bTheta_{M})]^2\mathds{1}_{|\Lambda_{n,i}| > 0}\big] \label{term_IK_6}\\ +  \E\big[\big|\frac{2}{n - a_n}\sum_{i = 1}^{n}&\varepsilon_i[m(\bX_i) - m_{M,n}^{(OOB)}(\bX_i, \bTheta_{M})]\mathds{1}_{|\Lambda_{n,i}| > 0}\big|\big]. \label{term_IK_7}
    \end{align}
    Now, let us consider all the terms on the right hand side one by one.
       For the first term (\ref{term_IK_1}), we can rewrite
    \begin{align*}
        \frac{1}{n - a_n} &\sum_{i = 1}^{n} ([m(\bX_i) - \E[m(\Xpermi)|\bX_i^{(-j)}]]^2 - MDA_{IK}^{\star}) \mathds{1}_{|\Lambda_{n,i}| > 0} \\
        = & \frac{n}{n - a_n} \frac{1}{n} \sum_{i = 1}^{n} ([m(\bX_i) - \E[m(\Xpermi)|\bX_i^{(-j)}]]^2 - MDA_{IK}^{\star}) \mathds{1}_{|\Lambda_{n,i}| > 0}, 
    \end{align*}
    and the multiplicative term in front $n/(n-a_n)$ is upper bounded by $1/\kappa > 0$ by Assumption \ref{A3}. Next, we can apply the strong law of large numbers to show that the sum converges almost surely towards
    \begin{align*}
        \E\big[& ([m(\bX_1) - \E[m(\bX_{1,\pi_j})|\bX_1^{(-j)}]]^2 - MDA_{IK}^{\star}) \mathds{1}_{|\Lambda_{n,1}| > 0} \big] \\
        & = \E\big[([m(\bX_1) - \E[m(\bX_{1,\pi_j})|\bX_1^{(-j)}]]^2 - MDA_{IK}^{\star}) \big] \P(|\Lambda_{n,1}| > 0) \\
        & = 0.
    \end{align*}
    Since almost sure convergence implies $\mathbb{L}^1$-convergence, the first term (\ref{term_IK_1}) converges towards $0$. 
    The fourth term (\ref{term_IK_4}) is handled similarly with the strong law of large number since the noise is centered and independent of $\Dn$.
    The second term
    \begin{align*}
        \E\big[\frac{1}{n - a_n}&\sum_{i = 1}^{n}[\E[m(\Xpermi)|\bX_i^{(-j)}] - m_{M,n,\pi_{j}}^{(OOB)}(\bX_i, \bTheta_{M})]^2\mathds{1}_{|\Lambda_{n,i}| > 0}\big] \\
        =& \frac{n}{n - a_n} \E\big[(\E[m(\bX_{1,\pi_{j}})|\bX_1^{(-j)}] - m_{M,n,\pi_{j}}^{(OOB)}(\bX_1, \bTheta_{M}))^2\mathds{1}_{|\Lambda_{n,1}| > 0}\big],
    \end{align*}
    converges towards $0$ from the second part of Lemma \ref{lemme_oob_consistency} and because $n/(n-a_n) < 1/\kappa$. The sixth term (\ref{term_IK_6}) is handled identically using the first part of Lemma \ref{lemme_oob_consistency}.
    For the third term (\ref{term_IK_3}), since $m$ is bounded (continuous on a compact), we have
    \begin{align*}
    \E\big[\big|\frac{2}{n - a_n}\sum_{i = 1}^{n}[m(\bX_i)& - \E[m(\Xpermi)|\bX_i^{(-j)}]] \\
     &\times [\E[m(\Xpermi)|\bX_i^{(-j)}] - m_{M,n,\pi_{j}}^{(OOB)}(\bX_i, \bTheta_{M})]\mathds{1}_{|\Lambda_{n,i}| > 0}\big|\big] \\
    \leq \frac{4n||m||_{\infty}}{n - a_n}& \E\big[\big|\E[m(\bX_{1,\pi_{j}})|\bX_1^{(-j)}] - m_{M,n,\pi_{j}}^{(OOB)}(\bX_1, \bTheta_{M})\big|\mathds{1}_{|\Lambda_{n,1}| > 0}\big],
    \end{align*}
    which converges towards $0$ by Lemma \ref{lemme_oob_consistency}.
    Similarly, for the fifth (\ref{term_IK_5}) and seventh (\ref{term_IK_7}) terms, we have the following bound
    \begin{align*}
        \E\big[\big|\frac{2}{n - a_n}\sum_{i = 1}^{n}&\varepsilon_i[\E[m(\Xpermi)|\bX_i^{(-j)}] - m_{M,n,\pi_{j}}^{(OOB)}(\bX_i, \bTheta_{M})]\mathds{1}_{|\Lambda_{n,i}| > 0}\big|\big] \\
        \leq& \frac{2n}{n - a_n}E[|\varepsilon|] \E\big[\big|\E[m(\bX_{1,\pi_{j}})|\bX_1^{(-j)}] - m_{M,n,\pi_{j}}^{(OOB)}(\bX_1, \bTheta_{M})\big|\mathds{1}_{|\Lambda_{n,1}| > 0}\big],
    \end{align*}
    and we conclude using Lemma \ref{lemme_oob_consistency} again. 
    Overall, we have 
	\begin{align*}
	     \widehat{\textrm{MDA}}_{M,n}^{(IK)}(X^{(j)}) \overset{\mathbb{L}^1}{\longrightarrow} \E[(m(\bX) - \E[m(\Xperm)|\bX^{(-j)}])^2].
	\end{align*}
\end{proof}

\begin{proof}[Proof of Lemma \ref{lemma_oob_risk}]
    We assume that Assumption \ref{A1} is satisfied, and consider $i \in \{1,\hdots,n\}$ and $M \in \mathbb{N}^{\star}$.
    To prove the first part of Lemma~\ref{lemma_oob_risk}, we begin with and expansion of the OOB estimate
    \begin{align*}
        \E\big[\big( m_{M,n}^{(OOB)}(\bX_i, \bTheta_{M})& - m(\bX_i) \big)^2 \big| |\Lambda_{n,i}| > 0\big] \\ =& \E\Big[ \Big( \frac{1}{|\Lambda_{n,i}|} \sum_{\ell \in \Lambda_{n,i}} m_n(\bX_i, \Theta_{\ell}) \mathds{1}_{|\Lambda_{n,i}| > 0}
        - m(\bX_i) \Big)^2 \big| |\Lambda_{n,i}| > 0\Big] \\
        =& \E\Big[ \Big( \frac{1}{|\Lambda_{n,i}|} \sum_{\ell =1}^{M} [m_n(\bX_i, \Theta_{\ell}) 
        - m(\bX_i)]\mathds{1}_{\ell \in \Lambda_{n,i}} \Big)^2 \big| |\Lambda_{n,i}| > 0\Big]. \\
    \end{align*}
    Now, we expand the square with a double sum,
    \begin{align*}
        \E\big[\big( m_{M,n}^{(OOB)}&(\bX_i, \bTheta_{M}) - m(\bX_i) \big)^2 \big| |\Lambda_{n,i}| > 0\big] \\
        =& \sum_{\ell, \ell'=1}^{M} \E\Big[ \frac{1}{|\Lambda_{n,i}|^2} [m_n(\bX_i, \Theta_{\ell}) 
        - m(\bX_i)][m_n(\bX_i, \Theta_{\ell'}) 
        - m(\bX_i)]\mathds{1}_{\ell, \ell' \in \Lambda_{n,i}} \big| |\Lambda_{n,i}| > 0\Big] \\
        =& \sum_{\ell, \ell'=1}^{M} \E\Big[ \frac{1}{|\Lambda_{n,i}|^2} [m_n(\bX_i, \Theta_{\ell}) 
        - m(\bX_i)][m_n(\bX_i, \Theta_{\ell'}) 
        - m(\bX_i)] \big| \ell, \ell' \in \Lambda_{n,i} \Big] \\ & \hspace*{1.5cm} \times \P\big(\ell, \ell' \in \Lambda_{n,i} \big| | \Lambda_{n,i}| > 0\big).
    \end{align*}
    Observe that conditionally on $\{ \ell, \ell' \in \Lambda_{n,i} \}$, $\Lambda_{n,i}$ only depends on $\{\Theta_{k}, k \in \{1,\hdots,M\} \setminus \{\ell,\ell'\} \}$. This means that $\Lambda_{n,i}$ and $[m_n(\bX_i, \Theta_{\ell}) - m(\bX_i)][m_n(\bX_i, \Theta_{\ell'}) - m(\bX_i)]$ are independent conditionally on $\{ \ell, \ell' \in \Lambda_{n,i} \}$. We can then write
     \begin{align*}
        \E\big[\big( m_{M,n}^{(OOB)}&(\bX_i, \bTheta_{M}) - m(\bX_i) \big)^2 \big| |\Lambda_{n,i}| > 0\big] \P(|\Lambda_{n,i}| > 0) \\
        = \sum_{\ell, \ell'=1}^{M} & \E\Big[ \frac{1}{|\Lambda_{n,i}|^2}\big| \ell, \ell' \in \Lambda_{n,i} \Big] \P\big(\ell, \ell' \in \Lambda_{n,i} \big| | \Lambda_{n,i}| > 0\big) \P(|\Lambda_{n,i}| > 0) \\ & \times \E\big[[m_n(\bX_i, \Theta_{\ell}) 
        - m(\bX_i)][m_n(\bX_i, \Theta_{\ell'}) 
        - m(\bX_i)] \big| \ell, \ell' \in \Lambda_{n,i} \big].  \\
        = \sum_{\ell, \ell'=1}^{M} & \E\Big[ \frac{1}{|\Lambda_{n,i}|^2}\big| \ell, \ell' \in \Lambda_{n,i} \Big] \P\big(\ell, \ell' \in \Lambda_{n,i} \big) \\ & \times \E\big[[m_n(\bX_i, \Theta_{\ell}) 
        - m(\bX_i)][m_n(\bX_i, \Theta_{\ell'}) 
        - m(\bX_i)] \big| \ell, \ell' \in \Lambda_{n,i} \big]. 
    \end{align*}
    Since $|\Lambda_{n,i}|$ is a binomial distribution, $\E\big[ \frac{1}{|\Lambda_{n,i}|^2}\big| \ell, \ell' \in \Lambda_{n,i} \big] \P(\ell, \ell' \in \Lambda_{n,i})$ takes the same value for each pair of distinct $\ell, \ell'$ and any sample $i \in \{1,\hdots,n\}$. Similarly for the case $\ell = \ell'$, $\E\big[ \frac{1}{|\Lambda_{n,i}|^2}\big| \ell \in \Lambda_{n,i} \big] \P(\ell \in \Lambda_{n,i})$ is constant when $\ell$ varies. 
    Therefore, we introduce
    \begin{align*}
        \delta_{M,n} = M^2 \E\Big[ \frac{1}{|\Lambda_{n,i}|^2}\big| \ell, \ell' \in \Lambda_{n,i} \Big] \P(\ell, \ell' \in \Lambda_{n,i} ),
    \end{align*}
    and 
    \begin{align*}
        \gamma_{M,n} = M^2 \E\Big[ \frac{1}{|\Lambda_{n,i}|^2}\big| \ell \in \Lambda_{n,i} \Big] \P(\ell \in \Lambda_{n,i}). 
    \end{align*}
    Then, we have
    \begin{align*}
        \E\big[\big( m_{M,n}^{(OOB)}&(\bX_i, \bTheta_{M}) - m(\bX_i) \big)^2 \big| |\Lambda_{n,i}| > 0\big] \P(|\Lambda_{n,i}| > 0) \\
        =&  \delta_{M,n} \frac{1}{M^2} \sum_{\ell, \ell' = 1}^{M} \E\big[[m_n(\bX_i, \Theta_{\ell}) 
        - m(\bX_i)][m_n(\bX_i, \Theta_{\ell'}) 
        - m(\bX_i)] \big| \ell, \ell' \in \Lambda_{n,i} \big]  \\
        &+ (\gamma_{M,n} - \delta_{M,n}) \frac{1}{M^2} \sum_{\ell = 1}^{M} \E\big[(m_n(\bX_i, \Theta_{\ell}) 
        - m(\bX_i))^2 \big| \ell \in \Lambda_{n,i} \big].
    \end{align*}
    Recall that $m_n(\bX_i, \Theta_{\ell})$ is the randomized CART estimate, built with $\Dn$ and $\Theta_{\ell}$, where the component $\Theta_{\ell}^{(S)}$ is used to subsample $a_n$ data points. When conditioned on $\{\ell \in \Lambda_{n,i}\}$ (i.e. $i \notin \Theta_{\ell}^{(S)}$), $m_n(\bX_i, \Theta_{\ell})$ can be seen as the CART estimate built with $\Dn \setminus \{(\bX_i, Y_i)\}$ and with the subsample size $a_{n}$, i.e., $m_{a_n,n-1}(\bX_i, \Theta_{\ell})$. Therefore, we have for all pairs $\ell, \ell'$,
    \begin{align} \label{eq_n-1}
        \E\big[[m_n(\bX_i, \Theta_{\ell}) 
        - &m(\bX_i)][m_n(\bX_i, \Theta_{\ell'}) 
        - m(\bX_i)] \big| \ell, \ell' \in \Lambda_{n,i} \big] \nonumber\\
        &= \E\big[[m_{a_n,n-1}(\bX_i, \Theta_{\ell}) 
        - m(\bX_i)][m_{a_n,n-1}(\bX_i, \Theta_{\ell'}) 
        - m(\bX_i)]\big] \nonumber \\
        &= \E\big[[m_{a_n,n-1}(\bX, \Theta_{\ell}) 
        - m(\bX)][m_{a_n,n-1}(\bX, \Theta_{\ell'}) 
        - m(\bX)]\big],
    \end{align}
    where the last equality holds because $\bX_i$ and $\bX$ are identically distributed and both independent of the training data of $m_{a_n,n-1}$.
    Then, this last equality is plugged in the previous result to obtain
    \begin{align} \label{decomposition_risk_OOB}
        \E\big[\big( m_{M,n}^{(OOB)}&(\bX_i, \bTheta_{M}) - m(\bX_i) \big)^2 \big| |\Lambda_{n,i}| > 0\big] \P(|\Lambda_{n,i}| > 0) \nonumber \\
        =&  \delta_{M,n} \frac{1}{M^2} \sum_{\ell, \ell' = 1}^{M} \E\big[[m_{a_n,n-1}(\bX, \Theta_{\ell}) 
        - m(\bX)][m_{a_n,n-1}(\bX, \Theta_{\ell'}) 
        - m(\bX)] \big] \nonumber \\
        &+ (\gamma_{M,n} - \delta_{M,n}) \frac{1}{M^2} \sum_{\ell = 1}^{M} \E\big[(m_{a_n,n-1}(\bX, \Theta_{\ell}) 
        - m(\bX))^2\big].
    \end{align}
    Next, we factorize the right hand side
    \begin{align} \label{bound_risk_OOB_M}
        \E\big[\big( m_{M,n}^{(OOB)}&(\bX_i, \bTheta_{M}) - m(\bX_i) \big)^2 \big| |\Lambda_{n,i}| > 0\big] \P(|\Lambda_{n,i}| > 0) \nonumber\\
        =&  \delta_{M,n} \E\Big[\Big(\frac{1}{M}\sum_{\ell=1}^{M} m_{a_n,n-1}(\bX, \Theta_{\ell}) 
        - m(\bX)\Big)^2 \Big]   \nonumber\\
        &+ (\gamma_{M,n} - \delta_{M,n}) \frac{1}{M} \E\big[(m_{a_n,n-1}(\bX, \Theta) - m(\bX))^2\big] \nonumber\\
        =&  \delta_{M,n} \E\big[\big(m_{M,a_n,n-1}(\bX, \bTheta_{M}) 
        - m(\bX)\big)^2 \big]   \nonumber\\
        &+ (\gamma_{M,n} - \delta_{M,n}) \frac{1}{M} \E\big[(m_{a_n,n-1}(\bX, \Theta) - m(\bX))^2\big],
    \end{align}
    where $m_{M,a_n,n-1}(\bX, \bTheta_{M})$ is the standard random forest estimate, built with a dataset of size $n-1$ and the subsample size $a_{n}$.
    Using the decomposition (\ref{risk_forest_cart}) of the risk of the finite forest, we have
    \begin{align*}
         \frac{1}{M} \E\big[(m_{a_n,n-1}(\bX, \Theta) - m(\bX))^2\big]
         \leq \E\big[\big(m_{M,a_n,n-1}(\bX, \bTheta_{M}) 
        - m(\bX)\big)^2 \big].
    \end{align*}
    Additionally, from Lemma \ref{lemma_tech}, $\gamma_{M,n} - \delta_{M,n} > 0$. We combine the last two inequalities with the previous result and obtain
    \begin{align*}
        \E\big[\big( m_{M,n}^{(OOB)}(\bX_i, \bTheta_{M}) - &m(\bX_i) \big)^2 \big| |\Lambda_{n,i}| > 0\big] \P(|\Lambda_{n,i}| > 0) \\
        \leq&  \delta_{M,n} \E\big[\big(m_{M,a_n,n-1}(\bX, \bTheta_{M}) 
        - m(\bX)\big)^2 \big]   \\
        &+ (\gamma_{M,n} - \delta_{M,n}) \E\big[\big(m_{M,a_n,n-1}(\bX, \bTheta_{M}) 
        - m(\bX)\big)^2 \big] \\
        \leq& \gamma_{M,n} \E\big[\big(m_{M,a_n,n-1}(\bX, \bTheta_{M}) 
        - m(\bX)\big)^2 \big], 
    \end{align*}
    and using again Lemma \ref{lemma_tech}, we finally get
    \begin{align*}
        \E\big[\big( m_{M,n}^{(OOB)}(\bX_i, \bTheta_{M}) - m(\bX_i) \big)^2 \mathds{1}_{|\Lambda_{n,i}| > 0}\big]
        \leq \frac{2}{1 - a_n/n} \E\big[\big(m_{M,a_n,n-1}(\bX, \bTheta_{M}) 
        - m(\bX)\big)^2 \big].
    \end{align*}
\end{proof}

\begin{proof}[Proof of Proposition \ref{prop_oob_risk}]
    We need to bound the difference between the risks of the OOB estimate and the standard forest. To do so, we go back to equation (\ref{bound_risk_OOB_M})
    \begin{align*}
        \E\big[\big( m_{M,n}^{(OOB)}(\bX_i, \bTheta_{M}) - m(\bX_i) \big)^2& \big| |\Lambda_{n,i}| > 0\big] \P(|\Lambda_{n,i}| > 0) \\
        =  \delta_{M,n}& \E\big[\big(m_{M,a_n,n-1}(\bX, \bTheta_{M}) 
        - m(\bX)\big)^2 \big]   \\
        &+ (\gamma_{M,n} - \delta_{M,n}) \frac{1}{M} \E\big[(m_{a_n,n-1}(\bX, \Theta) - m(\bX))^2\big],
    \end{align*}
    and rewrite it 
    \begin{align*}
        \Big|\E\big[\big( m_{M,n}^{(OOB)}(\bX_i, \bTheta_{M}) - m(\bX_i) \big)^2& \mathds{1}_{|\Lambda_{n,i}| > 0}\big] - \E\big[\big(m_{M,a_n,n-1}(\bX, \bTheta_{M}) 
        - m(\bX)\big)^2 \big]\Big| \\
        \leq  \big|\delta_{M,n}& - 1\big| \E\big[\big(m_{M,a_n,n-1}(\bX, \bTheta_{M}) 
        - m(\bX)\big)^2 \big]   \\
        &+ (\gamma_{M,n} - \delta_{M,n}) \frac{1}{M} \E\big[(m_{a_n,n-1}(\bX, \Theta) - m(\bX))^2\big].
    \end{align*}
    According to Lemma \ref{lemma_tech}, $\delta_{M,n} - 1 = O(1/M)$ and $\gamma_{M,n} - \delta_{M,n}$ is bounded. Therefore, for a fixed sample size $n$, we have
    \begin{align} \label{equivalent_M}
        \Big|\E\big[\big( m_{M,n}^{(OOB)}(\bX_i, \bTheta_{M}) - &m(\bX_i) \big)^2 \mathds{1}_{|\Lambda_{n,i}| > 0}\big] - \E\big[\big(m_{M,a_n,n-1}(\bX, \bTheta_{M}) 
        - m(\bX)\big)^2 \big]\Big| = O\Big(\frac{1}{M}\Big).
    \end{align}
    Finally, recall that $\P(|\Lambda_{n,i}| > 0)$ is the probability that the $i$-th observation does not belong to all trees (in this case the OOB forest estimate is properly defined). A simple calculation gives that $\P(|\Lambda_{n,i}| > 0) = 1 - (a_n/n)^M$, which converges towards $1$ exponentially fast as $M$ grows.
    Then, we have
    \begin{align} \label{bound_1}
        \Big|\E\big[\big(& m_{M,n}^{(OOB)}(\bX_i, \bTheta_{M}) - m(\bX_i) \big)^2\big] - \E\big[\big(m_{M,n}^{(OOB)}(\bX_i, \bTheta_{M}) - m(\bX_i) \big)^2 \mathds{1}_{|\Lambda_{n,i}| > 0}\big]\Big| \nonumber \\
        &= \E\big[\big( m_{M,n}^{(OOB)}(\bX_i, \bTheta_{M}) - m(\bX_i) \big)^2 \mathds{1}_{|\Lambda_{n,i}| = 0}\big] \nonumber \\
        &= \E\big[m(\bX_i)^2\big]\P(|\Lambda_{n,i}| = 0) \nonumber \\
        &\leq ||m||_{\infty}^2 (a_n/n)^M.
    \end{align}
    From Assumption \ref{A3}, $a_n/n < 1$, and combining the bound (\ref{bound_1}) with the previous result (\ref{equivalent_M}), we conclude that
    \begin{align*}
        \Big|\E\big[\big( m_{M,n}^{(OOB)}(\bX_i, \bTheta_{M}) - &m(\bX_i) \big)^2 \big] - \E\big[\big(m_{M,a_n,n-1}(\bX, \bTheta_{M}) 
        - m(\bX)\big)^2 \big]\Big| = O\Big(\frac{1}{M}\Big).
    \end{align*}
    
\end{proof}

\begin{proof}[Proof of Lemma \ref{lemme_oob_consistency}]
    We first assume that Assumptions \ref{A1}, \ref{A2}, \ref{A3}, and \ref{A4} are satisfied, and we consider $i \in \{1,\hdots,n\}$.
    Using Lemma \ref{lemma_oob_risk}, we have
    \begin{align} \label{eq_oob_risk}
        \E\big[ \big( m_{M,n}^{(OOB)}(\bX_i, \bTheta_{M}) - m(\bX_i) \big)^2 \mathds{1}_{|\Lambda_{n,i}| > 0}  \big]
        \leq \frac{2}{1 - a_n/n}& \E\big[ \big(m_{M,a_n,n-1}(\bX, \bTheta_{M}) - m(\bX) \big)^2 \big].
    \end{align}
    According to Assumption \ref{A3}, $1 - a_n/n > \kappa$ where $\kappa$ is fixed positive constant. Thus, we can directly apply Lemma \ref{lemme_rf_consistency} to obtain
    \begin{align*}
        \lim_{n \to \infty} \E\big[ \big(m_{M,a_n,n-1}(\bX, \bTheta_{M}) - m(\bX) \big)^2 \big] = 0,
    \end{align*}
    and then
    \begin{align*}
        \lim_{n \to \infty} \E\big[ \big( m_{M,n}^{(OOB)}(\bX_i, \bTheta_{M}) - m(\bX_i) \big)^2 \mathds{1}_{|\Lambda_{n,i}| > 0} \big] = 0.
    \end{align*}
    
    Next, we extend this result to the permuted case, i.e., $\bX_i$ is replaced by $\Xpermi$. 
    Following the same proof as in Lemma \ref{lemma_oob_risk}, we derive the following decomposition, similarly to equation (\ref{decomposition_risk_OOB})
    \begin{align*}
        \E\big[\big( m_{M,n,\pi_j}^{(OOB)}(\bX_i, \bTheta_{M}) - \E[m(\Xpermi)|&\bX_i^{(-j)}] \big)^2 \big| |\Lambda_{n,i}| > 0\big] \P(|\Lambda_{n,i}| > 0) \\
        =  \delta_{M,n} \frac{1}{M^2}\sum_{\ell\neq\ell'} \E\big[&(m_{a_n,n-1}(\bX_{i,\pi_{j\ell}}, \Theta_{\ell})
        - \E[m(\Xpermi)|\bX_i^{(-j)}]) \\[-1em] &\times (m_{a_n,n-1}(\bX_{i,\pi_{j\ell'}}, \Theta_{\ell'})
        - \E[m(\Xpermi)|\bX_i^{(-j)}])\big]   \\
        + \gamma_{M,n} \frac{1}{M^2} \sum_{\ell=1}^{M} &\E\big[(m_{a_n,n-1}(\bX_{i,\pi_{j\ell}}, \Theta) - \E[m(\Xpermi)|\bX_i^{(-j)}])^2\big].
    \end{align*}
    By symmetry, we have
    \begin{align*}
        \E\big[\big( m_{M,n,\pi_j}^{(OOB)}(\bX_i, \bTheta_{M}) - \E[m(\Xpermi)|&\bX_i^{(-j)}] \big)^2 \big| |\Lambda_{n,i}| > 0\big] \P(|\Lambda_{n,i}| > 0) \\
        =  \delta_{M,n} \frac{M-1}{M} \E\big[&(m_{a_n,n-1}(\bX_{i,\pi_{j 1}}, \Theta_{1})
        - \E[m(\Xpermi)|\bX_i^{(-j)}]) \\ &\times (m_{a_n,n-1}(\bX_{i,\pi_{j 2}}, \Theta_{2})
        - \E[m(\Xpermi)|\bX_i^{(-j)}])\big]   \\
        + \gamma_{M,n} \frac{1}{M} & \E\big[(m_{a_n,n-1}(\bX_{\pi_{j}}, \Theta) - \E[m(\Xperm)|\bX^{(-j)}])^2\big].
    \end{align*}
    In the first term of the right hand side, we need to deal with the specific case where $\pi_{j1}=\pi_{j2}$, which implies that $\bX_{i,\pi_{j 1}} = \bX_{i,\pi_{j 2}}$ since they have the same $j$-th permuted component:
    \begin{align*}
        \E\big[\big( m_{M,n,\pi_j}^{(OOB)}(\bX_i, \bTheta_{M})& - \E[m(\Xpermi)|\bX_i^{(-j)}] \big)^2 \big| |\Lambda_{n,i}| > 0\big] \P(|\Lambda_{n,i}| > 0) \\
        =  \delta_{M,n} \frac{M-1}{M} \E\big[&(m_{a_n,n-1}(\bX_{i,\pi_{j 1}}, \Theta_{1})
        - \E[m(\Xpermi)|\bX_i^{(-j)}]) \\ &\times (m_{a_n,n-1}(\bX_{i,\pi_{j 2}}, \Theta_{2})
        - \E[m(\Xpermi)|\bX_i^{(-j)}]) | \pi_{j1} \neq \pi_{j2} \big] \P(\pi_{j1} \neq \pi_{j2})  \\
        + \delta_{M,n} \frac{M-1}{M} & \E\big[(m_{a_n,n-1}(\bX_{i,\pi_{j 1}}, \Theta_{1})
        - \E[m(\Xpermi)|\bX_i^{(-j)}]) \\ &\times (m_{a_n,n-1}(\bX_{i,\pi_{j 2}}, \Theta_{2})
        - \E[m(\Xpermi)|\bX_i^{(-j)}]) | \pi_{j1}=\pi_{j2} \big]
        \P(\pi_{j1} = \pi_{j2}) \\
        + \gamma_{M,n} \frac{1}{M} & \E\big[(m_{a_n,n-1}(\bX_{\pi_{j}}, \Theta) - \E[m(\Xperm)|\bX^{(-j)}])^2\big],
    \end{align*}
    which can be simplified using Cauchy-Schwartz inequality for the second term as
    \begin{align} \label{decomposition_perm}
        \E\big[\big( m_{M,n,\pi_j}^{(OOB)}(\bX_i, \bTheta_{M})& - \E[m(\Xpermi)|\bX_i^{(-j)}] \big)^2 \big| |\Lambda_{n,i}| > 0\big] \P(|\Lambda_{n,i}| > 0) \\
        \leq  \delta_{M,n} \frac{M-1}{M} \E\big[&(m_{a_n,n-1}(\bX_{i,\pi_{j 1}}, \Theta_{1})
        - \E[m(\Xpermi)|\bX_i^{(-j)}]) \nonumber\\ &\times (m_{a_n,n-1}(\bX_{i,\pi_{j 2}}, \Theta_{2})
        - \E[m(\Xpermi)|\bX_i^{(-j)}]) | \pi_{j1} \neq \pi_{j2} \big] \P(\pi_{j1} \neq \pi_{j2})  \nonumber\\
        + \Big( \frac{\gamma_{M,n}}{M} +& \delta_{M,n} \frac{M-1}{M} \P(\pi_{j1} = \pi_{j2}) \Big) \E\big[(m_{a_n,n-1}(\bX_{\pi_{j}}, \Theta) - \E[m(\Xperm)|\bX^{(-j)}])^2\big].\nonumber
    \end{align}
    Now, we focus on the first term of the right hand side. We have
    \begin{align*}
        \E\big[(m_{a_n,n-1}&(\bX_{i,\pi_{j1}}, \Theta_{1})
        - \E[m(\Xpermi)|\bX_i^{(-j)}]) \\ &\times (m_{a_n,n-1}(\bX_{i,\pi_{j2}}, \Theta_{2})
        - \E[m(\Xpermi)|\bX_i^{(-j)}]) | \pi_{j1} \neq \pi_{j2} \big] \\
        =  \E\big[&[m_{a_n,n-1}(\bX_{i,\pi_{j1}}, \Theta_{1}) - m(\bX_{i,\pi_{j1}})
        - (\E[m(\Xpermi)|\bX_i^{(-j)}] - m(\bX_{i,\pi_{j1}}))] \\ &\times [m_{a_n,n-1}(\bX_{i,\pi_{j2}}, \Theta_{2}) - m(\bX_{i,\pi_{j2}})
        - (\E[m(\Xpermi)|\bX_i^{(-j)}] - m(\bX_{i,\pi_{j2}}))] | \pi_{j1} \neq \pi_{j2} \big] \\
         = \E\big[&(m_{a_n,n-1}(\bX_{i,\pi_{j1}}, \Theta_{1}) - m(\bX_{i,\pi_{j1}}))
         (m_{a_n,n-1}(\bX_{i,\pi_{j2}}, \Theta_{2}) - m(\bX_{i,\pi_{j2}}))| \pi_{j1} \neq \pi_{j2} \big]\\
        & -2 \E\big[(m_{a_n,n-1}(\bX_{i,\pi_{j1}}, \Theta_{1}) - m(\bX_{i,\pi_{j1}}))
        (\E[m(\Xpermi)|\bX_i^{(-j)}] - m(\bX_{i,\pi_{j2}})) | \pi_{j1} \neq \pi_{j2} \big] \\
        &+ \E\big[(\E[m(\Xpermi)|\bX_i^{(-j)}] - m(\bX_{i,\pi_{j1}}))(\E[m(\Xpermi)|\bX_i^{(-j)}] - m(\bX_{i,\pi_{j2}})) | \pi_{j1} \neq \pi_{j2} \big].
    \end{align*}
    For the second term, the two multiplied terms are independent conditional on $\bX_i^{(-j)}$ and $\pi_{j1} \neq \pi_{j2}$, then
    \begin{align*}
        \E&\big[(m_{a_n,n-1}(\bX_{i,\pi_{j1}}, \Theta_{1}) - m(\bX_{i,\pi_{j1}}))
        (\E[m(\Xpermi)|\bX_i^{(-j)}] - m(\bX_{i,\pi_{j2}})) \big| \pi_{j1} \neq \pi_{j2} \big] \\
        &=\E\big[\E\big[(m_{a_n,n-1}(\bX_{i,\pi_{j1}}, \Theta_{1}) - m(\bX_{i,\pi_{j1}}))
        (\E[m(\Xpermi)|\bX_i^{(-j)}] - m(\bX_{i,\pi_{j2}})) \big| \bX_i^{(-j)},\pi_{j1} \neq \pi_{j2}\big] \big] \\
        &=\E\big[\E\big[m_{a_n,n-1}(\bX_{i,\pi_{j1}}, \Theta_{1}) - m(\bX_{i,\pi_{j1}}) \big| \bX_i^{(-j)}\big] \E\big[ \E[m(\Xpermi)|\bX_i^{(-j)}] - m(\bX_{i,\pi_{j2}})) \big| \bX_i^{(-j)}\big] \big] \\
        &= 0.
    \end{align*}
    Similarly, the third term is also null. Finally, we apply Cauchy-Schwartz inequality to the first term to obtain
    \begin{align*}
        \delta_{M,n} \frac{M-1}{M} \E\big[(m_{a_n,n-1}(&\bX_{i,\pi_{j1}}, \Theta_{1})
        - \E[m(\Xpermi)|\bX_i^{(-j)}]) \\ &\times (m_{a_n,n-1}(\bX_{i,\pi_{j2}}, \Theta_{2})
        - \E[m(\Xpermi)|\bX_i^{(-j)}]) | \pi_{j1} \neq \pi_{j2} \big] \\
        \leq \delta_{M,n} & \E\big[(m_{a_n,n-1}(\bX_{i,\pi_{j1}}, \Theta_{1}) - m(\bX_{i,\pi_{j1}}))^2\big]\\
        \leq \delta_{M,n} & \E\big[(m_{a_n,n-1}(\Xperm, \Theta)
        - m(\Xperm))^2\big],
    \end{align*}
    where the last inequality holds because $\bX_{i,\pi_{j1}}$ is independent of the sample used to train $m_{a_n,n-1}$ and have the same distribution as $\Xperm$.
    Overall, using this last inequality with the decomposition (\ref{decomposition_perm}), we obtain the following bound
    \begin{align*}
        \E\big[\big( m_{M,n,\pi_j}^{(OOB)}&(\bX_i, \bTheta_{M}) - \E[m(\Xpermi)|\bX_i^{(-j)}] \big)^2 \big| |\Lambda_{n,i}| > 0\big] \P(|\Lambda_{n,i}| > 0) \\
        \leq  &\delta_{M,n} \E\big[(m_{a_n,n-1}(\Xperm, \Theta)
        - m(\Xperm))^2\big]  \\
        &+ \Big( \frac{\gamma_{M,n}}{M} + \delta_{M,n} \frac{M-1}{M} \P(\pi_{j1} = \pi_{j2}) \Big) \E\big[(m_{a_n,n-1}(\Xperm, \Theta) - \E[m(\Xperm)|\bX^{(-j)}])^2\big].
    \end{align*}
    Furthermore, using Lemma \ref{lemma_tech}, the bound can be simplified to get
    \begin{align*}
        \E\big[\big( m_{M,n,\pi_j}^{(OOB)}&(\bX_i, \bTheta_{M}) - \E[m(\Xpermi)|\bX_i^{(-j)}] \big)^2 \big| |\Lambda_{n,i}| > 0\big] \P(|\Lambda_{n,i}| > 0) \\
        \leq  & \E\big[(m_{a_n,n-1}(\Xperm, \Theta)
        - m(\Xperm))^2\big]  \\
        &+ \Big( \frac{2}{1-a_n/n} \frac{1}{M} + \P(\pi_{j1} = \pi_{j2}) \Big) \E\big[(m_{a_n,n-1}(\Xperm, \Theta) - \E[m(\Xperm)|\bX^{(-j)}])^2\big].
    \end{align*}
    Next, we break down the expectation of the second term
    \begin{align*}
        \E\big[(m_{a_n,n-1}(&\Xperm, \Theta) - \E[m(\Xperm)|\bX^{(-j)}])^2\big] \\
        =& \E\big[(m_{a_n,n-1}(\Xperm, \Theta) - m(\Xperm) + (m(\Xperm) - \E[m(\Xperm)|\bX^{(-j)}]))^2\big]\\
        =& \E\big[(m_{a_n,n-1}(\Xperm, \Theta) - m(\Xperm))^2\big] + \E\big[(m(\Xperm) - \E[m(\Xperm)|\bX^{(-j)}])^2\big] \\
        &+ 2 \E\big[(m_{a_n,n-1}(\Xperm, \Theta) - m(\Xperm))
        (m(\Xperm) - \E[m(\Xperm)|\bX^{(-j)}])\big].
    \end{align*}
    Since $m$ is bounded, we get
    \begin{align*}
        \E\big[(m_{a_n,n-1}(\Xperm, \Theta) - \E[m(\Xperm)&|\bX^{(-j)}])^2\big] \\
        \leq \E\big[(&m_{a_n,n-1}(\Xperm, \Theta) - m(\Xperm))^2\big] + 4 ||m||_{\infty}^2 \\
        &+ 4 ||m||_{\infty} \E\big[|m_{a_n,n-1}(\Xperm, \Theta) - m(\Xperm)|\big].
    \end{align*}
    Finally we obtain the following bound
    \begin{align*}
        \E\big[\big( m_{M,n,\pi_j}^{(OOB)}&(\bX_i, \bTheta_{M}) - \E[m(\Xpermi)|\bX_i^{(-j)}] \big)^2 \big| |\Lambda_{n,i}| > 0\big] \P(|\Lambda_{n,i}| > 0) \\
        \leq  & \Big( 1 + \frac{2}{1-a_n/n} \frac{1}{M} + \P(\pi_{j1} = \pi_{j2}) \Big) \E\big[(m_{a_n,n-1}(\Xperm, \Theta_{\ell})
        - m(\Xperm))^2\big] \\
        &+ \Big( \frac{2}{1-a_n/n} \frac{1}{M} + \P(\pi_{j1} = \pi_{j2}) \Big) 4 ||m||_{\infty} \E\big[|m_{a_n,n-1}(\Xperm, \Theta) - m(\Xperm)|\big] \\ &+ 4 ||m||_{\infty}^2 \Big( \frac{2}{1-a_n/n} \frac{1}{M} + \P(\pi_{j1} = \pi_{j2}) \Big).
    \end{align*}
    The second part of Lemma \ref{lemme_rf_consistency} for $M = 1$ gives that 
    \begin{align*}
        \lim_{n \to \infty} \E\big[(m_{a_n,n-1}(\Xperm, \Theta_{\ell})
        - m(\Xperm))^2\big] = 0,
    \end{align*}
    and since $\mathbb{L}^2$-convergence implies $\mathbb{L}^1$-convergence, we also have
    \begin{align*}
        \lim_{n \to \infty} \E\big[|m_{a_n,n-1}(\Xperm, \Theta_{\ell})
        - m(\Xperm)|\big] = 0.
    \end{align*}
    It is clear that $\P(\pi_{j1} = \pi_{j2}) < 1/(n - a_n)$, and then $\lim_{n \to \infty} \P(\pi_{j1} = \pi_{j2}) = 0$, since $1 - a_n/n > \kappa > 0$ by Assumption \ref{A3}. Additionally, according to Assumption \ref{A4}, $M \underset{n \to \infty}{\longrightarrow}\infty$, therefore
    \begin{align*}
        \lim_{n \to \infty} \frac{2}{1-a_n/n} \frac{1}{M} + \P(\pi_{j1} = \pi_{j2}) = 0.
    \end{align*}
    Overall, we have
    \begin{align*}
        \lim_{n \to \infty} \E\big[\big( m_{M,n,\pi_j}^{(OOB)}(\bX_i, \bTheta_{M}) - \E[m(\Xpermi)|&\bX_i^{(-j)}] \big)^2 \mathds{1}_{|\Lambda_{n,i}| > 0}\big]  = 0.
    \end{align*}

\end{proof}

\begin{proof}[Proof of Lemma \ref{lemma_tech}]
    We consider $M \in \mathbb{N}\setminus\{0,1\}$, $i \in \{1,\hdots,n\}$, and define
    \begin{align*}
        \delta_{M,n} =& M^2 \E \Big[ \frac{1}{|\Lambda_{n,i}|^2} \big| 1,2 \in \Lambda_{n,i} \Big] \P(1,2 \in \Lambda_{n,i}) \\
         =& M^2 \E \Big[ \frac{1}{|\Lambda_{n,i}|^2} \big| M-1,M \in \Lambda_{n,i} \Big] \P(M-1,M \in \Lambda_{n,i}).
    \end{align*}
    Recall that by definition, $|\Lambda_{n,i}| = \sum_{\ell=1}^{M} \mathds{1}_{i \notin \Theta^{(S)}_{\ell}}$. Since $\Theta_{\ell}$ are iid, $|\Lambda_{n,i}|$ is a binomial random variable. Then, we have
    \begin{align*}
        \E \Big[ \frac{1}{|\Lambda_{n,i}|^2} \big| M,M-1 \in \Lambda_{n,i} \Big] =& 
        \E \Big[ \frac{1}{(2 + \sum_{\ell=1}^{M-2} \mathds{1}_{i \notin \Theta^{(S)}_{\ell}})^2} \Big] \\ =&
        \sum_{k=0}^{M-2} \frac{1}{(k+2)^2} {M-2 \choose k}\big( 1 - \frac{a_n}{n} \big)^{k} \big(\frac{a_n}{n}\big)^{M-2-k}.
    \end{align*}
    On the other hand, 
    \begin{align*}
        \P(M-1,M \in \Lambda_{n,i}) =  \big(1 - \frac{a_n}{n}\big)^2.
    \end{align*}
    Combining the previous two equations, we get
    \begin{align*}
        \delta_{M,n} =& M^2 \big( 1 - \frac{a_n}{n} \big)^2 \sum_{k=0}^{M-2} \frac{1}{(k+2)^2} {M-2 \choose k}\big( 1 - \frac{a_n}{n} \big)^{k} \big(\frac{a_n}{n}\big)^{M-2-k} \\
        =& M^2 \sum_{k=0}^{M-2} \frac{1}{(k+2)^2} \frac{(M-2)!}{k!(M-(k+2))!} \big( 1 - \frac{a_n}{n} \big)^{k+2} \big(\frac{a_n}{n}\big)^{M-(k+2)} \\
        =& M^2 \sum_{k=0}^{M-2} \frac{k+1}{(k+2)M(M-1)} \frac{M!}{(k+2)!(M-(k+2))!} \big( 1 - \frac{a_n}{n} \big)^{k+2} \big(\frac{a_n}{n}\big)^{M-(k+2)} \\
        =& \frac{M}{M-1} \sum_{k=0}^{M-2} \frac{k+1}{k+2} {M \choose k+2} \big( 1 - \frac{a_n}{n} \big)^{k+2} \big(\frac{a_n}{n}\big)^{M-(k+2)} \\
    \end{align*}
    We reindex the sum with $k \shortleftarrow k + 2 $ and get
    \begin{align} \label{eq_delta_0}
        \delta_{M,n}
        =& \frac{M}{M-1} \sum_{k=2}^{M} \frac{k-1}{k} {M \choose k} \big( 1 - \frac{a_n}{n} \big)^{k} \big(\frac{a_n}{n}\big)^{M-k} \nonumber \\
         =& \frac{M}{M-1} \sum_{k=1}^{M} \Big(1 - \frac{1}{k}\Big) {M \choose k} \big( 1 - \frac{a_n}{n} \big)^{k} \big(\frac{a_n}{n}\big)^{M-k}.
    \end{align}
    Next, we bound $\delta_{M,n}$,
    \begin{align} \label{eq_delta}
        \delta_{M,n}
          \leq& \frac{M}{M-1} \sum_{k=1}^{M} \Big(1 - \frac{1}{M}\Big) {M \choose k} \big( 1 - \frac{a_n}{n} \big)^{k} \big(\frac{a_n}{n}\big)^{M-k} \nonumber \\
           \leq& \sum_{k=0}^{M} {M \choose k} \big( 1 - \frac{a_n}{n} \big)^{k} \big(\frac{a_n}{n}\big)^{M-k} - \big(\frac{a_n}{n}\big)^M \nonumber\\
           \leq& 1 - \big(\frac{a_n}{n}\big)^M \\
           \leq& 1. \nonumber
    \end{align}
    
    Similarly for the second inequality, we define
    \begin{align*}
        \gamma_{M,n} =& M^2 \E \Big[ \frac{1}{|\Lambda_{n,i}|^2} \big| 1 \in \Lambda_{n,i} \Big] \P(1 \in \Lambda_{n,i}) \\
         =& M^2 \E \Big[ \frac{1}{|\Lambda_{n,i}|^2} \big| M \in \Lambda_{n,i} \Big] \P(M \in \Lambda_{n,i}),
    \end{align*}
    and get
    \begin{align*}
        \gamma_{M,n} =& M^2 \big(1 - \frac{a_n}{n} \big) \sum_{k=0}^{M-1} \frac{1}{(k+1)^2} {M-1 \choose k}\big( 1 - \frac{a_n}{n} \big)^{k} \big(\frac{a_n}{n}\big)^{M-1-k} \\
        =& M \sum_{k=0}^{M-1} \frac{1}{k+1} {M \choose k+1} \big( 1 - \frac{a_n}{n} \big)^{k+1} \big(\frac{a_n}{n}\big)^{M-(k+1)} \\
        =& M \sum_{k=1}^{M} \frac{1}{k} {M \choose k} \big( 1 - \frac{a_n}{n} \big)^{k} \big(\frac{a_n}{n}\big)^{M-k} \\
        =& M \E\Big[\frac{1}{Z}\mathds{1}_{Z\geq1}\Big],
    \end{align*}
    where $Z$ is a binomial random variable with $M$ trials and parameter $1 - \frac{a_n}{n}$.
    Lemma 4.1 from \citet{gyorfi2006distribution} states that
    \begin{align} \label{bound_gyorfi}
        \E\Big[\frac{1}{Z}\mathds{1}_{Z\geq1}\Big] 
        \leq \frac{2}{(M+1)(1 - \frac{a_n}{n})},
    \end{align}
    which implies that
    \begin{align*}
        \gamma_{M,n} \leq \frac{2 M}{(M+1)(1 - \frac{a_n}{n})}
        \leq \frac{2}{1 - \frac{a_n}{n}}.
    \end{align*}
    On the other hand, 
    \begin{align*}
        \gamma_{M,n}
        =& M \sum_{k=1}^{M} \frac{1}{k} {M \choose k} \big( 1 - \frac{a_n}{n} \big)^{k} \big(\frac{a_n}{n}\big)^{M-k} \\
        \geq& M \sum_{k=1}^{M} \frac{1}{M} {M \choose k} \big( 1 - \frac{a_n}{n} \big)^{k} \big(\frac{a_n}{n}\big)^{M-k} \\
        \geq& 1 - \big(\frac{a_n}{n}\big)^M \\
        \geq& \delta_{M,n},
    \end{align*}
    where the last inequality uses (\ref{eq_delta}).
    
    To prove the last statement of Lemma~\ref{lemma_tech}, we go back to equation (\ref{eq_delta_0}):
         \begin{align*} 
        \delta_{M,n}
        =& \frac{M}{M-1} \sum_{k=1}^{M} \big(1 - \frac{1}{k}\big) {M \choose k} \big( 1 - \frac{a_n}{n} \big)^{k} \big(\frac{a_n}{n}\big)^{M-k} \\
        =& \frac{M}{M-1} \Big[ \sum_{k=1}^{M} {M \choose k} \big( 1 - \frac{a_n}{n} \big)^{k} \big(\frac{a_n}{n}\big)^{M-k} 
        - \sum_{k=1}^{M} \frac{1}{k} {M \choose k} \big( 1 - \frac{a_n}{n} \big)^{k} \big(\frac{a_n}{n}\big)^{M-k} \Big] \\
        =& \frac{M}{M-1} \Big[1 - \big(\frac{a_n}{n}\big)^{M} - \E\Big[\frac{1}{Z}\mathds{1}_{Z\geq1}\Big]\Big] \\
        \geq& \frac{M}{M-1} \Big[ 1 - \big(\frac{a_n}{n}\big)^{M} - \frac{2}{(M+1)(1 - \frac{a_n}{n})} \Big],
    \end{align*}
    where we use inequality (\ref{bound_gyorfi}) for the last statement.
    Overall, using also inequality (\ref{eq_delta}), we have
    \begin{align*}
        0 \geq M(\delta_{M,n} - 1) \geq \frac{M}{M-1} \Big[ 1 - M\big(\frac{a_n}{n}\big)^{M} - \frac{2M}{(M+1)(1 - \frac{a_n}{n})} \Big]
    \end{align*}
    The right hand side is an increasing function of $M$ and converges towards $-\frac{1+a_n/n}{1-a_n/n}$ as $M \to \infty$. 
    Additionally, the right hand side is always defined since $1 - a_n/n > \kappa > 0$ from Assumption \ref{A3}.
    Therefore, for a fixed sample size $n$, $M(\delta_{M,n} - 1)$ is a bounded sequence. Finally,
    \begin{align*}
       \delta_{M,n} - 1 = O\Big(\frac{1}{M}\Big).
    \end{align*}
    
\end{proof}

\subsection{Proof of Proposition \ref{prop_MDA}} \label{sec_proof_prop1}

\begin{proposition*} \label{prop_MDA}
    If Assumptions \ref{A1}, \ref{A2} and \ref{A3} are satisfied, then for all $M \in \mathbb{N}^{\star}$ and $j \in \{1,\hdots,p\}$ we have
	\begin{align*} &(i) \quad  \widehat{\textrm{MDA}}_{M,n}^{(TT)}(X^{(j)}) \overset{\mathbb{L}^1}{\longrightarrow} \V[Y] \times ST^{(j)} + \V[Y] \times ST^{(j)}_{mg} + \textrm{MDA}_3^{ \star (j)} \\ &(ii) \quad \widehat{\textrm{MDA}}_{M,n}^{(BC)}(X^{(j)}) \overset{\mathbb{L}^1}{\longrightarrow} \V[Y] \times ST^{(j)} + \V[Y] \times ST^{(j)}_{mg} + \textrm{MDA}_3^{ \star (j)}.
	\end{align*}
	If Assumption \ref{A4} is additionally satisfied, then
	\begin{align*}
	(iii) \quad \widehat{\textrm{MDA}}_{M,n}^{(IK)}(X^{(j)}) \overset{\mathbb{L}^1}{\longrightarrow} \V[Y] \times ST^{(j)} + \textrm{MDA}_3^{ \star (j)}.
	\end{align*}
\end{proposition*}

\begin{proof}[Proof of Proposition \ref{prop_MDA}]
	We assume that Assumptions \ref{A1}, \ref{A2}, and \ref{A3} are satisfied, and fix $j \in \{1,\hdots,p\}$ and $M \in \mathbb{N}^{\star}$. Then, using Theorem \ref{thm_MDA}-(i), we have 
	\begin{align*}
	     \widehat{\textrm{MDA}}_{M,n}^{(TT)}(X^{(j)}) \overset{\mathbb{L}^1}{\longrightarrow} \E[(m(\bX) - m(\Xperm))^2].
	\end{align*}
	First, we rewrite the MDA limit as
	\begin{align*}
	\E[(m(\bX) - &m(\Xperm))^2] \\
	=& \E[\E[(m(\bX) - m(\Xperm))^2 | \bX^{(-j)}]] \\
	=& \E\big[\E\big[\big((m(\bX) - \E[m(\bX)|\bX^{(-j)}]) -  
	(m(\Xperm) - \E[m(\Xperm)|\bX^{(-j)}]) \\ & + 
	(\E[m(\bX)|\bX^{(-j)}] - \E[m(\Xperm)|\bX^{(-j)}]) \big)^2 | \bX^{(-j)}\big]\big].
	\end{align*}
	Now, observing that these three terms are independent conditionally on $\bX^{(-j)}$, we can expand the MDA limit as follows
    \begin{align*}
	\E[(m(\bX) - &m(\Xperm))^2] \\
	=& \E\big[\E\big[(m(\bX) - \E[m(\bX)|\bX^{(-j)}])^2| \bX^{(-j)}\big] + \E[(m(\Xperm) - \E[m(\Xperm)|\bX^{(-j)}])^2| \bX^{(-j)}\big] \\ & + 
	(\E[m(\bX)|\bX^{(-j)}] - \E[m(\Xperm)|\bX^{(-j)}])^2 \big] \\
	=& \E[\V[m(\bX)|\bX^{(-j)}]] +  \E[\V[m(\Xperm)|\bX^{(-j)}]] \\ & +
	\E[(\E[m(\bX)|\bX^{(-j)}] - \E[m(\Xperm)|\bX^{(-j)}])^2] \\
	=& \V[Y] \times ST^{(j)} +  \V[Y] \times ST_{mg}^{(j)} + \E[(\E[m(\bX)| \bX^{(-j)}] - \E[m(\Xperm)| \bX^{(-j)}])^2].
    \end{align*}
    
    Theorem \ref{thm_MDA}-(ii) gives the same theoretical counterpart for BC-MDA, and thus the same decomposition applies
    \begin{align*}
	     \widehat{\textrm{MDA}}_{M,n}^{(BC)}(X^{(j)}) \overset{\mathbb{L}^1}{\longrightarrow} \V[Y] \times ST^{(j)} +  \V[Y] \times ST_{mg}^{(j)} + \E[(\E[m(\bX)| \bX^{(-j)}] - \E[m(\Xperm)| \bX^{(-j)}])^2].
	\end{align*}
    
    Now, we additionally assume that Assumption \ref{A4} is satisfied, i.e., the number of trees grows to infinity with $n$.
    Then, using Theorem \ref{thm_MDA}-(iii) we have 
	\begin{align*}
	     \widehat{\textrm{MDA}}_{M,n}^{(IK)}(X^{(j)}) \overset{\mathbb{L}^1}{\longrightarrow} \E[(m(\bX) - \E[m(\Xperm)|\bX^{(-j)}])^2].
	\end{align*}
	We decompose the theoretical counterpart as in the first case,
	\begin{align*}
	     \E[(m(\bX) &- \E[m(\Xperm)|\bX^{(-j)}])^2] \\ 
	     =& \E[(m(\bX) - \E[m(\bX)|\bX^{(-j)}] - (\E[m(\Xperm)|\bX^{(-j)}] - \E[m(\bX)|\bX^{(-j)}] ))^2] \\
	     =& \E[(m(\bX) - \E[m(\bX)|\bX^{(-j)}])^2] + \E[(\E[m(\bX)|\bX^{(-j)}] - \E[m(\Xperm)|\bX^{(-j)}]))^2] \\
	     =& \V[Y] \times ST^{(j)} + \E[(\E[m(\bX)| \bX^{(-j)}] - \E[m(\Xperm)| \bX^{(-j)}])^2].
	\end{align*}
	
\end{proof}

\subsection{Proof of Corollary $2$}

\begin{corollary*} \label{cor_MDA_indep}
If covariates are independent, and if Assumptions \ref{A1}-\ref{A3} are satisfied, for all $M \in \mathbb{N}^{\star}$ and $j \in \{1,\hdots,p\}$ we have 
\begin{align*}
\widehat{\textrm{MDA}}_{M,n}^{(TT)}(X^{(j)})& \overset{\mathbb{L}^1}{\longrightarrow} 2\V[Y] \times ST^{(j)} \\
\widehat{\textrm{MDA}}_{M,n}^{(BC)}(X^{(j)})& \overset{\mathbb{L}^1}{\longrightarrow} 2\V[Y] \times ST^{(j)}.
\end{align*}
In addition, if Assumption \ref{A4} is satisfied,
\begin{align*}
\widehat{\textrm{MDA}}_{M,n}^{(IK)}(X^{(j)})& \overset{\mathbb{L}^1}{\longrightarrow} \V[Y] \times ST^{(j)}.
\end{align*}
\end{corollary*}

\begin{corollary*} \label{cor_MDA_additive}
If the regression function $m$ is additive, and if Assumptions \ref{A1}-\ref{A3} are satisfied, for all $M \in \mathbb{N}^{\star}$ and $j \in \{1,\hdots,p\}$ we have 
\begin{align*}
\widehat{\textrm{MDA}}_{M,n}^{(TT)}(X^{(j)})& \overset{\mathbb{L}^1}{\longrightarrow} 2 \V[Y] \times ST_{mg}^{(j)} \\
\widehat{\textrm{MDA}}_{M,n}^{(BC)}(X^{(j)})& \overset{\mathbb{L}^1}{\longrightarrow} 2 \V[Y] \times ST_{mg}^{(j)}.
\end{align*}
In addition, if Assumption \ref{A4} is satisfied,
\begin{align*}
\widehat{\textrm{MDA}}_{M,n}^{(IK)}(X^{(j)})& \overset{\mathbb{L}^1}{\longrightarrow} \V[Y] \times ST_{mg}^{(j)}.
\end{align*}
\end{corollary*}

\begin{proof}[Proof of Corollary \ref{cor_MDA_additive}]

	We assume that Assumptions \ref{A1}, \ref{A2}, and \ref{A3} are satisfied, and fix $j \in \{1,\hdots,p\}$ and $M \in \mathbb{N}^{\star}$. Then, using Theorem \ref{thm_MDA}-(i), we have 
	\begin{align*}
	     \widehat{\textrm{MDA}}_{M,n}^{(TT)}(X^{(j)}) \overset{\mathbb{L}^1}{\longrightarrow} \E[(m(\bX) - m(\Xperm))^2].
	\end{align*}
    Since the regression function is assumed additive, we can write $m$ as
    \begin{align*}
        m(\bx) = \sum_{k=1}^{p} m_k(x^{(k)}).
    \end{align*}
    Then, the MDA limit writes
    \begin{align*}
        \E[(m(\bX) - m(\Xperm))^2] &= \E[\{m_j(X^{(j)}) - m_j(X'^{(j)})\}^2] \\
        &= \E[\{(m_j(X^{(j)}) - \E[m_j(X^{(j)})]) - (m_j(X'^{(j)}) - \E[m_j(X^{(j)})])\}^2] \\
        &= 2 \V[m_j(X^{(j)})],
    \end{align*}
    where $X'^{(j)}$ is an independent copy of $X^{(j)}$ by definition of $\Xperm$. 
    
    On the other hand, we have
    \begin{align*}
	    \V[Y] \times ST^{(j)}_{mg} &= \E[\V[m(\Xperm)|\bX^{(-j)}]] \\
	    &= \E[\{m(\Xperm) - \E[m(\Xperm)|\bX^{(-j)}]\}^2] \\
	    &= \E[\{m_j(X'^{(j)}) + \sum_{k \neq j}^{p} m_k(X^{(k)}) - \E[m_j(X'^{(j)}) + \sum_{k \neq j}^{p} m_k(X^{(k)})|\bX^{(-j)}]\}^2] \\
	    &= \E[\{m_j(X'^{(j)}) - \E[m_j(X'^{(j)})]\}^2] \\
	    &= \V[m_j(X^{(j)})] \\
	    &= 1/2 \E[(m(\bX) - m(\Xperm))^2],
	\end{align*}
    which gives the result of Corollary \ref{cor_MDA_additive} for the Train-Test MDA.
    
    The proof for the Breiman-Cutler MDA is identical. 
    For the Iswharan-Kogalur MDA, we assume that Assumption (A4) is additionally satisfied, and Theorem \ref{thm_MDA} gives that
	\begin{align*}
	\widehat{\textrm{MDA}}_{M,n}^{(IK)}(X^{(j)}) \overset{\mathbb{L}^1}{\longrightarrow} \E[\{m(\bX) - \E[m(\Xperm)|\bX^{(-j)}]\}^2].
	\end{align*}
    Again, we can simplify the MDA limit in the additive setting, and we get
    \begin{align*}
        \E[\{m(\bX) - \E[m(\Xperm)|\bX^{(-j)}]\}^2] &= \E[\{m_j(X^{(j)}) - \E[m_j(X'^{(j)})]\}^2] \\
        &= \V[m_j(X^{(j)})] \\
        &=  \V[Y] \times ST^{(j)}_{mg},
    \end{align*}
    which gives the final result.
    
\end{proof}

\subsection{Proof of Property $1$}

\begin{property*}[Marginal Total Sobol Index] \label{property_Smg}
    If Assumption $1$ is satisfied, the marginal total Sobol index $ST^{(j)}_{mg}$ satisfies the following properties.
    \begin{enumerate}[(a)]
        \item $ST^{(j)}_{mg} = 0 \iff ST^{(j)} = 0$.
        \item If the components of $X$ are independent, then we have $ST^{(j)}_{mg} = ST^{(j)}$.
        \item If $m$ is additive, i.e. $m(X) = \sum_{k} m_k(X^{(k)})$, then we have $ST^{(j)}_{mg} = \V[m_j(X^{(j)})]/\V[Y]$, and $ST^{(j)}_{mg} \geq ST^{(j)}$.
    \end{enumerate}
\end{property*}

\begin{proof}[Proof of Property \ref{property_Smg}]
We assume that Assumption $1$ is satisfied. \\

(a) First, we assume that $ST^{(j)} = 0$.
Using the definition of the total Sobol index, we get that
\begin{align*}
    \E[(m(\bX) - \E[m(\bX)|\bX^{(-j)}])^2] = 0.
\end{align*}
By Assumption $1$, the density of $\bX$ is strictly positive on its support $[0,1]^p$, and since the square function is positive, the previous equation gives that, almost surely, 
\begin{align*}
    (m(\bX) - \E[m(\bX)|\bX^{(-j)}])^2 = 0,
\end{align*}
which gives
\begin{align*}
    m(\bX) = \E[m(\bX)|\bX^{(-j)}] \quad \textrm{a.s.}
\end{align*}
Therefore, m(\bX) does not depend on the $j$-th component almost surely, and we have
\begin{align*}
    m(\Xperm) = m(\bX) \quad \textrm{a.s.},
\end{align*}
and consequently $ST^{(j)}_{mg} = ST^{(j)} = 0$.
The reverse case follows the same proof. \\

(b) By construction, $\Xperm$ and $\bX$ have the same joint distribution when $\bX$ has independent components, and the result follows. \\

(c) We assume that $m$ is additive and writes
\begin{align*}
    m(\bX) = \sum_{k=1}^p m_k(X^{(k)}).
\end{align*}
We expand the definition of the marginal total Sobol index using the above expression of $m$ and obtain
\begin{align*}
    \V[Y] \times ST^{(j)}_{mg} &= \E[\V[m(\Xperm)|\bX^{(-j)}]] \\
    &= \E[\{m(\Xperm) - \E[m(\Xperm)|\bX^{(-j)}]\}^2] \\
    &= \E[\{m_j(X'^{(j)}) + \sum_{k \neq j}^{p} m_k(X^{(k)}) - \E[m_j(X'^{(j)}) + \sum_{k \neq j}^{p} m_k(X^{(k)})|\bX^{(-j)}]\}^2] \\
    &= \E[\{m_j(X'^{(j)}) - \E[m_j(X'^{(j)})]\}^2] \\
    &= \V[m_j(X^{(j)})].
\end{align*}

For the second part of the statement, we similarly derive
\begin{align*}
    \V[Y] \times ST^{(j)}
    &= \E[\{m_j(X^{(j)}) - \E[m_j(X^{(j)})|\bX^{(-j)}]\}^2] \\
    &= \E[\V[m_j(X^{(j)})|\bX^{(-j)}]],
\end{align*}
and the law of total variance gives that $ST^{(j)}_{mg} \geq ST^{(j)}$.

\end{proof}

\section{Proof of the Sobol-MDA Consistency} \label{sec_proof_SMDA}

For the sake of clarity, we recall Assumptions \ref{A5}, \ref{A6}, and Theorem \ref{thm_MDA_sobol}.
\begin{assumption*} \label{A5}
A node split is constrained to generate child nodes with at least a small fraction $\gamma > 0$ of the parent node observations. 
Secondly, the split selection is slightly modified: at each tree node, the number \texttt{mtry} of candidate variables drawn to optimize the split is set to $\texttt{mtry} = 1$ with a small probability $\delta > 0$. Otherwise, with probability $1 - \delta$, the default value of \texttt{mtry} is used.
\end{assumption*}
\begin{assumption*} \label{A6}
The asymptotic regime of $a_n$, the size of the subsampling without replacement, and the number of terminal leaves $t_n$ is such that $a_n \leq n-2$, $a_n/n < 1 - \kappa$ for a fixed $\kappa > 0$, $\lim \limits_{n \to \infty} a_n = \infty$, $\lim \limits_{n \to \infty} t_n = \infty$, and $\lim \limits_{n \to \infty} 2^{t_n} \frac{(\log(a_n))^9}{a_n} = 0$.
\end{assumption*}
\begin{theorem*} \label{thm_MDA_sobol}
    If Assumptions \ref{A1}, \ref{A5}, and \ref{A6} are satisfied, for all $M \in \mathbb{N}^{\star}$ and $j \in \{1,\hdots,p\}$
    \begin{align*}
        \widehat{\textrm{S-MDA}}_{M,n}(X^{(j)}) \overset{p}{\longrightarrow} ST^{(j)}.
    \end{align*}
\end{theorem*}

The consistency of the Sobol-MDA relies on the consistency of the projected random forest, stated in Lemma \ref{consistency_cart_proj}, and Lemma \ref{lemme_sobol_oob_consistency} for the corresponding OOB estimate. Lemma \ref{lemme_diam} is an intermediate result on the asymptotic behavior of the original forest.
Under the small modifications of the random forest algorithm defined by Assumption \ref{A5}, Lemma \ref{lemme_diam} states that the cells of a random tree in the empirical forest become infinitely small as the sample size increases.
For a cell $A \in [0,1]$, we define $\textrm{diam}(A)$ the diameter of a cell as
\begin{align*}
    \textrm{diam}(A) = \sup_{\bx,\bx' \in A} ||\bx - \bx'||_2.
\end{align*}
Recall that $A_n(\bX, \Theta)$ is the cell of the original $\Theta$-random CART where $\bX$ falls.
\begin{lemma*} \label{lemme_diam}
    If Assumptions \ref{A1}, \ref{A5}, and \ref{A6} are satisfied, we have in probability
    \begin{align*}
        \lim \limits_{n \to \infty} \textrm{diam}(A_n(\bX, \Theta)) = 0.
    \end{align*}
\end{lemma*}

The following lemma states that the Projected-CART estimate is consistent.
Recall that $A_n^{(-j)}(\bX^{(-j)}, \Theta)$ is the cell of the projected partition where $\bX^{(-j)}$ falls, $\smash{m_n^{(-j)}(\bX^{(-j)}, \Theta)}$ is the associated projected tree, and $\smash{m_n^{(-j)}(\bX^{(-j)}) = \E[m_n^{(-j)}(\bX^{(-j)}, \Theta)|\Dn, \bX^{(-j)}]}$ is the projected infinite forest estimate.
We also define $m^{(-j)}(\bz) = \E[m(\bX)|\bX^{(-j)} = \bz]$ for $\bz \in [0,1]^{p-1}$.
\begin{lemma*} \label{consistency_cart_proj}
    If Assumptions \ref{A1}, \ref{A5}, and \ref{A6} are satisfied, we have for $j \in \{1,\hdots,p\}$
    \begin{align*}
        \lim \limits_{n \to \infty} \E[(m_n^{(-j)}(\bX^{(-j)}) - m^{(-j)}(\bX^{(-j)}))^2] = 0.
    \end{align*}
\end{lemma*}

\begin{lemma*} \label{lemme_sobol_oob_consistency}
    If Assumptions \ref{A1}, \ref{A5}, and \ref{A6} are satisfied, for all $i \in \{1,\hdots,n\}$, $j \in \{1,\hdots,p\}$, and $M \in \mathbb{N}^{\star}$ we have
    \begin{align*}
    \lim \limits_{n \to \infty} \E[(m_{M,n}^{(-j,OOB)}(\bX_i^{(-j)}, \bTheta_{M}) - m(\bX_i^{(-j)}))^2\mathds{1}_{|\Lambda_{n,i}| > 0} ] = 0.
    \end{align*}
\end{lemma*}

\begin{proof}[Proof of Theorem \ref{thm_MDA_sobol}]
We assume that Assumptions \ref{A1}, \ref{A5}, and \ref{A6} are satisfied and consider $j \in\{1,\hdots,p\}$.
We can exactly follow the proof of Theorem \ref{thm_MDA}-(iii) by only replacing $\E[m(\Xperm)|\bX^{(-j)}]$ by $\E[m(\bX)|\bX^{(-j)}]$ in the main decomposition, and get the $\mathbb{L}^1$-consistency of the unnormalized Sobol-MDA using Lemmas \ref{lemme_oob_consistency} and \ref{lemme_sobol_oob_consistency}. Finally, the Sobol-MDA is normalized by the standard variance estimate $\hat{\sigma}_Y$ of the output $Y$, which is consistent by the Law of Large Numbers. Next, according to the continuous mapping theorem $1/\hat{\sigma}_Y \overset{p}{\longrightarrow} 1/\V[Y]$. 
Overall, the Sobol-MDA is the product of two random quantities which convergence in probability, and we have
\begin{align*}
    \widehat{\textrm{S-MDA}}_{M,n}(X^{(j)}) \overset{p}{\longrightarrow} ST^{(j)}.
\end{align*}
\end{proof}

The brute force approach of retraining the forest with the data $\smash{\Dn^{(-j)}}$, where covariate $\smash{X^{(j)}}$ is removed, also estimates the total Sobol index, as proved in Theorem \ref{thm_MDA_retrain} below. The associated forest estimate is then denoted by $\smash{m_{M,n}(X^{(-j)}, \Dn^{(-j)})}$.
\begin{theorem*} \label{thm_MDA_retrain}
    If Assumptions \ref{A1}, \ref{A5}, and \ref{A6} are satisfied, for all $M \in \mathbb{N}^{\star}$ and $j \in \{1,\hdots,p\}$
    \begin{align*}
        \frac{\E[(Y - m_{M,n}(X^{(-j)}, \Dn^{(-j)}))^2]}{\V[Y]}& - \frac{\E[(Y - m_{M,n}(X))^2]}{\V[Y]}  \longrightarrow ST^{(j)}.
    \end{align*}
\end{theorem*}
\begin{proof}[Proof of Theorem \ref{thm_MDA_retrain}]
We first need to show the $\mathbb{L}^2$-consistency of random forests under Assumptions $1$, $5$, and $6$.
Under these assumptions, Lemma \ref{lemme_diam} gives that
\begin{align*}
    \lim \limits_{n \to \infty} \textrm{diam}(A_n(\bX, \Theta)) = 0.
\end{align*}
Then, we follow the proof of Lemma \ref{consistency_cart_proj} to get the $\mathbb{L}^2$-consistency of random forests under Assumptions $1$, $5$, and $6$. Next, we break down the following quantity,
\begin{align*}
    \E[(Y - m_{M,n}(\bX))^2] &= 
    \E[(m(\bX) - m_{M,n}(\bX) + \varepsilon)^2] \\ 
    &= \E[(m(\bX) - m_{M,n}(\bX))^2] + \sigma^2
    + 2\E[\varepsilon(m(\bX) - m_{M,n}(\bX))],
\end{align*}
where the last term is null because $\varepsilon$ is centered and independent from $\Dn$, $\bTheta_M$, and $\bX$ by construction, and $\sigma^2 = \V[\varepsilon]$. Finally, since the forest estimate is $\mathbb{L}^2$-consistent, the first term converges towards zero, and we get
\begin{align*}
    \E[(Y - m_{M,n}(\bX))^2] \rightarrow \sigma^2.
\end{align*}
For the second term, we can write
\begin{align*}
    \E[(Y - m_{M,n}(&\bX^{(-j)}, \Dn^{(-j)}))^2] \\
    =& \E[(m(\bX) - m^{(-j)}(\bX^{(-j)}) + (m^{(-j)}(\bX^{(-j)}) - m_{M,n}(\bX^{(-j)}, \Dn^{(-j)})) + \varepsilon)^2] \\ 
    =& \E[(m(\bX) - m^{(-j)}(\bX^{(-j)}))^2] + \sigma^2
    + 2\E[\varepsilon(m(\bX) - m^{(-j)}(\bX^{(-j)}))] \\
    &+ \E[(m^{(-j)}(\bX^{(-j)}) - m_{M,n}(\bX^{(-j)}, \Dn^{(-j)}))^2] \\
    &+ 2\E[\varepsilon(m^{(-j)}(\bX^{(-j)}) - m_{M,n}(\bX^{(-j)}, \Dn^{(-j)}))] \\
    &+ 2\E[(m(\bX) - m^{(-j)}(\bX^{(-j)}))(m^{(-j)}(\bX^{(-j)}) - m_{M,n}(\bX^{(-j)}, \Dn^{(-j)}))].
\end{align*}
As above, the third and fifth terms are null because $\varepsilon$ is independent of the other random variables involved.
The forest estimate is $\mathbb{L}^2$-consistent, and the result is valid for any dimension of the input vector, and in particular for $p-1$ when $X^{(j)}$ is removed from the data. Therefore, the fourth term is null asymptotically. It is also the case for the last term, since $m$ is bounded (continuous on a compact) and $\mathbb{L}^2$-convergence implies $\mathbb{L}^1$-convergence.
Overall, we get that
\begin{align*}
    \E[(Y - m_{M,n}(\bX^{(-j)},& \Dn^{(-j)}))^2] 
    \rightarrow \E[(m(\bX) - m^{(-j)}(\bX^{(-j)}) )^2] + \sigma^2.
\end{align*}
We can rewrite 
\begin{align*}
    \E[(m(\bX) - m^{(-j)}(\bX^{(-j)}) )^2]
    &= \E[(m(\bX) - \E[m(\bX) \mid \bX^{(-j)}])^2] \\
    &= \E[\V[m(\bX) \mid \bX^{(-j)}]] \\
    &= \V[Y] \times ST^{(j)}.
\end{align*}
Finally, we have
\begin{align*}
    \frac{\E[(Y - m_{M,n}(\bX^{(-j)}, \Dn^{(-j)}))^2]}{\V[Y]}& - \frac{\E[(Y - m_{M,n}(\bX))^2]}{\V[Y]}  \longrightarrow ST^{(j)}.
\end{align*}
\end{proof}

\begin{proof}[Proof of Lemma \ref{lemme_diam}]
    The proof is inspired by Lemma $2$ from \citet{meinshausen2006quantile}.
    We define $s_n(\bX, \Theta)$ as the number of splits to reach the terminal cell $A_n(\bX, \Theta)$ where $\bX$ falls.
    The asymptotic regime of the tree growing is controlled by Assumption \ref{A6} by setting the number of terminal leaves to $t_n$.
    Since $A_n(\bX, \Theta)$ is a terminal leave, there are two possible cases: further splitting $A_n(\bX, \Theta)$ will necessarily lead to cells with a number of observations smaller than the algorithm parameter \texttt{minimum node size}, that we call $N_{min}$, and is typically equal to $5$ in practice. Formally, it means that 
    \begin{align} \label{eq_terminal_1}
        N_n(\bX, \Theta) < 2 N_{min},
    \end{align}
    where $N_n(\bX, \Theta)$ is the number of observations in $A_n(\bX, \Theta)$. The other possibility is that the total number of leaves $t_n$ is reached, which implies that
    \begin{align*}
        2^{s_n(\bX, \Theta)} \geq t_n,
    \end{align*}
    the equality case happening if the tree is balanced.
    Next, according to Assumption \ref{A5}, all children nodes have at least a fraction $0.5 > \gamma > 0$ of the parent node observations. Then we have $a_n\gamma^{s_n(\bX,\Theta)} \leq N_n(\bX, \Theta)$. Combining this last inequality with (\ref{eq_terminal_1}), we obtain $a_n\gamma^{s_n(\bX,\Theta)} < 2 N_{min}$. 
    Overall, at least one of the two following inequalities is satisfied
    \begin{align*}
        &s_n(\bX, \Theta) \geq \log_2(t_n) \\
        &s_n(\bX,\Theta) > \frac{\log_2(a_n/2 N_{min})}{\log_2(1/\gamma)}.
    \end{align*}
    From Assumption \ref{A6}, $a_n \rightarrow \infty$ and $t_n \rightarrow \infty$. Therefore, we can conclude that
    \begin{align} \label{inf_splits}
        s_n(\bX, \Theta) \overset{p}{\longrightarrow} \infty.
    \end{align}
    
    Now, we fix $j \in \{1,\hdots,p\}$, and define $s_n^{(j)}(\bX, \Theta)$ as the number of splits involving the $j$-th variable in the path to $A_n(\bX, \Theta)$. According to Assumption \ref{A5}, variable $j$ can be selected at each node with probability at least $\delta/p$. Combined with result (\ref{inf_splits}), we consequently have
    \begin{align} \label{inf_splits_j}
        s_n^{(j)}(\bX, \Theta) \overset{p}{\longrightarrow} \infty.
    \end{align}
    
    Next, we break down the cell $A_n(\bX, \Theta)$ with a collection of intervals for each of the $p$ directions: 
    \begin{align*}
        A_n(\bX, \Theta) = \bigotimes_{j=1}^{p} A_n^{(j)}(\bX, \Theta),
    \end{align*}
    where each $A_n^{(j)}(\bX, \Theta)$ is an interval and can be written as $A_n^{(j)}(\bX, \Theta) = [l_n^{(j)}(\bX, \Theta), u_n^{(j)}(\bX, \Theta)]$.
    Then, we can bound from above the number $N_n^{(j)}(\bX, \Theta)$ of observations whose $j$-th coordinate belongs to $A_n^{(j)}(\bX, \Theta)$ using Assumption \ref{A2},
    \begin{align*}
        N_n^{(j)}(\bX, \Theta) \leq a_n(1-\gamma)^{s_n^{(j)}(\bX,\Theta)},
    \end{align*}
    and using (\ref{inf_splits_j}), we get that 
    \begin{align*}
        N_n^{(j)}(\bX, \Theta)/a_n \overset{p}{\longrightarrow} 0.
    \end{align*}
    Next, we introduce $F_{a_n}^{(j)}$ the empirical cdf of $X^{(j)}$, estimated with the $\Theta^{(S)}$-subsample of $\Dn$. Similarly, $F^{(j)}$ denotes the cdf of $X^{(j)}$.
    By definition, we have
    \begin{align} \label{cdf_lim}
        N_n^{(j)}(\bX, \Theta)/a_n = F_{a_n}^{(j)}(u_n^{(j)}(\bX, \Theta)) - F_{a_n}^{(j)}(l_n^{(j)}(\bX, \Theta))\overset{p}{\longrightarrow} 0.
    \end{align}
    On the other hand, we can write
    \begin{align*}
        F^{(j)}(u_n^{(j)}(\bX, \Theta)) - F^{(j)}(l_n^{(j)}(\bX, \Theta)) = &
        F_{a_n}^{(j)}(u_n^{(j)}(\bX, \Theta)) - F_{a_n}^{(j)}(l_n^{(j)}(\bX, \Theta)) \\ &- [F_{a_n}^{(j)}(u_n^{(j)}(\bX, \Theta)) - F^{(j)}(u_n^{(j)}(\bX, \Theta))]  \\ &+ [F_{a_n}^{(j)}(l_n^{(j)}(\bX, \Theta)) - F^{(j)}(l_n^{(j)}(\bX, \Theta))],
    \end{align*}
    and we get the following bound
    \begin{align*}
        F^{(j)}(u_n^{(j)}(\bX, \Theta)) - F^{(j)}(l_n^{(j)}(\bX, \Theta)) \leq
        F_{a_n}^{(j)}&(u_n^{(j)}(\bX, \Theta)) - F_{a_n}^{(j)}(l_n^{(j)}(\bX, \Theta)) \\ &+ 2 \sup_{z \in [0,1]} |F_{a_n}^{(j)}(z) - F^{(j)}(z)|.
    \end{align*}
    The Glivenko-Cantelli Theorem gives that 
    \begin{align*}
        \sup_{z \in [0,1]} |F_{a_n}^{(j)}(z) - F^{(j)}(z)| \overset{p}{\longrightarrow} 0,
    \end{align*}
    and combined with (\ref{cdf_lim}), we obtain
    \begin{align} \label{cdf_lim_th}
        F^{(j)}(u_n^{(j)}(\bX, \Theta)) - F^{(j)}(l_n^{(j)}(\bX, \Theta)) \overset{p}{\longrightarrow} 0.
    \end{align}
    Finally, using the integral form of the difference above, we have
    \begin{align*}
        F^{(j)}(u_n^{(j)}(\bX, \Theta)) - F^{(j)}(l_n^{(j)}(\bX, \Theta)) = \int_{A_n^{(j)}(\bX, \Theta)} f^{(j)}(x)dx,
    \end{align*}
    and since $f^{(j)}$ is lower bounded by $c_1$ according to Assumption \ref{A1},
    \begin{align*}
        F^{(j)}(u_n^{(j)}(\bX, \Theta)) - F^{(j)}(l_n^{(j)}(\bX, \Theta)) \geq c_1 \textrm{diam}(A_n^{(j)}(\bX, \Theta)).
    \end{align*}
    This last inequality combined with limit (\ref{cdf_lim_th}) gives
    \begin{align*}
        \textrm{diam}(A_n^{(j)}(\bX, \Theta)) \overset{p}{\longrightarrow} 0,
    \end{align*}
    and since this is true for each direction $j = 1,\hdots,p$, the final result follows. Then, we have in probability
     \begin{align*}
        \lim \limits_{n \to \infty} \textrm{diam}(A_n(\bX, \Theta)) = 0.
    \end{align*}
\end{proof}

The proof of Lemma \ref{consistency_cart_proj} is based on Theorem 10.2 from \citet{gyorfi2006distribution} and Theorem 1 from \citet{scornet2015consistency}. First, we introduce several notations following \citet{scornet2015consistency}.
The partition of $[0,1]^{p-1}$ obtained with the $\Theta$-random tree projected along the $j$-th direction is denoted by $\mathcal{P}^{(-j)}_n(\Dn,\Theta)$. We define the family of all achievable partitions with $\Theta$ as
\begin{align*}
    \Pi_n^{(-j)}(\Theta) = \{\mathcal{P}^{(-j)}((\bx_1,y_1),\hdots,(\bx_n,y_n), \Theta):(\bx_i,y_i)\in[0,1]^{p-1}\times\R\},
\end{align*}
and the associated maximal number $M(\Pi_n^{(-j)}(\Theta))$ of terminal nodes among all partitions in $\Pi_n^{(-j)}(\Theta)$ is
\begin{align*}
    M(\Pi_n^{(-j)}(\Theta)) = \max\{|\mathcal{P}|:\mathcal{P}\in\Pi_n^{(-j)}(\Theta)\}.
\end{align*}
Next, we consider $\bz_1, \hdots, \bz_n \in [0,1]^{p-1}$ and denotes $\Gamma(\bz_1, \hdots, \bz_n, \Pi_n^{(-j)}(\Theta))$ the number of distinct partitions of $\bz_1, \hdots, \bz_n$ induced by the elements of $\Pi_n^{(-j)}(\Theta)$. Then, the partitioning number $\Gamma(\Pi_n^{(-j)}(\Theta))$ is defined as
\begin{align*}
    \Gamma(\Pi_n^{(-j)}(\Theta)) = \max \{ \Gamma(\bz_1, \hdots, \bz_n, \Pi_n^{(-j)}(\Theta)) : \bz_1, \hdots, \bz_n \in [0,1]^{p-1} \}.
\end{align*}

We define the truncated operator $T_L$ for $L > 0$. Thus, the truncated tree estimate $T_{L}m_n^{(-j)}(\bX^{(-j)}, \Theta)$ returns the constant $L$ whenever $|m_n^{(-j)}(\bX^{(-j)}, \Theta)| > L$.
Finally, we define $\mathcal{F}_n^{(-j)}(\Theta)$ the set of piecewise constant functions over the partition $\mathcal{P}^{(-j)}_n(\Dn,\Theta)$.
Then, the projected tree estimate $m_n^{(-j)}(\bX^{(-j)}, \Theta)$ is defined as the element of $\mathcal{F}_n^{(-j)}(\Theta)$ which minimizes the quadratic risk.

For the sake of clarity, we recall Theorem 10.2 from \citet{gyorfi2006distribution}, as presented in \citet{scornet2015consistency} in the case of random forests.
\begin{theorem*}[Theorem 10.2 in \citet{gyorfi2006distribution}] \label{thm_102}
    Assume that
    \begin{align*}
        (i)&\lim_{n \to \infty} \beta_n = \infty, \\
        (ii)&\lim_{n \to \infty} \E\big[\inf_{f \in \mathcal{F}^{(-j)}_n(\Theta), ||f||_{\infty} \leq \beta_n} \E[ (f(\bX^{(-j)}) - m^{(-j)}(\bX^{(-j)}))^2 ] \big] = 0, \\
        (iii)& \textrm{ for all } L > 0,\\
        & \lim_{n \to \infty} \E\Big[\sup_{\begin{tabular}{c}
             \small{$f \in \mathcal{F}_n^{(-j)}(\Theta),$} \\ \small{$||f||_{\infty} \leq \beta_n$} \end{tabular}} \Big|\frac{1}{a_n} \sum_{i \in \Theta^{(S)}} [f(\bX_i^{(-j)}) - Y_{i,L}]^2 - \E[(f(\bX^{(-j)}) - Y_{L})^2 ] \Big| \Big] = 0.
    \end{align*}
    Then, we have
    \begin{align*}
        \lim \limits_{n \to \infty} \E[(T_{\beta_n} m_n^{(-j)}(\bX^{(-j)}) - m^{(-j)}(\bX^{(-j)}))^2] = 0.
    \end{align*}
\end{theorem*}

\begin{proof}[Proof of Lemma \ref{consistency_cart_proj}]
    We assume that Assumptions \ref{A1}, \ref{A5}, and \ref{A6} are satisfied, and we fix $j \in \{1,\hdots,p\}$.
    We closely follow the proof of Theorem 1 from \citet{scornet2015consistency} to adapt it to the case of projected forest.
    
    (i) We set $\beta_n = ||m||_{\infty} + \V[\varepsilon] \sqrt{2} \log^2(a_n)$. By definition, $\beta_n \rightarrow \infty$ and (i) is satisfied.
    
    (ii) Approximation Error.
    Fix $\xi > 0$. We can show that (see \citet[page 17]{scornet2015consistency} for the details), for $n$ large enough such that $\beta_n > ||m||_{\infty}$,
    \begin{align*}
        \E\Big[\inf_{\begin{tabular}{c} \small{$f \in \mathcal{F}^{(-j)}_n(\Theta),$} \\ \small{$||f||_{\infty} \leq \beta_n$} \end{tabular}} \E[ (f(\bX^{(-j)}) - &m^{(-j)}(\bX^{(-j)}))^2 ] \Big] \\[-2em] &< \xi^2 + 4||m||_{\infty}^2
        \P(\Delta(m, A_n^{(-j)}(\bX^{(-j)}, \Theta)) > \xi).
    \end{align*}
    On the other hand, observe that $A_n^{(-j)}(\bX^{(-j)}, \Theta)$ is included in the projection of $A_n(\bX, \Theta)$ along the $j$-th direction by construction---see Figure \ref{fig_proj_CART_2} for an illustration. Furthermore, when a cell is projected, its diameter is smaller than the original one. 
    Thus, we have
    \begin{align*}
        \textrm{diam}(A_n^{(-j)}(\bX^{(-j)}, \Theta)) \leq \textrm{diam}(A_n(\bX, \Theta)).
    \end{align*}
    and consequently Lemma  \ref{lemme_diam} implies that in probability
    \begin{align*}
        \lim \limits_{n \to \infty} \textrm{diam}(A_n^{(-j)}(\bX^{(-j)}, \Theta)) = 0.
    \end{align*}
    Since $m$ is continuous, the control on the cell diameter implies that 
    \begin{align*}
        \Delta(m, A_n^{(-j)}(\bX^{(-j)}, \Theta)) \overset{p}{\longrightarrow} 0.
    \end{align*}
    This enables to control the approximation error, i.e., for $n$ large enough
    \begin{align*}
        \E\Big[\inf_{\begin{tabular}{c} \small{$f \in \mathcal{F}^{(-j)}_n(\Theta),$} \\ \small{$||f||_{\infty} \leq \beta_n$} \end{tabular}} \E[ (f(\bX^{(-j)}) - &m^{(-j)}(\bX^{(-j)}))^2 ] \Big] < 2 \xi^2,
    \end{align*}
    and therefore (ii) is satisfied.
    
    (iii) Estimation Error.
    The number of terminal leaves in the original tree is $t_n$. Consequently, the number of leaves in the projected tree is upper bounded by $2^{t_n}$. Thus, by definition 
    $M(\Pi_n^{(-j)}(\Theta)) \leq 2^{t_n}$, and simple calculations give $\Gamma(\Pi_n^{(-j)}(\Theta)) \leq [(p-1)a_n]^{2^{t_n}}$.
    Since Assumption \ref{A6} ensures that $\lim \limits_{n \to \infty} 2^{t_n} \frac{(\log(a_n))^9}{a_n} = 0$, we can show (iii) exactly as in \citet[page 17-18]{scornet2015consistency}.
    
    Since (i), (ii), and (iii) are satisfied, Theorem \ref{thm_102} gives the consistency of the truncated projected tree estimate,
   \begin{align*}
        \lim \limits_{n \to \infty} \E[(T_{\beta_n} m_n^{(-j)}(\bX^{(-j)}) - m^{(-j)}(\bX^{(-j)}))^2] = 0.
    \end{align*}
    
    Finally, the extension to the untruncated projected tree estimate strictly follows \citet[pages 18-19]{scornet2015consistency} when the noise is Gaussian, and is still valid for our case of a sub-Gaussian noise (Assumption \ref{A1}). Overall, we have
        \begin{align*}
        \lim \limits_{n \to \infty} \E[(m_n^{(-j)}(\bX^{(-j)}) - m^{(-j)}(\bX^{(-j)}))^2] = 0.
    \end{align*}
\end{proof}

\begin{figure}
\centering
\begin{tikzpicture}

\draw[->] (-0,0) -- (5.5,0);
\draw (5.5,0) node[right] {$X^{(1)}$};
\draw [->] (0,-0) -- (0,5);
\draw (0,5) node[above] {$X^{(2)}$};
\draw [color = blue, line width = 0.4 mm] (3.5,2) -- (3.5,5);
\draw [color = blue, line width = 0.4 mm] (1,2) -- (1,3.5);
\draw [color = blue, line width = 0.4 mm] (2.5,2) -- (2.5,-0);
\draw [color = blue, line width = 0.4 mm] (5.5,2) -- (-0,2);
\draw [color = blue, line width = 0.4 mm] (-0,3.5) -- (3.5,3.5);
\draw [color = blue, line width = 0.4 mm] (4.8,2) -- (4.8,-0);
\draw (1.7,3.1) node[right] {$A_n(\bX, \Theta)$};
\filldraw (1.5,2.5) circle[radius=1.5pt];
\draw (1.6,2.5) node[right] {$\bX$};

\draw[->, line width = 0.4 mm, color = red] (8,0) -- (13.5,0);
\draw (13.5,0) node[right] {$X^{(1)}$};
\draw [->, dashed] (8,-0) -- (8,5);
\draw (8,5) node[above] {$X^{(2)}$};
\draw [color = blue, line width = 0.4 mm, dashed] (11.5,2) -- (11.5,5);
\draw [color = blue, line width = 0.4 mm, dashed] (9,2) -- (9,3.5);
\draw [color = blue, line width = 0.4 mm, dashed] (10.5,2) -- (10.5,-0);
\draw [color = blue, line width = 0.4 mm, dashed] (13.5,2) -- (8,2);
\draw [color = blue, line width = 0.4 mm, dashed] (8,3.5) -- (11.5,3.5);
\draw [color = blue, line width = 0.4 mm, dashed] (12.8,2) -- (12.8,0);
\draw [color = red, line width = 0.2 mm] (11.5,-0) -- (11.5,5);
\draw [color = red, line width = 0.2 mm] (9,-0) -- (9,5);
\draw [color = red, line width = 0.2 mm] (10.5,5) -- (10.5,-0);
\draw [color = red, line width = 0.2 mm] (12.8,5) -- (12.8,-0);
\filldraw (9.5,2.5) circle[radius=1.5pt];
\draw (9.6,2.5) node[right] {$\bX$};
\draw [line width = 0.2 mm, dashed] (9.5,2.5) -- (9.5,0);
\filldraw (9.5,0) circle[radius=1.5pt];
\draw (9.5,0) node[above right] {$\bX^{(-j)}$};
\draw [decorate,decoration={brace,mirror,amplitude=5pt},xshift=0pt,yshift=-2pt]
(9,0) -- (10.5,0) node [black,midway,yshift=-0.4cm] 
{\footnotesize $A_n^{(-j)}(\bX^{(-j)}, \Theta)$};
\draw [color = orange] (9.8,1.5) node[right] {$A_1$};
\draw [color = orange] (9.8,3) node[right] {$A_2$};
\draw [color = orange] (9.8,4.5) node[right] {$A_3$};

\end{tikzpicture}
\caption{Example of the partition of $[0,1]^2$ by a random CART tree (left side) projected on the subspace span by $\bX^{({-2})} = X^{(1)}$ (right side). Here, $p = 2$ and $j = 2$.}
\label{fig_proj_CART_2}
\end{figure}

\begin{proof}[Proof of Lemma \ref{lemme_sobol_oob_consistency}]
    We assume that Assumptions \ref{A1}, \ref{A5}, and \ref{A6} are satisfied, and we fix $j \in \{1,\hdots,p\}$.
    First, we expand the considered risk
    \begin{align*}
    \E[(m_{M,n}^{(-j,OOB)}(\bX_i, \bTheta_{M})& - m(\bX_i^{(-j)}))^2\mathds{1}_{|\Lambda_{n,i}| > 0} ] \\
    =& \E\Big[ \Big(\frac{1}{|\Lambda_{n,i}|} \sum_{\ell \in \Lambda_{n,i}} [m_n^{(-j)}(\bX_i^{(-j)}, \Theta_{\ell}) - m(\bX_i^{(-j)})]\mathds{1}_{|\Lambda_{n,i}| > 0} \Big)^2 \Big].
    \end{align*}
    Then, identically to the proof of Lemma \ref{lemma_oob_risk}, we can handle the randomness of the selected batch of trees $\Lambda_{n,i}$, and bound the OOB risk with the risk of the standard projected forest, i.e.,
    \begin{align*}
        \E\big[\big( m_{M,n}^{(-j, OOB)}(\bX_i^{(-j)}, \bTheta_{M})& - m(\bX_i^{(-j)}) \big)^2 \mathds{1}_{|\Lambda_{n,i}| > 0}\big] \\
        &\leq \frac{2}{1 - a_n/n} \E\big[\big(m_{M,a_n,n-1}^{(-j)}(\bX^{(-j)}, \bTheta_{M}) 
        - m(\bX^{(-j)})\big)^2 \big].
    \end{align*}
    Lemma \ref{consistency_cart_proj} gives the consistency of the infinite projected forest, which also implies the consistency of the finite projected forest, that is
     \begin{align*}
         \E\big[\big(m_{M,a_n,n-1}^{(-j)}(\bX^{(-j)}, \bTheta_{M}) 
        - m(\bX^{(-j)})\big)^2 \big] \longrightarrow 0.
    \end{align*}
    Additionally, from Assumption \ref{A6}, $a_n/n < 1 - \kappa$ with $\kappa > 0$, and thus
    \begin{align*}
        \lim \limits_{n \to \infty} \E[(m_{M,n}^{(-j,OOB)}(\bX_i^{(-j)}, \bTheta_{M}) - m(\bX_i^{(-j)}))^2\mathds{1}_{|\Lambda_{n,i}| > 0} ] = 0.
    \end{align*}
\end{proof}

\section{MDA Software Implementations}  \label{appendix_soft}

We provide detailed references of the MDA implementations of the main random forest packages:
\begin{enumerate}
    \item \texttt{scikit-learn 0.24} \\ (\textit{https://scikit-learn.org/stable/})
    \item \texttt{randomForest 4.6-14} \\ (\textit{https://cran.r-project.org/web/packages/randomForest/index.html})
    \item \texttt{ranger 0.12.1} \\ (\textit{https://cran.r-project.org/web/packages/ranger/index.html})
    \item \texttt{randomForestSRC 2.9.3} \\ (\textit{https://cran.r-project.org/web/packages/randomForestSRC/index.html})
\end{enumerate}

\subsection{\texttt{scikit-learn 0.24}}

In \texttt{scikit-learn}, the MDA is not specific for random forests, but is a generic procedure taking a trained model and an independent testing sample as inputs.
The MDA implementation is located in the file: ``scikit-learn/sklearn/inspection/\_permutation\_importance.py''.

The method \textit{\_calculate\_permutation\_scores(estimator, X, y, sample\_weight, col\_idx, random\_state, n\_repeats, scorer)} computes the error of the model \textit{estimator} when the column of index \textit{col\_idx} of the testing sample \textit{X} is permuted, over multiple repetitions defined by the parameter \textit{n\_repeats}. The model error is defined by \textit{scorer}, and \textit{random\_state} defines the random seed. Finally, the permuted and the original errors are subtracted and the multiple repetitions are aggregated in the method \textit{permutation\_importance(estimator, X, y, *, scoring=None, n\_repeats=5, n\_jobs=None, random\_state=None)} which thus implements the Train/Test MDA.

\subsection{\texttt{randomForest 4.6-14}}

The \texttt{R} script ``randomForest/R/importance.R'' implements the function\\ \textit{importance.randomForest <- function(x, type=NULL, class=NULL, scale=TRUE, ...)} between lines $6$ and $44$, where \textit{x} is a fitted forest, which as the attribute \textit{x\$importance} storing the Breiman-Cutler MDA and the standard deviation of the risk differences across trees, computed with the script ``randomForest/src/regrf.c'' for regression forests. The function \textit{importance.randomForest} handles exceptions and normalizes the MDA with the standard deviations, and thus implements the normalized Breiman-Cutler MDA.

For regression forests, the \texttt{C} script ``randomForest/src/regrf.c'' computes the difference between the permuted and original errors for each tree between lines $262$ and $295$. The associated means and standard deviations across all trees are computed between lines $327$ and $338$. These computations are done right after the forest construction at the end of the method \textit{void regRF}.

\subsection{\texttt{ranger 0.12.1}}

In \texttt{ranger}, the MDA is computed during the forest growing by specifying the paramater \textit{importance = 'permutation'} in the call to the main function \textit{ranger}.
For each tree of the forest, the accuracy decrease is computed in the \texttt{C++} file ``ranger/src/Tree.cpp'' with the method \textit{void Tree::computePermutationImportance()}, located between lines $206$ and $255$. Next, the importance measures are averaged over all trees with the method \textit{void Forest:: computePermutationImportance()} between lines $646$ and $763$ of the \texttt{C++} file ``ranger/src/Forest.cpp'', and thus the BC-MDA is computed. If the paramater \textit{scale.permutation.importance} is set to \textit{True}, then the normalized BC-MDA is computed (default value is \textit{False}).

\subsection{\texttt{randomForestSRC 2.9.3}}

The package \texttt{randomForestSRC} can compute the three types of MDA. 
The function \textit{vimp.rfsrc} (lines $1$ to $82$ of file ``randomForestSRC/R/vimp.rfsrc.R'') computes the MDA, and takes a fitted forest \textit{object} as an input. If an independent testing sample is provided as the input \textit{newdata}, TT-MDA is computed. Otherwise if \textit{importance = 'permute'}, the IK-MDA by blocks is estimated: the trees of the forest are divided in multiple blocks and the IK-MDA is computed for each block and averaged. The parameter \textit{block.size} set the number of trees in each block, $10$ by default. If $\textit{block.size} = 1$, this procedure is the BC-MDA.

The function \textit{vimp.rfsrc} computes the MDA calling a chain of \texttt{C} subroutines, located in the file ``randomForestSRC/src/randomForestSRC.c'' between lines $2026$ and $2564$: \textit{permute}, \textit{getPermuteMembership}, \textit{getVimpMembership}, \textit{updateVimpEnsemble}, \textit{summarizePerturbedPerformance}, and \textit{finalizeVimpPerformance}.

\section{Analytical Example Computations} \label{appendix_example}

We first recall the analytical example definition, and all computations are provided next.
The input $\bX$ is a Gaussian vector of dimension $p = 5$. Its covariance matrix is defined by $\smash{\V[X^{(j)}] = \sigma_j^2}$ for $\smash{j \in \{1,\hdots,5\}}$, and all covariance terms are null except 
\begin{align*}
    \textrm{Cov}[X^{(1)},X^{(2)}] = \rho_{1,2} \sigma_1 \sigma_2,
\end{align*}
and
\begin{align*}
    \textrm{Cov}[X^{(4)},X^{(5)}] = \rho_{4,5} \sigma_4 \sigma_5.
\end{align*}
The regression function $m$ is given by
\begin{align*}
    m(\bX) = \alpha X^{(1)} X^{(2)} \mathds{1}_{X^{(3)} > 0} + \beta X^{(4)} X^{(5)} \mathds{1}_{X^{(3)} < 0}. 
\end{align*}

\subsection{Total Sobol Index $\pmb{ST^{(1)}}$.}
By definition, $\V[Y] \times ST^{(1)} = \E[\V[m(\bX) | \bX^{(-1)}]]$. Since $X^{(1)}$ and $X^{(2)}$ are independent of $X^{(3)}$, $X^{(4)}$, and $X^{(5)}$, we have
\begin{align*}
    \E[m(\bX) | \bX^{(-1)}] &= \E[ \alpha X^{(1)} X^{(2)} \mathds{1}_{X^{(3)} > 0} + \beta X^{(4)} X^{(5)} \mathds{1}_{X^{(3)} < 0} | \bX^{(-1)}] \\
    &= \E[ \alpha X^{(1)} X^{(2)} \mathds{1}_{X^{(3)} > 0} | X^{(2)}] + \beta X^{(4)} X^{(5)} \mathds{1}_{X^{(3)} < 0} \\
    &= \alpha X^{(2)} \E[ X^{(1)} | X^{(2)}] \mathds{1}_{X^{(3)} > 0} + \beta X^{(4)} X^{(5)} \mathds{1}_{X^{(3)} < 0}.
\end{align*}
Since $(X^{(1)}, X^{(2)})$ is a bivariate centered Gaussian vector,
\begin{align*}
    \E[ X^{(1)} | X^{(2)}] = \rho_{1,2} \frac{\sigma_1}{\sigma_2} X^{(2)},
\end{align*}
and then
\begin{align*}
    \E[m(\bX) | \bX^{(-1)}] &= \alpha \rho_{1,2} \frac{\sigma_1}{\sigma_2} X^{(2) 2} \mathds{1}_{X^{(3)} > 0} + \beta X^{(4)} X^{(5)} \mathds{1}_{X^{(3)} < 0}.
\end{align*}
Next, we compute
\begin{align*}
    \E[\V[m(\bX) | \bX^{(-1)}]] &= \E[(m(\bX) - \E[ m(\bX) | \bX^{(-1)}])^2] \\ 
    &= \E[ (\alpha X^{(1)} X^{(2)} \mathds{1}_{X^{(3)} > 0} - \alpha \rho_{1,2} \frac{\sigma_1}{\sigma_2} X^{(2) 2} \mathds{1}_{X^{(3)} > 0})^2] \\
    &= \frac{\alpha^2}{2} \E[(X^{(1)} X^{(2)} - \rho_{1,2} \frac{\sigma_1}{\sigma_2} X^{(2) 2})^2] \\
    &= \frac{\alpha^2}{2} \big( \E[(X^{(1)} X^{(2)})^2] + (\rho_{1,2} \frac{\sigma_1}{\sigma_2})^2 \E[X^{(2) 4}]
     - 2 \rho_{1,2} \frac{\sigma_1}{\sigma_2} \E[ X^{(1)} X^{(2) 3} ] \big).
\end{align*}
Standard formulas give 
\begin{align*}
    \E[(X^{(1)}X^{(2)})^2] = (1 + 2 \rho_{1,2}^2) \sigma_1^2 \sigma_2^2,
\end{align*} 
\begin{align*}
    \E[X^{(2) 4}] = 3\sigma_2^4,
\end{align*} and 
\begin{align*}
    \E[ X^{(1)} X^{(2) 3} ] = \E[X^{(2) 3} \E[X^{(1)} | X^{(2)}]] 
        = \rho_{1,2} \frac{\sigma_1}{\sigma_2} \E[X^{(2) 4} ].
\end{align*} 
Using these last three formulas in the previous result, we get
\begin{align*}
    \E[\V[m(\bX) | \bX^{(-1)}]]
    &= \frac{\alpha^2}{2} \big[ (1 + 2 \rho_{1,2}^2) \sigma_1^2 \sigma_2^2 + (\rho_{1,2} \frac{\sigma_1}{\sigma_2})^2 3\sigma_2^4 - 2 (\rho_{1,2} \frac{\sigma_1}{\sigma_2})^2 3\sigma_2^4 \big] \\
    &= \frac{\alpha^2}{2} \big[ (1 + 2 \rho_{1,2}^2) \sigma_1^2 \sigma_2^2 + 3 (\rho_{1,2} \sigma_1 \sigma_2)^2 - 6 (\rho_{1,2} \sigma_1 \sigma_2)^2 \big] \\
    &= \pmb{\frac{1}{2} (\alpha \sigma_1 \sigma_2)^2  (1 - \rho_{1,2}^2)}. 
\end{align*}

\subsection{Marginal Total Sobol Index $\pmb{ST_{mg}^{(1)}}$.}
By definition, $\V[Y] \times ST_{mg}^{(1)} = \E[\V[m(\bX_{\pi_1}) | \bX^{(-1)}]]$.
\begin{align*}
    \E[\V[m(\bX_{\pi_1}) | \bX^{(-1)}]] &= \E[(m(\bX_{\pi_1}) - \E[ m(\bX_{\pi_1}) | \bX^{(-1)}])^2] \\ 
    &= \E[ (\alpha X'^{(1)} X^{(2)} \mathds{1}_{X^{(3)} > 0} - \alpha \E[X'^{(1)} | \bX^{(-1)}] X^{(2)} \mathds{1}_{X^{(3)} > 0})^2],
\end{align*}
where $X'^{(1)}$ is an iid copy of $X^{(1)}$. Therefore $X'^{(1)}$ is independent of $\bX$ and $\E[X'^{(1)} | \bX^{(-1)}] = 0$, and we get
\begin{align*}
    \E[\V[m(\bX_{\pi_1}) | \bX^{(-1)}]] &= \frac{\alpha^2}{2} \E[(X'^{(1)}X^{(2)})^2]
    = \frac{\alpha^2}{2} \E[(X'^{(1)}] \E[X^{(2)})^2] \\
    &= \pmb{\frac{1}{2} (\alpha \sigma_1 \sigma_2)^2}.
\end{align*}

\subsection{Third MDA Component $\pmb{MDA_3^{(1)}}$.}
By definition,
\begin{align*}
    MDA_3^{(1)} = \E[(\E[m(\bX) | \bX^{(-1)}] - \E[m(\bX_{\pi_1}) |\bX^{(-1)}])^2]
\end{align*}
As computed above for the marginal total Sobol index, $\E[ m(\bX_{\pi_1}) | \bX^{(-1)}] = \beta X^{(4)} X^{(5)} \mathds{1}_{X^{(3)} > 0}$, thus
\begin{align*}
    MDA_3^{(1)} &= \E[(\alpha X^{(1)} \E[X^{(2)} | \bX^{(-1)}] \mathds{1}_{X^{(3)} > 0})^2] \\
    &= \frac{1}{2} \alpha^2 \E[ (X^{(1)} \E[X^{(2)} | X^{(1)}])^2 ] \\
    &= \frac{1}{2} \alpha^2  (\rho_{1,2} \frac{\sigma_1}{\sigma_2})^2 \E[X^{(2) 4}] \\
    &= \pmb{\frac{3}{2}\rho_{1,2}^2(\alpha\sigma_{1}\sigma_2)^2}.
\end{align*}

\subsection{Final MDA Limits}
Overall, using Proposition \ref{prop_MDA}, we obtain
\begin{align*}
    \textrm{MDA}^{\star (1)} =& 
    \underbrace{\frac{1}{2}(\alpha\sigma_{1}\sigma_2)^2(1 - \rho_{1,2}^2)}_{\textrm{MDA}_1^{\star (1)}} + \underbrace{\frac{1}{2}(\alpha\sigma_{1}\sigma_2)^2}_{\textrm{MDA}_2^{\star (1)}} + \underbrace{\frac{3}{2}\rho_{1,2}^2(\alpha\sigma_{1}\sigma_2)^2}_{\textrm{MDA}_3^{\star (1)}} \\
    \textrm{MDA}^{\star (1)} =& \pmb{(\alpha\sigma_{1}\sigma_2)^2(1 + \rho_{1,2}^2)}.
\end{align*}
By symmetry, $\textrm{MDA}^{\star (2)} = \textrm{MDA}^{\star (1)} = \pmb{(\alpha\sigma_{1}\sigma_2)^2(1 + \rho_{1,2}^2)}$, and
\begin{align*}
    \textrm{MDA}^{\star (4)} = \textrm{MDA}^{\star (5)} = \pmb{(\beta\sigma_{4}\sigma_5)^2(1 + \rho_{4,5}^2)}.
\end{align*}
Finally, since $X^{(3)}$ is independent of the other variables, Corollary $1$ gives
\begin{align*}
    \textrm{MDA}^{\star (3)} &= 2\textrm{MDA}^{\star (3)}_1 = 2 \E[\V[m(\bX) | \bX^{(-3}]].
\end{align*}
Next,
\begin{align*}
    \E[m(\bX) | \bX^{(-3)}] &= \E[ \alpha X^{(1)} X^{(2)} \mathds{1}_{X^{(3)} > 0} + \beta X^{(4)} X^{(5)} \mathds{1}_{X^{(3)} < 0} | \bX^{(-3)}] \\
    &= \frac{1}{2} \alpha X^{(1)} X^{(2)} + \frac{1}{2} \beta X^{(4)} X^{(5)},
\end{align*}
and 
\begin{align*}
    \V[\E[m(\bX) | \bX^{(-3)}]] &= \frac{1}{4} \alpha^2 \V[X^{(1)} X^{(2)}] + \frac{1}{4} \beta^2 \V[X^{(4)} X^{(5)}].
\end{align*}
Since
\begin{align*}
    \V[X^{(1)} X^{(2)}] =& \E[(X^{(1)} X^{(2)})^2] - \E[X^{(1)} X^{(2)}]^2 \\
    =& (1 + 2\rho_{1,2}^2)\sigma_1^2 \sigma_2^2 - (\rho_{1,2} \sigma_1 \sigma_2)^2 \\
    =& (1 + \rho_{1,2}^2) \sigma_1^2 \sigma_2^2,
\end{align*}
we obtain
\begin{align*}
    \V[\E[m(\bX) | \bX^{(-3)}]] &= \frac{1}{4} \alpha^2 (1 + \rho_{1,2}^2) \sigma_1^2 \sigma_2^2 + \frac{1}{4} \beta^2 (1 + \rho_{4,5}^2) \sigma_4^2 \sigma_5^2.
\end{align*}
On the other hand, 
\begin{align*}
    \V[m(\bX)] =& \alpha^2 \V[X^{(1)} X^{(2)} \mathds{1}_{X^{(3)} > 0}] + \beta^2 \V[X^{(4)} X^{(5)} \mathds{1}_{X^{(3)} < 0}] \\ \quad & + 2\textrm{Cov}[\alpha X^{(1)} X^{(2)} \mathds{1}_{X^{(3)} > 0}, \beta X^{(4)} X^{(5)} \mathds{1}_{X^{(3)} < 0}] \\
    =& \frac{\alpha^2}{2}(1 + 2\rho_{1,2}^2)\sigma_1^2 \sigma_2^2 - \frac{\alpha^2}{4}(\rho_{1,2} \sigma_1 \sigma_2)^2
    + \frac{\beta^2}{2}(1 + 2\rho_{4,5}^2)\sigma_4^2 \sigma_5^2 - \frac{\beta^2}{4}(\rho_{4,5} \sigma_4 \sigma_5)^2 \\
    \quad & - 2 \alpha \beta \frac{1}{4} \E[X^{(1)} X^{(2)}] \E[X^{(4)} X^{(5)}] \\
    =& \frac{\alpha^2}{2}(1 + \frac{3}{2} \rho_{1,2}^2)\sigma_1^2 \sigma_2^2
    + \frac{\beta^2}{2}(1 + \frac{3}{2} \rho_{4,5}^2)\sigma_4^2 \sigma_5^2
    - 2 \alpha \beta \frac{1}{4} \rho_{1,2} \sigma_1 \sigma_2 \rho_{4,5} \sigma_4 \sigma_5.
\end{align*}
Finally,
\begin{align*}
    \textrm{MDA}^{\star (3)} &= 2 \E[\V[m(\bX) | \bX^{(-3}]] = 2 (\V[m(\bX)] - \V[\E[m(\bX) | \bX^{(-3)}]]) \\
    &= 2 (\frac{\alpha^2}{4}(1 + 2 \rho_{1,2}^2)\sigma_1^2 \sigma_2^2
    + \frac{\beta^2}{4}(1 + 2 \rho_{4,5}^2)\sigma_4^2 \sigma_5^2
    - 2 \alpha \beta \frac{1}{4} \rho_{1,2} \sigma_1 \sigma_2 \rho_{4,5} \sigma_4 \sigma_5) \\
    &= \pmb{\frac{1}{2}(\alpha\sigma_{1}\sigma_2)^2(1 + \rho_{1,2}^2) + \frac{1}{2}(\beta\sigma_{4}\sigma_5)^2(1 + \rho_{4,5}^2) + \frac{1}{2}(\alpha\rho_{1,2}\sigma_{1}\sigma_2 - \beta\rho_{4,5}\sigma_{4}\sigma_5)^2}.
\end{align*}

\subsection{High Correlation Setting.}
In a high correlation setting, the third term becomes the main MDA contribution for variables $\bX^{(1)}$, $\bX^{(2)}$, $\bX^{(4)}$, and $\bX^{(5)}$. Since computations are similar, we only consider $\bX^{(1)}$:
\begin{align*}
    \textrm{MDA}^{\star (1)}_3 &> \textrm{MDA}^{\star (1)}_1 + \textrm{MDA}^{\star (1)}_2 \\
    \frac{3}{2}\rho_{1,2}^2(\alpha\sigma_{1}\sigma_2)^2 &> \frac{1}{2}(\alpha\sigma_{1}\sigma_2)^2(1 - \rho_{1,2}^2) + \frac{1}{2}(\alpha\sigma_{1}\sigma_2)^2 \\
    3 \rho_{1,2}^2(\alpha\sigma_{1}\sigma_2)^2 &> 2 (\alpha\sigma_{1}\sigma_2)^2 - (\alpha\sigma_{1}\sigma_2)^2 \rho_{1,2}^2 \\
    4 \rho_{1,2}^2(\alpha\sigma_{1}\sigma_2)^2 &> 2 (\alpha\sigma_{1}\sigma_2)^2 \\
    \rho_{1,2}^2 &> \frac{1}{2} \\
    \pmb{\rho_{1,2}} &> \pmb{\frac{\sqrt{2}}{2}}.
\end{align*}

\end{document}